\documentclass[10pt]{article} %\onecolumn \linespread{1.05}
\usepackage{fullpage, appendix}
\usepackage{url}            % simple URL typesetting
\usepackage{booktabs}       % professional-quality tables
\usepackage{amsfonts}       % blackboard math symbols
\usepackage{nicefrac}       % compact symbols for 1/2, etc.
\usepackage{microtype}      % microtypography
\usepackage{graphicx, wrapfig}

\usepackage{etex}
\usepackage{microtype}
\usepackage{graphicx}
\usepackage{subfigure}
\usepackage{booktabs} % for professional tables
\usepackage{tikz}
\usepackage{amsmath,amssymb,amsthm, bm}
\usepackage{tikz}
\usepackage{pgfplots}
\usepgfplotslibrary{groupplots}
\usetikzlibrary{spy}

\pgfplotsset{select coords between index/.style 2 args={
    x filter/.code={
        \ifnum\coordindex<#1\fi
        \ifnum\coordindex>#2\fi
    }
}}

\pgfplotstableread[col sep = comma]{figures/taubatch_mc2.dat}\vartau
\pgfplotstableread[col sep = comma]{figures/sigc_mc.dat}\varsigc
\pgfplotstableread[col sep = comma]{figures/rat_mc.dat}\varrat
\pgfplotstableread[col sep = comma]{figures/stmiss_mc.dat}\stmiss

\pgfplotstableread[col sep = comma]{figures/new_genmodel_large.csv}\newgenlarge
\pgfplotstableread[col sep = comma]{figures/small_rot.csv}\smallrot
\pgfplotstableread[col sep = comma]{figures/small_rot_fed.csv}\smallrotfed

\pgfplotstableread[col sep = comma]{figures/pw_const.csv}\pwconst
\pgfplotstableread[col sep = comma]{figures/pw_const_fed.csv}\pwconstfed

\pgfplotstableread[col sep = comma]{figures/new_algo_rot.csv}\smallrotcomp

\pgfplotstableread[col sep = comma]{figures/rhovar_stmiss.dat}\rhovar
\pgfplotstableread[col sep = comma]{figures/sigvar_stmiss.dat}\sigvar

\usetikzlibrary{positioning}

\usetikzlibrary{calc}
\tikzstyle{block}  = [rectangle, draw, rounded corners, text width=1cm, text centered, minimum height=1em]
\tikzstyle{smallblock}  = [rectangle, draw, rounded corners,text width=1.8cm, text centered, minimum height=1em]
\tikzstyle{input}  = [rectangle, draw, text width=1.2cm, text centered, minimum height=1em]
\tikzstyle{output}  = [rectangle, draw, rounded corners, text width=1cm, text centered, minimum height=1em]
\tikzstyle{block1}  = [rectangle, draw, rounded corners,text width=5cm, text centered, minimum height=1em]
\tikzstyle{blockl1}  = [rectangle, draw=red, rounded corners,text width=5cm, text centered, minimum height=1em]

\usepackage{amsmath, mathtools, amsfonts, bm, amssymb, amsthm,graphicx,epstopdf, caption}%, subfigure}
\usepackage{algorithm, algpseudocode}
\usepackage{bm, bbm}
\usepackage{color}
\usepackage[colorlinks=true, allcolors=blue]{hyperref}
\usepackage{cite, chngcntr}

\usepackage{enumitem}

\usepackage{dsfont}
\newcommand{\yt}{\bm{y}_t}

\newcommand{\R}{\mathbb{R}}
\newcommand{\bv}{\bm{v}}

\newcommand{\B}{\bm{B}}
\newcommand{\gap}{\mathrm{gap}}

\newtheorem{theorem}{Theorem}
\newtheorem{lem}[theorem]{Lemma}
\newtheorem{corollary}[theorem]{Corollary}
\newtheorem{definition}[theorem]{Definition}
\newtheorem{remark}[theorem]{Remark}

\newtheorem{assu}[theorem]{Assumption}

\newtheorem{fact}[theorem]{Fact}

\numberwithin{theorem}{section}
\newcommand{\norm}[1]{\left\|#1\right\|}

\renewcommand\thetheorem{\arabic{section}.\arabic{theorem}}

\newcommand{\Z}{\bm{Z}}

 \newcommand{\E}{\mathbb{E}}

\newcommand{\nsrmax}{\text{NSR}}

\renewcommand{\subsubsection}[1]{\noindent {\bf #1. }}

\newcommand{\y}{\bm{y}}
\renewcommand{\l}{{\bm{\ell}}}
\newcommand{\U}{{\bm{U}}}
\newcommand{\Y}{\bm{Y}}
\newcommand{\X}{\bm{X}}
\renewcommand{\L}{\bm{L}}
\newcommand{\Lhat}{\hat\L}
\newcommand{\Uhat}{\tilde{\U}}
\newcommand{\lhat}{\hat\l}
\newcommand{\at}{\bm{a}}

\newcommand{\J}{\mathcal{J}}

\newcommand{\Lam}{\bm{\Lambda}}
\newcommand{\zz}{\varepsilon}
\newcommand{\missfraccol}{\text{max-miss-frac-col}}
\newcommand{\missfracrow}{\text{max-miss-frac-row}}
\newcommand{\outfraccol}{\text{max-out-frac-col}}
\newcommand{\outfracrow}{\text{max-out-frac-row}}

\newcommand{\that}{\hat{t}}

\newcommand{\bz}{b}

\newcommand{\Tmisst}{{\T_t}}
\newcommand{\z}{\bm{z}}
\newcommand{\et}{\bm{e}}
\newcommand{\vt}{\bm{v}}
\newcommand{\A}{\bm{A}}
\newcommand{\init}{\mathrm{init}}

\newcommand{\calI}{\mathcal{I}}
\newcommand{\I}{\bm{I}}
\renewcommand{\yt}{\bm{y}}
\newcommand{\x}{\bm{x}}
\renewcommand{\P}{\bm{P}}
\renewcommand{\at}{\bm{a}}
\newcommand{\proj}{{\mathcal{P}}}
\newcommand{\tmax}{d}

\newcommand{\Phat}{\hat{\bm{P}}}

\newcommand{\W}{\bm{W}}
\newcommand{\w}{\bm{w}}
\newcommand{\T}{\mathcal{M}}

\newcommand{\bpsi}{\bm{\Psi}}
\newcommand{\s}{\bm{s}}
\newcommand{\SE}{\mathrm{\small{dist}}}
\renewcommand{\v}{\bm{v}}

\newcommand{\lthres}{\omega_{evals}} %{\lambda_{\mathrm{thresh}}}

\newcommand{\tU}{\tilde\U}
\newcommand{\Qhat}{\hat{\bm{U}}}
\newcommand{\taubatch}{\eta}

\newcommand{\beq}{\begin{equation} }
\newcommand{\eeq}{\end{equation} }
\newcommand{\bi}{\begin{itemize}} %[<+>]
\newcommand{\ei}{\end{itemize}}
\newcommand{\ben}{\begin{enumerate}}
\newcommand{\een}{\end{enumerate}}

\newcommand{\bphi}{\bpsi}

\newcommand{\M}{\bm{M}}

\renewcommand{\x}{\bm{s}}

\newcommand{\Tspart}{\T_{\mathrm{sparse},t}}
\newcommand{\xmint}{x_{\min}}
\newcommand{\tL}{\tilde{\L}}
\newcommand{\V}{\bm{V}}

\newcommand{\SVD}{\text{SVD}}

\newcommand{\G}{\bm{G}}
\newcommand{\tl}{\tilde{\l}}
\renewcommand{\S}{\bm{S}}

\newcommand{\xhat}{\hat{\x}}
\newcommand{\sparse}{\mathrm{sparse}}
\newcommand{\xmin}{s_{\min}}

\newcommand{\That}{\hat{\T}}

\renewcommand{\b}{\bm{b}}

\newcommand{\smin}{\xmin}

\newcommand{\e}{\bm{e}}

\title{Federated Over-Air Subspace Tracking from Incomplete and Corrupted Data}
\author{%
  Praneeth Narayanamurthy, Namrata Vaswani and Aditya Ramamoorthy  \\
  Iowa State University
  \thanks{A part of this work was presented at ICASSP 2022 \cite{fedrst_icassp}.}
}

\begin{document}
\maketitle

%\begin{document}

%\twocolumn[
%
%\aistatstitle{Federated Over-Air Subspace Tracking from Incomplete Data}
%
%\aistatsauthor{ Praneeth Narayanamurthy \And Namrata Vaswani \And  Aditya Ramamoorthy}
%
%\aistatsaddress{ Iowa State University} ]

%\newcommand{\Section}[1]{\vspace{-.2cm} \section{#1} \vspace{-.2cm}}
%\newcommand{\Subsection}[1]{\vspace{-.15cm} \subsection{#1} \vspace{-.15cm}}

\newcommand{\Section}[1]{\vspace{-.1cm} \section{#1} \vspace{-.1cm}}
\newcommand{\Subsection}[1]{\subsection{#1}}
\newcommand{\shat}{\xhat}
%\clearpage
\sloppy

{\begin{abstract}
 %Subspace tracking (ST) with missing data (ST-miss) or outliers (Robust ST) or both (Robust ST-miss) has been extensively studied in the past decade. 
In this work we study the problem of Subspace Tracking with missing data (ST-miss) and outliers (Robust ST-miss).  We propose a novel algorithm, and provide a guarantee for both these problems. Unlike past work on this topic, the current work does not impose the piecewise constant subspace change assumption. Additionally, the proposed algorithm is much simpler (uses fewer parameters) than our previous work. Secondly, we extend our approach and its analysis to provably solving these problems when the data is federated and when the over-air data communication modality is used for information exchange between the $K$ peer nodes and the center. We validate our theoretical claims with extensive numerical experiments. 
%This work provides a provably accurate and fast solution to the problem of dynamic subspace learning (tracking) from incomplete/missing data in a federated ``over-air'' wireless setting. The data is distributed across $K$ peer nodes. The nodes can only share a summary of their raw data with the central server. The available data is streaming and can be incomplete: at each time and at each node, some entries of the data matrix may be missing. The goal is to track the (changing) subspace in which the data lies from this incomplete data in a federated over-air fashion. Wireless over-air is a new transmission modality that allows for synchronous transmission by the peer nodes ($K$ times time- and bandwidth- efficient). However, the central server only receives a sum (superposition) of the individual transmissions and the received sum is corrupted by additive channel noise. From an algorithmic perspective, this means that the aggregation step at the central server needs to be a summation operation and that the aggregate is noisy. Additive noise in each algorithm iterate is a very different type of perturbation than noise or outliers corrupting the observed data. It introduces a novel set of challenges that have not been previously explored in the literature except in the context of introducing privacy.
%%
\end{abstract}}

\Section{Introduction}

%subspace tracking with missing data, and with missing data and outliers,
Subspace tracking (ST) with missing data, or outliers, or both has been extensively studied in the last few decades \cite{past, petrels, grouse_global, rrpcp_perf, petrels_theory}.  ST with outlier data is  commonly referred to as Robust ST (RST); it is the dynamic or ``tracking" version of Robust PCA \cite{rpca, altproj}.   This work provides a new simple algorithm and guarantee for both ST with missing data (ST-miss) and RST-miss. %Unlike past work on this topic, the algorithm is much simpler (uses fewer parameters) and the guarantee does not make the artificial assumption of piecewise constant subspace change, although it still handles that setting. %We also provide a second guarantee that works in case of piecewise constant sub
Secondly, we extend our approach and its analysis to provably solving these problems when the data is federated and when the over-air data communication modality \cite{ml_over_air} is used for information exchange between the $K$ peer nodes and the central server.  
%
%``Federated data'' means that the data is distributed across $K$ geographically distant nodes and the sharing of the raw data is not allowed due to privacy reasons. In each algorithm iteration, only summaries computed from the raw data can be shared with the center, which can aggregate the received summaries from the $K$ nodes and broadcast them to all the nodes for use in the next iteration. As we explain below, when the data aggregation step is a summation operation, over-air addition is up to $K$-times more time- or bandwidth-efficient than the traditional digital transmission mode. The tradeoff is that the algorithm needs to deal with additive channel noise which corrupts each algorithm iterate. Unlike traditional settings, this is not noise in the data; it is noise that modifies the algorithm iterates.%Hence, one needs the design o
%
 (R)ST-miss has important applications in video analytics \cite{matcomp_candes}, social network activity learning \cite{tensor_graph} (anomaly detection) and recommendation system design \cite{distpca_review} (learning time-varying low-dimensional user preferences from incomplete user ratings). The federated setting is most relevant for the latter two. At each time, each local node would have access to user ratings or messaging data from a subset of nearby users, but the subspace learning and matrix completion algorithm needs to use data from all the users.%(foreground-background subtraction)
 
Federated learning \cite{fed_learn} refers to a distributed learning scenario in which individual nodes keep their data private and only share intermediate locally computed summary statistics with the central server at each algorithm iteration. The central server in turn, shares a global aggregate of these iterates with all the nodes.  There has been extensive recent work on solving machine learning problems in a federated setting \cite{fed_ml_2, fed_ml_1, fed_ml_3, fed_ml_4, fed_ml_survey} but all these assume a perfect channel between the peer nodes and the central server. This is a valid assumption in the traditional digital transmission mode in which different peer nodes transmit in different time or frequency bands, and appropriate channel coding is done at lower network layers to enable error-free recovery with very high probability. 
 
 Advances in wireless communication technology now allow for (nearly) synchronous transmission by the various peer nodes and thus enable an alternate computation/communication paradigm for learning algorithms for which the aggregation step is a summation operation.  In this alternate paradigm, the summation can be performed ``over-air'' using the superposition property of the wireless channel and the summed aggregate (or its processed version) can then be broadcasted to all the nodes \cite{fed_fading,  ml_over_air, fed_ooa}. Assuming $K$ peer nodes, this over-air addition is up to $K$-times more time- or bandwidth-efficient than the traditional mode. In the absence of error control coding at the lower network layers, additive channel noise and channel fading effects corrupt the transmitted data. In general, there exist well-established  physical layer communication techniques to estimate and compensate for channel fading \cite{wireless_comm}. Also, while perfect synchrony in transmission is impossible, small timing mismatches can be handled using standard techniques. We expand upon both these points in Sec \ref{sec:fl_explain}. From a signal processing perspective, therefore, the main issue to be tackled is that of additive channel noise which now corrupts each algorithm iterate. %This introduces a {new and different set of challenges} in algorithm design and analysis compared to what has been largely explored in existing literature.

%%To the best of our knowledge, this work is the first systematic attempt to investigate the effect of iteration noise on an ML algorithm}.
%

\subsubsection{Related Work}
%?? While RST-miss has been studied in past work \cite{petrels, grouse_global, rrpcp_tsp19, eldar_jmlr_ss},  the main new contribution of this work is (i) we develop a generalized result for RST-miss in a centralized setting that does not employ a generative model on how the low-dimensional data is generated, (ii) a solution approach that respects the federated over-air data sharing constraints, and (ii) a guarantee for it that handles noisy algorithm iterates (due to the added channel noise). %Moreoever, because in our current setting, we have a set of $\alpha$ data points at a given time, our algorithm is, in fact, only instead of the one studied in \cite{add} which is mini-batch.
Provable ST with missing or corrupted data (ST-miss and RST-miss) in the centralized setting has been extensively studied in past work \cite{petrels, grouse_global, rrpcp_perf, rrpcp_dynrpca, rrpcp_tsp19, eldar_jmlr_ss}. { Provable analyses can be one of two kinds -- ones that come with a {\em complete guarantee or correctness result} and ones that come with only a {\em partial guarantee}. By {\em complete guarantee or correctness result}, we mean a result that makes assumptions only on the inputs to the algorithm (the observed data and the algorithm initialization if any) and guarantees that, %with high probability (or always), for any $\epsilon>0$,
the algorithm output will be close to the true value of the quantity of interest, either at all times, or after a finite delay. If a guarantee does not do this, we refer to it as a  {\em  partial guarantee}. Most existing works \cite{petrels, grouse_global, rrpcp_perf,eldar_jmlr_ss, rst_new} are partial guarantees. Although,  \cite{rrpcp_dynrpca, rrpcp_tsp19} obtain {complete guarantees}, these works impose a piecewise constant subspace change assumption.} This assumption is often not valid in practice, e.g, there is no reason for a ``subspace change time'' in case of slowly changing video backgrounds. The results of \cite{rrpcp_dynrpca, rrpcp_tsp19} assume it in order to obtain simple guarantees for $\epsilon$-accurate subspace recovery for any $\epsilon>0$ (in the noise-free case)  or for any $\epsilon$ larger than the noise-level (in the noisy case).

The only other existing works that also study unsupervised learning algorithms with noisy algorithm iterations are \cite{noisy_pm, improved_npm};  both these works study the noisy iteration version of the power method (PM) for computing the top $r$ singular vectors of a given data matrix. In these works, noise is deliberately added to each algorithm iterate in order to ensure privacy of the data matrix.

It should be noted that other solutions to batch low-rank matrix completion (LRMC) cannot be implemented to respect the federated constraints (the aggregation step needs to be a summation operation). We briefly discuss these in Sec. \ref{sec:subtrack}. Another somewhat related line of work involves distributed algorithms for PCA; these are reviewed in  \cite{distpca_review}, and there is also one for distributed ST-miss \cite{dist_rst}, % and low-rank matrix completion (LRMC) 
Most of these come without provable guarantees, and most also do not account for either missing data or iteration-noise or both. For example, the recent work \cite{distpca_balcan} aims to optimize communication efficiency but the channel is assumed to be perfect, and so iteration noise is not considered. Moreover, the algorithm is computationally expensive (involves computing a full SVD of a large matrix); and the guarantee provided is a multiplicative one on the PCA reconstruction error. Finally, LRMC in a decentralized setting is studied in \cite{dist_mc2} with the goal of speeding up computation via parallel processing using multiple computing nodes. In this paper as well, the full data is communicated to the central server and hence this is not a federated setting. Also, no channel noise is considered. It is not clear if this algorithm or guarantee can be modified to deal with federated data or over-air communication. Finally, there also exist heuristics for various types of distributed LRMC such as \cite{dist_mc1, byzantine_sgd, byzantine_ip}.

Other works that also develop algorithms for the federated over-air aggregation setting include \cite{ml_over_air, ml_in_air}. However, all these develop stochastic gradient descent (SGD) based algorithms and the focus is on optimizing resource allocation to satisfy transmit power constraints. These do not provide performance guarantees for the resulting perturbed SGD algorithm. 
A different related line of work is in developing {federated algorithms}, albeit not in the over-air aggregation mode. Recent works such as \cite{fed_ml_2, fed_comm} attempt to empirically optimize the {communication efficiency}. Similarly, \cite{fed_pca_1} studies federated PCA but it does not consider over-air communication paradigm, and does not deal with outliers or missing data. %Our work is strict generalization of this setting since we also consider iteration-noise.

% this assumes that each peer node performs LRMC locally, and transmits the result to the central server over a {\em perfect} channel. So
%Analyzing the effect of iteration-noise will require significant changes to the algorithm (if possible at all).

% since the final step involves computing a full SVD of a very large matrix which is a
%%that the resulting algorithm output is differentially private up to a certain level.

\subsubsection{Contributions}
%As noted above, all existing complete guarantees (those that provide a set of assumptions on algorithm inputs -- initialization, input data, or parameters -- that are needed to guarantee that the algorithm output is a good enough estimate of the quantity-of-interest) for ST-miss or RST-miss need to assume piecewise constant subspace change \cite{rrpcp_dynrpca, rrpcp_tsp19}.
This work has two contributions. First, we obtain a new set of results that provide a complete guarantee for ST-miss and RST-miss without assuming piecewise constant subspace change. The tradeoff is our error bounds are a little more complicated. Another advantage of our new result with respect to previous ReProCS algorithms \cite{rrpcp_jsait, rrpcp_tsp19} is that it analyzes a much simpler tracking algorithm (only one algorithm parameter needs to be set instead of three).
%: other than rank, we only need to set one other scalar parameter -- $\alpha$ which is the number of samples used for each PCA step.
Our guarantee is useful (improves upon the naive approach of standard PCA repeated every $\alpha$ frames) when the subspace changes are indeed slow enough. At the same time, we can still obtain a guarantee for our simpler algorithm that holds under piecewise constant subspace change but does not require an upper bound on the amount of change, i.e, we recover the result of \cite{rrpcp_tsp19}.

The second contribution of this work is a provable solution to the above problem in the federated data setting when the data communication is done in the over-air mode.  As explained above, the main new challenge here is to develop approaches that are provably robust to additive noise in the algorithm iterates. This setting of noisy iterations has received little attention in literature as noted above.  To the best of our knowledge, this is the first provable algorithm that studies (R)ST-miss in a federated, over-air paradigm. The main challenges here are (i) a design of an algorithm for this setting (this requires use of a federated over-air power method (FedOA-PM) for solving the PCA step) and (ii) dealing with noise iterates due to the channel noise. For the latter, the main work is in obtaining a modified result for PCA in sparse data-dependent noise solved via the FedOA-PM; see Lemma \ref{lem:fed_pca}.
% (data distributed across multiple nodes, respect privacy of data)

%This introduces a {new and different set of challenges} in algorithm design and analysis compared to what has been largely explored in existing literature.

\subsubsection{Paper organization} We give the centralized problem formulation next. After this, in Sec \ref{sec:stmiss}, we develop our solution for ST-miss in the centralized setting and explain how it successfully relaxes the piecewise constant subspace change assumption made by existing guarantees. Next, we directly consider RST-miss in the federated over-air setting in Sec \ref{sec:subtrack}. Simulations are provided in Sec \ref{sec:sims} and we conclude in Sec. \ref{sec:conc}. 

%The subspace update step needs careful modifications and involves solving a principal components analysis (PCA) problem. We first show how to solve this in a federated over-air (FedOA) fashion using the power method (PM) in Sec. \ref{sec:fedpm}. To simplify the paper, we first provide a guarantee for ST-miss in a centralized setting. This is given in Theorem \ref{thm:central_timevar}. Following this, we show how to use the above result to handle sparse outliers in Theorem \ref{thm:central_rstmiss}. Finally, we give our guarantee for the RST-miss in a federated setting in Theorem \ref{thm:stmiss}. The changes required in going from the centralized setting to the federated setting is modified result for PCA in sparse data-dependent noise solved via the FedOA PM; see Lemma \ref{lem:fed_pca}. This relies on Lemma \ref{lem:fed_app} for FedOA-PM. 

\newcommand{\tlhat}{\hat{\lhat}} %{\hat{\tl}}
\newcommand{\tLhat}{\hat{\Lhat}} %{\hat{\tL}}

\newcommand{\noiselev}{{\text{{no-lev}}}}
\renewcommand{\varepsilon}{\small{\noiselev}}

\newcommand{\epsse}{\epsilon}%{{\epsilon_{pca}}} %\epsse

\renewcommand{\B}{\bm{B}} %\tilde{\V}

\Section{Notation and Problem Formulation}
\subsection{Notation}
We use the interval notation $[a, b]$ to refer to all integers between $a$ and $b$, inclusive, and we use $[a,b): = [a,b-1]$. We use $[K] := [1,K]$.  $\|.\|$ denotes the $l_2$ norm for vectors and induced $l_2$ norm for matrices unless specified otherwise. We use $\I$ to denote the identity matrix of appropriate dimensions. We use $\M_{\mathcal{T}}$ to denote a sub-matrix of $\M$ formed by its columns indexed by entries in the set ${\mathcal{T}}$.
%For a matrix $\P$ we use $\P^{(i)}$ to denote its $i$-th row.
A matrix $\P$ with mutually orthonormal columns is referred to as a {\em basis matrix}; it represents the subspace spanned by its columns.
For basis matrices $\P_1,\P_2$, $\SE(\P_1,\P_2):=\|(\I - \P_1 \P_1^\top)\P_2\|$ quantifies the Subspace Error (distance) between their respective subspaces. This is equal to the sine of the largest principal angle between the subspaces. %It is also called ``projection distance'' \cite{chordal_dist}.
If $\P_1$ and $\P_2$ are of the same dimension, $\SE(\P_1, \P_2) = \SE(\P_2, \P_1)$.
{We reuse the letters $C,c$ to denote different numerical constants in each use with the convention that $C \ge 1$ and $c < 1$.}%

We use {\em $r$-SVD} to refer to the matrix of top-$r$ left singular vectors (vectors corresponding to the $r$ largest singular values) of the given matrix. Finally, $\M^\dag:= (\M^\top \M)^{-1} \M^\top$ is used to denote the pseudo inverse of $\M$.

%?? frames, samples.

\subsection{ST with missing data (ST-miss)}
%In this work, we study the problem of subspace tracking from missing data.
Assume that at each time $t$, we observe an $n$-dimensional data stream of the form
\begin{align}\label{eq:time_var}
\y_t = \proj_{\Omega_t}(\tl_t), \quad t = 1, 2, \cdots, \tmax
\end{align}
where $\proj_{\Omega_t}(\cdot)$ is a binary mask that selects entries in the index set $\Omega_t$ (this is known), and $\tl_t$ approximately lies in a low (at most $r$) dimensional subspace that is either constant or changes slowly over time. The goal is to track the  subspace(s). This statement can be made precise in several ways. The first is as done in past work \cite{rrpcp_tsp19} and references therein. One assumes a ``generative model'': $\tl_t = \P_t \at_t$ with $\P_t$ being a $ n \times r$ basis matrix. The goal is to track the column span of $\P_t$, $\mathrm{span}(\P_t)$. To make this problem well-posed (number of unknowns smaller than number of observed scalars), the piecewise constant subspace change model assumption becomes essential as explained in  \cite{rrpcp_tsp19}. %Assuming constant subspace change for at least $r$ time instants implies there is one $n \times r$ matrix to be learned from $
However, this is a restrictive assumption that is typically not valid for real data, e.g., there is no reason for the subspaces to change at certain select time instants in case of slowly changing videos.

% ($\alpha > C r \log n$)
A second approach to make our problem statement precise, and the one that we use in this work, is as follows. For an $\alpha$ large enough\footnote{as we show later $\alpha \ge C r \log n$ suffices}, consider $\alpha$-length sub-matrices formed by consecutive $\tl_t$'s. Let $\tL_1:=[\tl_1, \tl_2, \dots, \tl_\alpha]$; $\tL_2:= [\tl_{\alpha+1}, \tl_{\alpha+2}, \dots, \tl_{2 \alpha}]$ and so on. Let $\P_j$ be the $r$-SVD (matrix of top $r$ singular vectors) of $\tL_j$. Slow subspace change means that, for all $j$,
\begin{align*}
\Delta_j:= \SE(\P_{j-1}, \P_j) \leq \Delta_{tv}
\end{align*}
for a $\Delta_{tv} \ll 1$. { Note that the above problem formulation does not necessarily assume that the vectors $\{\tl_t\}_{t=(j-1)\alpha+1}^{j\alpha}$ are drawn from a specific subspace, but rather the subspace is defined post-facto. This subtle, yet important, difference allows us to eliminate the piecewise constant subspace change assumption in this work.} %Our guarantee assumes $\Delta_{tv}=0.1$.
%where $\SE(\cdot, \cdot)$ is the maximum principal angle distance between the corresponding subspaces (defined in Notation section).
The goal is to track (sequentially estimate) the subspace spanned by the columns of  $\P_j$ as well as the rank-$r$ approximation, $\L_j:= \P_j \P_j^\top \tL_j$. As is well known from the Eckart-Young theorem, this minimizes $\|\tL_j-\check\L\|_2$ over all rank $r$ matrices $\check\L$.
We will occasionally refer to $\L_j$ and its columns $\l_t$ as the {\em true data}.

Let $\A_j:= \P_j^\top \tL_j$ be the matrix of subspace coefficients along $\P_j$. %By the well-known Eckart-Young theorem for PCA, the rank-$r$ matrix $\L_j:= \P_j \A_j$ is the best rank $r$ approximation of $\tL_j$; best in the sense of minimizing $\|\tL-\L\|_2$.
Let $\V_j:= \tL_j - \L_j$ be the residual noise/error. Clearly,
%Under the approximate rank $r$ assumption on $\tL_j$, $\V_j$ will be small. Also,
since
\[
\tL_j \overset{\text{SVD}}{=} [ \P_j \underbrace{\bm{S} \B^\top}_{\A_j} + \underbrace{\P_{j,\perp} \bm{S}_{\perp} \B_\perp^\top}_{\V_j}] = \underbrace{\P_j \A_j}_{\L_j} + \V_j,
\]
%??pn $\V_j \P_{j, \perp} \bm{S}_{\perp} \B_\perp^\top$
it is immediate that
$
\L_j \V_j^\top = 0.
$

Let $\at_t$, $\l_t$ and $\v_t$ be the columns of $\A_j$, $\L_j$, and $\V_j$ respectively. Thus, for $t \in \J_j:=[{(j-1)\alpha + 1}, {(j-1)\alpha + 2},\dots j \alpha]$,
$\at_t  = \P_j^\top \tl_t$, $\l_t = \P_j \at_t$, and $\v_t = \tl_t - \l_t$.

Also, let $\T_t = (\Omega_t)^c$ be the index set of missing entries at time $t$. With this, we can rewrite \eqref{eq:time_var} as
\begin{align*}
\y_t = \proj_{\Omega_t}(\tl_t) &= \tl_t - {\I_{\T_t}\I_{\T_t}{}^\top \tl_t} \\
&= \l_t + \v_t -  \I_{\T_t}\I_{\T_t}{}^\top (\l_t + \v_t)   %_{\s_t}  = \l_t + \v_t + \s_t  %\underbrace{\I_{\T_t}\I_{\T_t}{}^\top \tl_t}_{\s_t}  = \l_t + \v_t + \s_t
\end{align*}

\subsection{Robust ST-miss (RST-miss)}
Robust ST-miss assumes that there can also be additive sparse outliers in the observed data $\y_t$. Thus, for all $ t = 1, 2, \cdots, \tmax $,
\begin{align}\label{eq:rob_time_var}
\y_t = %\proj_{\Omega_t}(\tl_t + g_t) :=
\proj_{\Omega_t}(\tl_t) + \s_t
\end{align}
where $\s_t$'s are the sparse outliers with supports $\Tspart$. The assumptions on $\Omega_t$, and the true data, $\tl_t$ remain as in the previous section. We provide the complete algorithm and guarantee for this problem in the supplementary material. %Notice that it is impossible to recover $\g_t$ on the set $\Tmisst$ and thus we only work with $\s_t$ in the sequel. Furthermore, by definition, $\s_t$ is supported outside $\Tmisst$ and thus $\Tmisst$ and $\Tspart$ are disjoint.

\subsection{Federated Over-Air Data Sharing Constraints and Iteration Noise}
We also solve RST-miss in a federated over-air fashion. Concretely, this means the following for an iterative algorithm. At iteration $l$, the central server broadcasts the $(l-1)$-th estimate of the quantity of interest\footnote{The quantity of interest could be a vector or a matrix depending on the application. For the problem we study (subspace learning/tracking), the quantity of interest is a $n \times r$ basis matrix.} denoted $\Qhat_{l-1}$ to each of the $K$ nodes. Each node then uses this estimate and its (locally) available data to compute the new local estimate denoted $\Uhat_{k,l}$. The nodes then synchronously transmit these to the central server but the transmission is corrupted by channel noise and thus the central server receives $$\Uhat_l := \sum_k \Uhat_{k,l} + \W_l$$ where $\W_l$ is the channel noise. We assume that $\W_l$ is independent of data and that each entry of $\W_l$ is i.i.d. zero-mean Gaussian with variance $\sigma_c^2$. The central server then processes $\Uhat_l$ to get the new estimate of the quantity of interest, $\Qhat_l$ which is then broadcast to all $K$ nodes for the next iteration. The presence of $\W_l$ in each iteration introduces a {new and different set of challenges} in algorithm design and analysis compared to what has been largely explored in existing literature.

\Section{ST from Missing Data (ST-miss)}\label{sec:stmiss}

\subsection{Proposed Algorithm}
Recall that we split our data into mini-batches of size $\alpha$; thus $\Y_1:=[\y_1, \y_2, \dots \y_\alpha]$, $\Y_2:=[\y_{\alpha +1}, \y_{\alpha +2}, \dots \y_{2\alpha}]$ and so on. Thus $\Y_j:=[\y_{(j-1)\alpha+1}, \y_{(j-1)\alpha+2}, \dots, \y_{ j\alpha}]$.
Without the slow subspace change assumption, the obvious way to solve ST-miss would be to use what can be called {\em simple PCA}: for each mini-batch $j$, compute $\Phat_j$ as the $r$-SVD of $\Y_j$.
However, when slow subspace change is assumed, a better approach is a simplification of our algorithm from \cite{rrpcp_tsp19}. %The complete algorithm is summarized as Algorithm \ref{algo:norst_nodet}. The algorithm is initialized
We initialize via $r$-SVD: compute $\Phat_1$ as the $r$-SVD of $\Y_1$. For the $j$-th mini-batch, we first obtain {\em an} estimate of the missing entries for each column using the previous subspace estimate and projected Least Squares (LS) as follows. For every $t \in ((j-1)\alpha, j\alpha]$, we compute
\begin{align}\label{eq:proj_step}
\lhat_t &= \y_t -\I_{\T_t} \bpsi_{\T_t}^{\dagger} \bpsi \y_t
\end{align}
where $\bpsi = \I - \Phat_{j-1}\Phat_{j-1}^\top$. This step works as long as (i) the span of $\Phat_{j-1}$ is a good estimate of that of $\P_{j}$ and (ii) $\bpsi_{\T_t}$ is well conditioned (or has full-column rank). We argue the first point by assuming that the span of $\Phat_{j-1}$ is a good estimate of that of $\P_{j-1}$ and furthermore, owing to slow subspace change, it is also a good estimate of the span of $\P_{j}$. We ensure the second point by bounding the number of missing entries in each column, $|\T_{t}|$ in our main result. This point is further explained in Remark \ref{rem:psi_cond}. 

Observe that \eqref{eq:proj_step} is a compact way to write the following: $(\lhat_t)_{\T_t^c} = (\y_t)_{\T_t^c} =(\tl_t)_{\T_t^c}$ (use the observed entries as is) and $(\lhat_t)_{\T_t} =  \bpsi_{\T_t}^{\dagger} ( \bpsi \y_t)$.
To understand this, notice that $\bpsi \y_t = -  \bpsi_{\T_t} \z_t  + (\bpsi\l_t + \bpsi\v_t) $ where $\z_t:= (\I_{\T_t}{}^\top\tl_t)$ is the vector of missing entries. The second two terms can be treated as small ``noise''/disturbance\footnote{The first is small because of slow subspace change and $\Phat_{j-1}$ being a good estimate (if $\mathrm{span}(\Phat_{j-1}) = \mathrm{span}(\P_j)$ this term would be zero); the second is small because $\|\v_t\|$ is small due to the approximate low-rank assumption.} and so we can compute an estimate of $\z_t$ from $\bpsi \y_t$ by LS. %as $\bpsi_{\T_t}^\dag (\bpsi \y_t)$.
%Also see the first paragraph of Sec. \ref{sec:proof_cent}.

The second step is to compute $\Phat_j$ as the $r$-SVD of $\Lhat_j := [\lhat_{(j-1)\alpha+1}, \cdots, \lhat_{j\alpha}]$.

Finally, we can use $\Phat_j$ to obtain an optional improved estimate, $\tlhat_t = \y_t - \bm{I}_{\T_t} \tilde\bpsi_{\T_t}^{\dagger} ( \tilde\bpsi \y_t)$ where $\tilde\bpsi = \I - \Phat_j \Phat_j^\top$.
We summarize this approach in Algorithm \ref{algo:norst_nodet}.
We show next that, under slow subspace change, Algorithm \ref{algo:norst_nodet} yields  significantly better subspace estimates than simple PCA (PCA on each $\Y_j$).

%We use For all $j$, we add an optional projected LS step to get a better estimate of the low-rank data. Following this, we obtain an improved subspace estimate, as the top-$r$ left singular vectors of $\Lhat_j$ and show that this is close to the ``true subspace'', $\P_j$.

\begin{algorithm}[t!]
\caption{STMiss-NoDet}
\label{algo:norst_nodet}
\begin{algorithmic}[1]
\Require $\Y$, $\T$ %$r'$, $\tau_t^* = C \log(n /\epsilon_t)$
\State {\em Parameters:} %$T_{iter} \leftarrow C \log(1/\zz)$, $\mathrm{phase} \leftarrow \mathrm{update}$,
%$L \leftarrow C \log (nr/\varepsilon)$, %$L_{\mathrm{det}} \leftarrow C \log(nr/\varepsilon)$, $\lthres \leftarrow 2 \zz^2\lambda^+$. \em{Parameters:}
 $\alpha$
\State {\em Initialize:} $\Phat_1 \leftarrow$  $r$-$\SVD[\y_{1}, \cdots, \y_{\alpha} ] $, $j\leftarrow 2$
%\State $\Lhat_0 \leftarrow$ \Call{Rob-FedOA-ProjLS}{$\calI_{k,0}$, $\y_i$, $\T_i$, $\Phat_{0}$}
%\State $\Phat_{1} \leftarrow$  \Call{FedOA-PM}{$\Lhat_{1}$, $r$, $L$}
\For{$j \geq 2$ } %\in ((b-1)\alpha, b\alpha]$}
\State {\em Projected LS:}
\State $\bpsi \leftarrow \I - \Phat_{j-1}\Phat_{j-1}^{\top}$
\ForAll  {$t \in ((j-1)\alpha, j\alpha]$}
\State $\lhat_t \leftarrow \y_{t} -  \I_{{\T}_{t}} (\bpsi_{{\T}_t})^{\dagger} (\bpsi \y_{t} )$
\EndFor
%\If {$t = t_{\train} + u \alpha - 1$ for $u = 1,\ 2,\ \cdots,$}
%$\tilde{\y}_{t} \leftarrow \bpsi \y_{t}$;
\State {\em PCA on $\Lhat_j$:}
\State $\Phat_{j} \leftarrow$  $r$-$\SVD(\Lhat_j)$ where $\Lhat_j:=[\lhat_{(j-1)\alpha + 1}, \cdots, \lhat_{j\alpha} ]$

%\State {\em Improved estimate  (Optional)}:
%\State $\tilde{\bpsi} \leftarrow \bm{I} - \Phat_{j}\Phat_{j}^{\top}$
\ForAll {$t \in ((j-1)\alpha, j\alpha]$} \Comment{optional}
 %\State For all $t \in ((j-1)\alpha, j\alpha]$,
%$\tilde{\y}_{t} \leftarrow \tilde{\bpsi} \y_{t}$;
\State $\tilde{\bpsi} \leftarrow \bm{I} - \Phat_{j}\Phat_{j}^{\top}$
\State $\tlhat_{t} \leftarrow \y_{t} -  \I_{{\T}_{t}} (\tilde{\bpsi}_{{\T}_t})^{\dagger} (\tilde{\bpsi} \y_{t})$ %with
\EndFor
\EndFor
\Ensure $\Phat_j$, $\lhat_t$, $\tlhat_t$.
\end{algorithmic}

\end{algorithm}

\subsection{Assumptions and Main Result}
It is well known from the LRMC literature \cite{matcomp_candes} that for guaranteeing correct matrix recovery, we need to assume incoherence (w.r.t. the standard basis) of the left and right singular vectors of the matrix.
We need a similar assumption on $\P_j$'s.

{
\begin{assu}[$\mu$-Incoherence of $\P_j$s]
 Assume that
%In this vein, we assume incoherence of the subspace basis matrices, $\P_t$, i.e., assume that for all $t$, for some constant $\mu$, the following holds
\begin{align*}
\max_{j \in[\tmax/\alpha]} \max_{m \in [r]} \|\P_j^{(m)}\|_2^2 \le \frac{\mu r}{n}
%\label{mu_incoh}`
\end{align*}
where $\P^{(m)}_j$ denotes the $m$-th row of $\P_j$ and $\mu \ge 1$ is a constant (incoherence parameter).
\label{def_left_incoh}
\end{assu}
}
Since we study a tracking algorithm (we want to track subspace changes quickly), we replace the standard right singular vectors' incoherence assumption with the following simple statistical assumption on the subspace coefficients $\at_t$. This helps us obtain guarantees on our mini-batch algorithm that operates on $\alpha$-size mini-batches of the data.
%?? This helps us obtain high probability upper bounds on tracking delay.
%Since we study a subspace tracking problem, we use the following statistical model on the subspace coefficients in lieu of right $\mu$-incoherence.
{
\begin{assu}[Statistical $\mu$-Incoherence of $\at_i$s]
Recall that  $\at_t = \P_j^\top\tl_t$ for all $t \in \J_j$.
Assume that the $\at_t$'s are zero mean; mutually independent; have identical diagonal covariance matrix $\Lam$, i.e., that $\E[\at_t \at_t^{\top}] = \Lam$ with $\Lam$ diagonal; and are bounded, i.e., $\max_t \|\at_t\|^2 \le \mu r \lambda^+ $, where $\lambda^+ := \lambda_{\max}(\Lam)$ and $\mu \ge 1$ is a small constant. Also, let $\lambda^- := \lambda_{\min}(\Lam)$ and $f:= \lambda^+/\lambda^-$.
\label{def_right_incoh}
\end{assu}
}

If a few complete rows (columns) of the entries are missing, in general it is not possible to recover the underlying matrix. This can be avoided by either assuming bounds on the number of missing entries in any row and in any column, or by assuming that each entry is observed uniformly at random with probability $\rho$ independent of all others. In this work we assume the former which is a weaker assumption. We need the following definition.

{
\begin{definition}[Bounded Missing Entry Fractions]
Consider the $n \times \alpha$ observed matrix $\Y_j$ for the $j$-th mini-batch of data. %across all the $K$ nodes.
We use $\missfraccol$ ($\missfracrow$)  to denote the maximum of the fraction of missing entries in any column (row) of this matrix.
\end{definition}
}
%\subsection{Main Result}
%In this section we consider the following setting on the subspaces, $\P_j$. Notice that $\P_j$ is the matrix of top-$r$ left singular vectors of $j$-th mini-batch of data, $\tL_j := [\tl_{(j-1)\alpha + 1}, \cdots, \tl_{j\alpha}]$. Let $\Delta_{tv} := \max_j \SE(\P_{j-1}, \P_j) < 1$. Thus, $\Delta_{tv}$ quantifies the distance between the principal subspaces of $\alpha$-frames of the true data.
%With this assumption, as mentioned before, in general, the true data, $\tL_t$ could in general have a rank greater than $r$. However, we assume that $\tL_t$ is approximately low-rank.
Owing to the assumption that $\tL_j$ is approximately low-rank, it follows that $\tL_j - \L_j := \V_j$ is ``small''.
{
\begin{assu}[Small, bounded, independent modeling error]
Let $\lambda_v^+:= \max_{t} \|\E[ \vt_{t} \vt_{t}^{\top}]\|$. We assume that $\lambda_v^+ < \lambda^-$, $\max_{t} \|\vt_{t}\|^2 \le C r \lambda_v^+$ and $\vt_t$'s are mutually independent over time.
\label{noisebnd}
\end{assu}}
%Instead of quantifying the bounds on $\check{\V}_t$ we directly impose a bound on $\V_t := \tilde{\V}_t - \check{\V}_t$. We have the following result.

\subsubsection{Main result}
We have the following result for the naive algorithm of PCA on every mini-batch of $\alpha$ observed samples $\Y_j$.
We use the following definition of noise level
\[
\noiselev: = \sqrt{\lambda_v^+/ \lambda^-}
\]

\newcommand{\rhocol}{\rho_{\mathrm{col}}}
\newcommand{\rhorow}{\rho_{\mathrm{row}}}

{
\begin{theorem}[STmiss Algorithm \ref{algo:fed_nodet_given_init}]

Set algorithm parameter $\alpha = C f^2 r \log n$.  %Pick an $\zz$ that satisfies $c \sqrt{\lambda_v^+/\lambda^-} \leq \zz \leq 0.2$.

Assume that $\zz < 0.2$ and the following hold:
\begin{enumerate}

\item  {\bf Incoherence:} $\P_j$'s satisfy $\mu$-incoherence, and $\at_t$'s satisfy statistical right $\mu$-incoherence;

\item {\bf Missing Entries:} %$\missfraccol \in O(1/\mu r)$, $\missfracrow \in O(1)$;
 $\missfraccol \le \rhocol/(\mu r)$, $\missfracrow \le \rhorow/f^2$ s.t., 
 \begin{align*}
 7 \sqrt{\rhorow \rhocol} + \zz^2 \leq \max(\zz,0.25\sqrt{\rhocol})
\end{align*}

\item {\bf Modeling Error:} Assumption \ref{noisebnd} holds.
%we assume that $\max_{t} \|\vt_{t}\|^2 \le C r \lambda_v^+$ % and assume that $\lambda_v^+/\lambda^- \leq c \zz^2/f$.
%\item {\bf Subspace Model:} The total data available at each time $t$, $\alpha \in \Omega(r \log n)$ and $\Delta_{tv} := \max_t \SE(\P_{t-1},\P_t)$ s.t.
%\begin{gather*}
%\Delta_{tv} \leq c \varepsilon^2/f^2, \quad 0.3 \epsilon_{\init} + 1.5\Delta_{tv} \leq 0.28 \quad \text{and} \\
%C \sqrt{r \lambda^+} (0.3^{t-1}\epsilon_{\init} + 1.5\Delta_{tv} ) + \sqrt{r_v\lambda_v^+} \leq \xmin
%\end{gather*}
\item {\bf Subspace Change:} $\max_j \SE(\P_{j-1}, \P_j) := \Delta_{tv}$, s.t. ${\Delta_{tv} \leq \zz}$ and
\begin{align*}
\max(\zz, 0.25\sqrt{\rhocol} + 3/7\Delta_{tv}) \leq \Delta_{red} < 1
\end{align*}
\end{enumerate}
then, with probability at least $1 - 10 \tmax n^{-10}$, we have
%the $j$-th subspace change is detected immediately i.e., $t_j \leq \that_j \leq t_j + 1$, and its tracking error decays exponentially after detection, i.e., $\SE(\Phat_{t}, \P_t) \le$
 \begin{align*}
&\SE(\Phat_j, \P_j) \\
&\leq \max(\sqrt{\rhocol} \cdot 0.3^{j-1} + \Delta_{tv} (\Delta_{red} + \Delta_{red}^2 ... + \Delta_{red}^{j-1}), \zz) \\
&< \max\left(\sqrt{\rhocol} \cdot \Delta_{red}^{j-1} + \frac{\Delta_{red} \Delta_{tv}}{1 - \Delta_{red}}, \zz\right)
\end{align*}
Also, at all $j$, and for $t \in [(j-1)\alpha, j\alpha)$,
$\|\tlhat_{t}-\tl_{t}\| \le 1.2 \cdot  \SE(\Phat_{j}, \P_j) \|\tl_{t}\| + \|\v_t\|$
while
$\|\lhat_{t}-\tl_{t}\| \le 1.2 \cdot  \SE(\Phat_{j-1}, \P_j) \|\tl_{t}\| + \|\v_t\| \le 1.2 \cdot (\Delta_{tv} +  \SE(\Phat_{j}, \P_j) ) \|\tl_{t}\| + \|\v_t\| $.
\label{thm:central_timevar_generic}
\end{theorem}

Proof: See Sec. \ref{sec:proof_cent}.

\subsection{Discussion}

%\subsubsection{Regarding Assumptions}
First consider the noiseless setting, i.e., the data is exactly rank-$r$. The condition on the missing entries requires that $\rhorow \leq (0.25/7)^2 \leq 0.16$. While this might seem restrictive at first glance, these constants can be varied by modifying the PCA-SDDN result (Corollary \ref{cor:cent_pca_dd}). Next, from the final subspace error expression, note that $\Delta_{red} < 1$, governs the rate at which the error decays. As expected, increasing the number of missing entries in a column (proportional to $\rhocol$), or the maximum amount of subspace change, $\Delta_{tv}$ reduces the convergence rate to an $\epsilon$-accurate solution. In the presence of noise, without further assumptions on $\vt$, in general it is not possible to obtain a final error that is lower than $\zz$ and in this case, the tradeoffs are not as straightforward.    

For ease of notation, we provide a special case of Theorem \ref{thm:central_timevar_generic} with specific values of the various constants next. 

\begin{theorem}[STmiss -- Special Case]
Set algorithm parameter $\alpha = C f^2 r \log n$.  %Pick an $\zz$ that satisfies $c \sqrt{\lambda_v^+/\lambda^-} \leq \zz \leq 0.2$.

Assume that $\zz < 0.2$ and the following hold:
\begin{enumerate}

\item  {\bf Incoherence:} $\P_j$'s satisfy $\mu$-incoherence, and $\at_t$'s satisfy statistical right $\mu$-incoherence;

\item {\bf Missing Entries:} %$\missfraccol \in O(1/\mu r)$, $\missfracrow \in O(1)$;
 $\missfraccol \le 0.01/ (\mu r)$, $\missfracrow \le 0.0001 / f^2$;

\item {\bf Modeling Error:} Assumption \ref{noisebnd} holds
%we assume that $\max_{t} \|\vt_{t}\|^2 \le C r \lambda_v^+$ % and assume that $\lambda_v^+/\lambda^- \leq c \zz^2/f$.
%\item {\bf Subspace Model:} The total data available at each time $t$, $\alpha \in \Omega(r \log n)$ and $\Delta_{tv} := \max_t \SE(\P_{t-1},\P_t)$ s.t.
%\begin{gather*}
%\Delta_{tv} \leq c \varepsilon^2/f^2, \quad 0.3 \epsilon_{\init} + 1.5\Delta_{tv} \leq 0.28 \quad \text{and} \\
%C \sqrt{r \lambda^+} (0.3^{t-1}\epsilon_{\init} + 1.5\Delta_{tv} ) + \sqrt{r_v\lambda_v^+} \leq \xmin
%\end{gather*}
\item {\bf Subspace Change:} $\max_j \SE(\P_{j-1}, \P_j) \le \Delta_{tv}=0.1$,
\end{enumerate}
then, with probability at least $1 - 10 \tmax n^{-10}$, we have
%the $j$-th subspace change is detected immediately i.e., $t_j \leq \that_j \leq t_j + 1$, and its tracking error decays exponentially after detection, i.e., $\SE(\Phat_{t}, \P_t) \le$
 \begin{align*}
&\SE(\Phat_j, \P_j) \\
&\leq \max(0.1 \cdot 0.3^{j-1} + \Delta_{tv} (0.3 + 0.3^2 ... + 0.3^{j-1}), \zz) \\
&< \max(0.1 \cdot 0.3^{j-1} + 0.5\Delta_{tv}, \zz)
\end{align*}

Also, at all $j$, and for $t \in [(j-1)\alpha, j\alpha)$,
$\|\tlhat_{t}-\tl_{t}\| \le 1.2 \cdot  \SE(\Phat_{j}, \P_j) \|\tl_{t}\| + \|\v_t\|$
while
$\|\lhat_{t}-\tl_{t}\| \le 1.2 \cdot  \SE(\Phat_{j-1}, \P_j) \|\tl_{t}\| + \|\v_t\| \le 1.2 \cdot (\Delta_{tv} +  \SE(\Phat_{j}, \P_j) ) \|\tl_{t}\| + \|\v_t\| $.
\label{thm:central_timevar}
\end{theorem}

Proof: See Sec. \ref{sec:proof_cent}.

In the sequel, we only build upon the special case, Theorem \ref{thm:central_timevar}. As a baseline, consider the following naive approach to solve the ST-miss problem and its associated result:

\begin{theorem}[Simple PCA]
%Consider the naive algorithm where we just compute the subspace estimates as the top $r$ left singular vectors of each mini-batch of the data matrix $\Y_j$.
Let $\Phat_j$ be the $r$-SVD of $\Y_j$ with $\alpha = C f^2 r \log n$.  Assume $\mu$-incoherence of $\P_j$s, statistical $\mu$-incoherence of $\at_i$s, modeling error assumption given in Assumption \ref{noisebnd}, $\missfraccol \le 0.01/ (\mu r)$, $\missfracrow \le 0.01 / f^2$.
%Let $\max_{t} \|\vt_{t}\|^2 \le C r \lambda_v^+$, %and pick an $\zz$ that satisfies $c \sqrt{\lambda_v^+/\lambda^-} \leq \zz \leq 0.2$. Assume that $\lambda_v^+/\lambda^- \leq c \zz^2/f$.
Then, with probability at least $1 - 10dn^{-10}$,
\begin{align*}
\SE(\Phat_j, \P_j) \leq \max(0.1 \cdot 0.25, \zz)
\end{align*}
\label{thm:naive}
\end{theorem}

Proof: The proof is the same as that for the initialization step of Algorithm \ref{algo:norst_nodet}; see Sec. \ref{sec:proof_cent}.

To compare our main result, Theorem \ref{thm:central_timevar}, consider the practically relevant setting of approximately rank $r$ $\tL_j$'s so that the noise level $\sqrt{\lambda_v^+/\lambda^-}$ is small. In particular, assume it is smaller than $0.1 \cdot 0.25$.
Then, if $\Delta_{tv}$ is small enough, the bound of Theorem \ref{thm:central_timevar} is significantly smaller. If the noise level is larger, then in both cases, the noise level term dominates and both results give the same bound. Thus, in all cases, as long as $\Delta_{tv}$ is small (slow subspace change holds), Theorem \ref{thm:central_timevar} gives an as good or better bound as the naive approach. We demonstrate this in Fig \ref{fig:st_miss}.

%proof of Theorem \ref{thm:central_timevar} given in

\textbf{Note:} Our result assumes a mix of deterministic and stochastic assumptions due to the following. As mentioned earlier, Algorithm \ref{algo:norst_nodet} is modification of an algorithm for RST-miss from \cite{rrpcp_jsait} in which we treat {\em missing entry recovery} as a special case of sparse recovery using ideas from Compressive Sensing (CS), and the subspace update step as a Dynamic PCA problem. For the CS problem, we require a (deterministic) bound on the number of non-zero entries (missing entries) but not on {\em how} the support set is generated, i.e., we can tolerate deterministic patterns on set of missing entries. Furthermore, for the sparse recovery step to work, the CS result \cite{candes_rip} requires that the $2|\T_t|$-level incoherence of the measurement matrix, $\bpsi_{\T_t}$ is bounded by $\sqrt{2}-1$. This translates to our incoherence bounds on the subspaces. For the dynamic PCA problem, it is customary to impose stochastic assumptions on either the subspaces or the subspace coefficients and in this paper, we choose the latter. For a detailed comparison with the best known results in subspace tracking, we refer the reader to \cite[Sec. III]{rrpcp_tsp19}.

\begin{remark}[Demonstrating full column-rank of $\bpsi_{\T_t}$] \label{rem:psi_cond}
 Notice that $\bpsi_{\T_i} \in \R^{n \times |\T_t|}$ and under the conditions of Theorem \ref{thm:central_timevar}, we assume that (for some $c < 1$) $|\T_i| \leq c n/(\mu r)$ and thus for $\bpsi_{\T_i}$ to have full column rank, one needs, 
\begin{align*}
\frac{cn}{\mu r} \leq n - r \implies \mu \geq \frac{ cn}{r(n-r)}
\end{align*}
notice that $r (n-r) \in [0, n^2/4]$ and thus one only needs that $\mu \geq 4c/n$. Now, as long as $c \in [0, 1/4]$, this bound is satisfied for all $n$ since $\mu \geq 1$ by definition. If $c \in (1/4, 1]$, as long as $n > 4/c$, the condition is again satisfied. In other words, the matrix $\bpsi_{\T_i}$ has full-column rank for all $n > 4$.
\end{remark}
}

\subsection{Guarantee for piecewise constant subspace change}
Previous work on provable ST-miss \cite{rrpcp_tsp19} assumed piecewise constant subspace change (required the subspace to be constant for long enough), but did not require an upper bound on the amount of change. As we show next STmiss-NoDet is able to track such changes as well and provide similar tracking guarantees even under a (mild) generalization of the previous model.

\begin{theorem}
Set algorithm parameter $\alpha = C f^2 r \log n$.
Assume that $\zz < 0.02$ and the first three assumptions of Theorem \ref{thm:central_timevar} hold.
Under an approximately piecewise constant subspace change model ($\Delta_j \le \zz$ for all $j$ except for $j = j_\gamma$, for $\gamma = 1,2, \dots, $) with the subspace change times satisfying $j_{\gamma} - j_{\gamma - 1} > K:=C \log(1/\zz)$,
then, w.p. at least $ 1- d n^{-10}$,
%at least $K$ contiguous $\alpha$-length intervals with $K = C \log(1/\zz)$, but no upper bound when the change occurs), the following holds:
\begin{align*}
&\SE(\Phat_j, \P_j) \leq  \\
&\begin{cases}
(0.2 + 2\zz) \cdot 0.25+ \zz), \quad \text{if} \quad j = j_\gamma \\ %\text{ for some $\gamma$}
(0.2 + 2\zz) \cdot 0.3^{(j-j_\gamma)-1} + \zz, \quad \text{if} \quad  j_\gamma <  j <   j_{\gamma+1}
%\\ \epsilon, \quad \text{if} \quad j^\ast > J^\ast
\end{cases}
\end{align*}
Notice that for $ j_{\gamma+1} > j > j_\gamma + K$, the bound is at most $2\zz$.
\label{thm:large_ss}
\end{theorem}
The subspace change model in this result does not require an upper bound on the amount of subspace change as long as the change occurs infrequently. However, it still allows for small rotations to the subspace at each time. The exponential decay in the subspace recovery error bound is the same as that guaranteed by the results is \cite{rrpcp_tsp19}.
STmiss-NoDet does not detect subspace changes. However, a detection step similar to that used in previous work can be included and then a similar detection guarantee can also be proved. We provide these in the Supplementary Material.

%Under the assumptions of Theorem \ref{thm:central_timevar} the following holds.

\subsection{Proof of Theorem \ref{thm:central_timevar} and \ref{thm:naive}} \label{sec:proof_cent}% and  Theorem \ref{thm:large_ss}}
The proof follows by a careful application of a result from  \cite{rrpcp_jsait} that analyzes PCA in sparse data-dependent noise (SDDN) along with  simple linear algebra tricks, some of which are also borrowed from there. The novel contribution here is the application of the same ideas for providing a result that holds under a much simpler and practically valid assumption of slow changing subspaces (without any artificial piecewise constant assumption). Also, the proof provided here is much shorter.

\subsubsection{Subspace error bounds}
Consider the projected LS step. Recall that $\bpsi = \I - \Phat_{j-1}\Phat_{j-1}^\top$.
Since $\y_t$ can be expressed as $\y_t = \tl_t -\I_{\T_t} (\I_{\T_t}^\top \tl_t) $, using the idea explained while developing the algorithm,
\begin{align*}
\lhat_t &= \y_t - \I_{\T_t}\bpsi_{\T_t}^{\dagger} \bpsi (-\I_{\T_t} \I_{\T_t}^\top \tl_t + \tl_t) \\
&= \y_t - \I_{\T_t}(\bpsi_{\T_t}^\top \bpsi_{\T_t})^{-1} \bpsi_{\T_t}^\top \bpsi (-\I_{\T_t} \I_{\T_t}^\top \tl_t + \tl_t) \\
%&= \y_t + \I_{\T_t}(\bpsi_{\T_t}^\top \bpsi_{\T_t})^{-1} (\bpsi_{\T_t}^\top \bpsi_{\T_t}) \I_{\T_t}^\top \tl_t - \I_{\T_t}(\bpsi_{\T_t}^\top \bpsi_{\T_t})^{-1} \bpsi_{\T_t}^\top \tl_t \\
&= \y_t + \I_{\T_t} \I_{\T_t}^\top \tl_t - \I_{\T_t}(\bpsi_{\T_t}^\top \bpsi_{\T_t})^{-1} \bpsi_{\T_t}^\top \tl_t \\
&= \tl_t - \I_{\T_t}(\bpsi_{\T_t}^\top \bpsi_{\T_t})^{-1} \bpsi_{\T_t}^\top \tl_t \\
&= \l_t + \bv_t - \bm{I}_{\T_t} \left(\bpsi_{\T_t}\right)^{\dagger} \bpsi_{\T_t}^\top  (\l_t + \bv_t)
\end{align*}
This final expression can be reorganized as follows.
%The above uses the fact that $\y_t$ can be written as $\y_t =  - \I_{\T_t} (\I_{\T_t}^{\top} \tilde{\l}_t) + \tl_t$ with  $\tl_t = \l_t + \bv_t $.
%
\begin{align} \label{eq:etdef_cent}
%\lhat_t &= \l_t + \underbrace{\bv_t}_\text{{small, unstructured noise}} - \underbrace{\bm{I}_{\T_t} \left(\bpsi_{\T_t}\right)^{\dagger} \bpsi_{\T_t}^\top  (\l_t + \bv_t))}_{\text{sparse, data dependent noise}}
\lhat_t &= \l_t + \underbrace{\bv_t - \bm{I}_{\T_t} \left(\bpsi_{\T_t}\right)^{\dagger} \bpsi_{\T_t}^\top \bv_t}_\text{{small, unstructured noise}} - \underbrace{\bm{I}_{\T_t} \left(\bpsi_{\T_t}\right)^{\dagger} \bpsi_{\T_t}^\top  \l_t}_{\text{sparse, data dependent noise}} \nonumber \\
& : = \l_t + \e_t
\end{align}
%
%which gives us the expression provided in equation \eqref{eq:etdef_cent}.
% From this, we notice the following: $\bv_t$ is orthogonal to the best rank-$r$ approximation of the true data, $\l_t$; (b) the second component is dependent on the data, and is also sparse with support $\T_t$.
Thus, recovering $\P_j$ from estimates $\Lhat_j$ is a problem of PCA in sparse data-dependent noise (SDDN): the ``noise'' $\e_t$ consists of two
terms, the first is just small unstructured noise (depends on $\v_t$) while the second is sparse with support $\T_t$ and depends linearly on the true data $\l_t$.  %?? $\l_t$ is true data or $\tl_t$ -- say that
%Observe from \eqref{eq:etdef_cent} that the problem of PCA on $\Lhat_j$ is one of PCA in SDDN. The term in the $\lhat_t$ expression that depends on the true data $\lt$ is sparse with support $\T_t$.
%
We studied PCA-SDDN in detail in \cite{rrpcp_jsait} where we showed the following.
%Suppose that we have $\alpha$ samples of data $\z_i$ satisfying $\z_i = \P \at_i + \w_i + \v_i$ with $\P$ satisfying $\mu$ incoherence and $\at_i$'s satisfying statistical $\mu$-incoherence, $\w_i$ being SDDN and $\v_i$ is small unstructured noise with $\sum_i \l_i \v_i^\top = 0$. We compute $\Phat$ as the top $r$ left singular vectors of $\Z = [\z_1, \z_2, \dots \z_\alpha]$. PCA ($r$-SVD) on observed data that is equal to true data plus SDDN plus small unstructured noise provides an accurate estimate of the true data's principal subspace.
\begin{lem}[PCA-SDDN]\label{cor:cent_pca_dd}
For $i = 1, \cdots, \alpha$, assume that $\z_i = \l_i + \w_i +\v_i$ with $\w_i = \bm{I}_{\T_i} \B_i\l_i$ being sparse, data-dependent noise with support $\T_i$;  $\l_i = \P \at_i$  with $\P$ being an $n \times r$ basis matrix  that satisfies $\mu$-incoherence, and $\at_i$'s satisfy statistical $\mu$-incoherence; and $\v_i$ is small bounded noise with $\lambda_v^+:=\|\E[\v_i \v_i^{\top}]\| < \lambda^-$ and $\max_i \|\v_i\|^2 \leq C r_v  \lambda_v^+$.
Let $q := max_i \|\B_i \P\|$ and let $\bz$ be the maximum fraction of non-zeros in any row of the matrix $[\w_1, \cdots, \w_{\alpha}]$.
Let $\Phat$ be the matrix of top $r$ eigenvectors of $\frac{1}{\alpha} \sum_i \z_i \z_i^{\top}$.
Assume that %$\max_i \|\B_{i} \P\| \leq q$ for a
$q \le 3$. %and  that the fraction of non-zeros in any row of the matrix $[\w_1, \cdots, \w_{\alpha}]$ is bounded by $\bz$.
Pick an $\epsse >0$. If
\begin{align}\label{eq:lam}
7 \sqrt{\bz} q f + \frac{\lambda_v^+}{ \lambda^-}  < 0.4 \epsse, \text{  and}  %+ \sqrt{f \frac{\lambda_v^+}{ \lambda^-}}
\end{align}
%and if $\alpha \ge \alpha^*$ where
\begin{align}\label{eq:alpha}
\alpha \ge \alpha^* := C \max\left( \frac{q^2 f^2}{\epsse^2} r \log n, \frac{\frac{\lambda_v^+}{\lambda^-} f}{\epsse^2} r \log n\right),
\end{align}
then, w.p. at least $1- 10n^{-10}$, $\SE(\Phat, \P) \le \epsse$.
%Additionally, as long as $\alpha \geq \alpha^*$, with probability at least $1 - 10n^{-10}$,
%\begin{align}\label{eq:perturb}
%&\norm{\mathrm{perturb}} := \norm{\frac{1}{\alpha} \sum_i (\z_i \z_i^\top - \l_i\l_i^\top)} \nonumber \\
%%&= \norm{\frac{1}{\alpha} \sum_i (\l_i \et_i^{\top} + \et_i \l_i^{\top} + \et_i \et_i^{\top} + \bv_i \bv_i^{\top} + \l_i \bv_i^{\top} + \bv_i \l_i^{\top} + \bv_i \et_i^{\top} + \et_i \bv_i^{\top})}, \nonumber \\
%&\leq \norm{\frac{1}{\alpha} \sum_i \et_i \et_i^{\top}} + 2 \norm{\frac{1}{\alpha} \sum_i \l_i \et_i^{\top}} + 2 \norm{\frac{1}{\alpha} \sum_i \l_i \bv_i^{\top}} \nonumber \\
%& + 2 \norm{\frac{1}{\alpha} \sum_i \bv_i \et_i^{\top}} + \norm{\frac{1}{\alpha} \sum_i \bv_i \bv_i^{\top}}, \nonumber \\
%&\leq \left(6.6 \sqrt{\bz} q f + 4.4 \frac{\lambda_v^+}{\lambda^-} \right) \lambda^-
%\end{align}
%and
%\begin{align*}
%\lambda_r\left(\frac{1}{\alpha} \sum_i \l_i \l_i^{\top} \right) &\geq  0.99 \lambda^- .
%\end{align*}
%
\end{lem}
This result says that, under the incoherence assumptions, and assuming that the unstructured noise satisfies the stated assumptions, if the support of the SDDN, $\w_i$, changes enough over time so that $\bz$, which is the maximum fraction of nonzeros in any row of the matrix $[\w_1, \w_2, \dots, \w_\alpha]$, is sufficiently small, if the unstructured noise power is small enough compared to the $r$-th eigenvalue of the true data covariance matrix and it is bounded with small effective dimension, $\|\vt_i\|^2/\lambda_v^+ \le Cr$, and if $\alpha$ is large enough, then $\mathrm{span}(\Phat)$, is a good approximation of $\mathrm{span}(\P)$.
Notice here that for SDDN, the true data and noise correlation, $\E[\l_i \w_i{}^\top]$, is not zero, and the noise power, $\E[\w_i \w_i{}^\top]$,  itself is also not small. However, the key idea used to obtain this result is the following: enough support changes over time (small $\bz$) helps ensure that the upper bounds on sample averaged values of both these quantities, $\|(1/\alpha) \sum_i \E[\l_i \w_i{}^\top]\|$ and $\|(1/\alpha) \sum_i \E[\w_i \w_i {}^\top]\|$ are $\sqrt{b}$ times smaller than those on their maximum instantaneous values, $\|\E[\l_i \w_i{}^\top]\|$ and $\|\E[\w_i \w_i{}^\top]\|$.  %THIS POINT NOT VALID IN OUR CASE: When applying this result for our problem,  assuming small enough $\Delta_{tv}$, the instantaneous values of both these are proportional to the recovery error in the previous subspace. By applying this result, we are able to show that the next subspace recovery error is $\sqrt{b} f $ times smaller. Here,  $b = \missfracrow$. Our assumed bound on $\missfracrow$ is what makes this possible.

%There are some minor differences since (i) the result does not consider missing entries; and (ii) we do not assume any model on subspace change; the work of \cite{rrpcp_jsait} assumed a piecewise constant subspace change model.
%The proof uses the above PCA-SDDN result and the following simple facts proved elsewhere.

Our proof uses Lemma \ref{cor:cent_pca_dd} applied on the $j$-th mini-batch of estimates, $\Lhat_j$ along with the following simple facts.
\begin{fact}\label{fact:simp_cent}
\begin{enumerate}
\item From \cite[Remark 3.6]{rrpcp_perf} we have: let $\P$ be an $\mu$-incoherent, $n \times r$ basis matrix. Then, for any set $\T \subseteq [n]$, we have
\begin{align*}
\|\I_{\T}^\top \P\|^2 \leq |\T| \cdot  \frac{\mu r}{n}
\end{align*}
\item For $n \times r$ basis matrices $\P$, $\Phat$ (useful when the column span of $\Phat$ is a good approximation of that of $\P$), and any set $\T \subseteq [n]$, we have
\begin{align*}
\|\I_{\T}^\top \Phat\| \leq \SE(\Phat, \P) + \|\I_{\T}^\top\P\|
\end{align*}
\item For a $\mu$-incoherent $n \times r$ basis matrix, $\P$, and any set $\T \subseteq [n]$,
\[
\lambda_{\min}(\I_\T{}^\top(\I - \P \P^\top) \I_\T)  = 1 - \|\I_T{}^\top \P\|^2  %= \lambda_{\min}(\I - (\I_\T{}\top \P)(\P^\top \I_\T) )
\]
\end{enumerate}
Thus, combining the above three facts,
\[
\|(\I_\T{}^\top(\I - \Phat \Phat^\top) \I_\T)^{-1}\| \le \frac{1}{1 - (\SE(\Phat, \P) + \sqrt{|\T| \mu r / n })^2 }
\]

%the $s$-level Restricted Isometry Constant (RIC) of $\I - \P \P^\top$ satisfies ?? define ric
%\begin{align*}
%\delta_s(\I - \P\P^\top) = \max_{|\T| \leq s} \|\I_{\T}^\top \P\|^2  CITE
%\end{align*}
%\item For any $\T$ such that $|\T| \leq s$, we have
%\begin{align*}
%\norm{(\bpsi_{\T}^\top\bpsi_{\T})^{-1}} \leq \frac{1}{1 - \delta_s(\bpsi)}
%\end{align*}

\end{fact}

{
The proof for $ j=1$ is a little different from $j > 1$. For $j=1$,  $\bpsi = \I$ and $\lhat_t = \y_t$. Also, $i=t$. For $j>1$, $\bpsi = \I - \Phat_{j-1} \Phat_{j-1}{}^\top$ and $i = t - (j-1)\alpha$.
Consider $j=1$ (initialization). In this case, $\lhat_t = \y_t$ satisfies \eqref{eq:etdef_cent} with $\bpsi = \I$. We apply Lemma \ref{cor:cent_pca_dd} with $i=t$,
 $\z_i \equiv \lhat_t  = \y_t$, $\l_i \equiv \l_t$, $\P \equiv \P_1$, $\w_i \equiv -\I_{\T_t}\I_{\T_t}^\top \l_t$, $\v_i \equiv \v_t  -\I_{\T_t}\I_{\T_t}^\top \v_t$, $\B_i \equiv \I_{\T_t}^\top$.  Notice that the fraction of non-zeros in the matrix $[\w_1, \cdots \w_{\alpha}]$ is bounded by $\missfracrow$ and thus $b \equiv \missfracrow$.
To obtain $q$,  we need to bound $\max_{t \in \J_1} \|\B_t \P_1\| = \max_{t \in \J_1}  \|\I_{\T_t}{}^{\top}\P_1\|$. By item 1 of Fact \ref{fact:simp_cent}, $\|\I_{\T_t}^\top\P_1\|^2 \leq |\T_t| \mu r/n \leq  \missfraccol \cdot n \mu r/n$. Under the assumptions of Theorem \ref{thm:central_timevar}, $|\missfraccol| \leq \rhocol/\mu r$ and thus $\max_t \|\B_t \P\| \leq \sqrt{\rhocol} = q_1 \equiv q$. We pick $\epsse = \max(\zz, 0.25 q_1)$.
From the Theorem assumptions (missing entry fractions), $b = \missfracrow \leq \rhorow/f^2$ and $\zz \le 0.2$ and so \eqref{eq:lam} is satisfied.  %and $c \sqrt{\lambda_v^+/\lambda^-} \leq \zz \leq 0.2$), we have
%\begin{align*}
%7 \sqrt{b} qf + \frac{\lambda_{v}^+}{\lambda^-} &\leq 7 q \cdot \sqrt{\rhorow} +  \zz^2 \\
%&\leq 7\sqrt{\rhorow}q + 0.2 \zz  \le 0.4 \epsse_1
%\end{align*}
Furthermore, since $\epsse = \max(\zz,0.25q_1)$, the value of $\alpha$ used in the Theorem satisfies the requirements of Lemma  \ref{cor:cent_pca_dd}. Thus, we can apply this lemma to conclude that $\SE(\Phat_1, \P_1) \leq \epsse = \max(\zz, 0.25 q_1)$ with $q_1 = 0.1 = \sqrt{\rhocol}$. This completes the proof of Theorem \ref{thm:naive} since simple-PCA just repeats this step at each $j$.

%Before we show invoking Lemma \ref{cor:cent_pca_dd} helps us prove Theorem \ref{thm:central_timevar}, we state the following fact that is used in several places in the proof.
%Using the result of Lemma \ref{cor:cent_pca_dd}, and Fact \ref{fact:simp_cent}, we complete the proof of the theorem.

Now consider any $j > 1$. %For the $j$-th interval, we use $q_j$ to denote the upper bound on $\max_{t \in \J_j} \|\B_t \P_j\|$ and we use $\epsse_j$ to denote the chosen value of $\epsse$ when applying the lemma. The value of both $q_j$ and $\epsse_j$ will depend on $\epsse_{j-1}$ which bounds $\SE(\Phat_{j-1}, \P_{j-1}) $.
We claim that for $j>1$, %$\epsse_j$ satisfies the following with $q_1 = 0.1$: $q_j \le
\[
\SE(\Phat_j, \P_j) \le \epsse_j
% \max(\zz, (\zz + \Delta_{tv})\sum_{j'=1}^{j-1}  (0.3)^{j'}, 0.3^j (q_1/4)  + \Delta_{tv} \sum_{j'=1}^{j-1}  (0.3)^{j'} )
\]
with $\epsse_j$ satisfying the following recursion: $\epsse_1 =\max(\zz, 0.25 q_1)$ with $q_1=0.1$, and
\begin{align}
\epsse_j = \max( \ \zz, \ 0.25 \cdot 1.2 \cdot (\epsse_{j-1} + \Delta_{tv}) )
\label{epsj_recursion}
\end{align}

This can be simplified to show that %. This can be simplified to
\begin{align}
\epsse_j  &\le \max(\zz, (\zz + \Delta_{tv})\sum_{j'=1}^{j-1}  (0.3)^{j'}, \nonumber \\
 &0.3^j \cdot 0.25q_1  + \Delta_{tv} \sum_{j'=1}^{j-1}  (0.3)^{j'} ) \nonumber \\
&\le \max( \zz,  0.3^j (0.25q_1)  + \Delta_{tv} \sum_{j'=1}^{j-1}  (0.3)^{j'} )
%\max(\zz, (\zz + \Delta_{tv})\sum_{j'=1}^{j-1}  (0.3)^{j'}, 0.3^j (0.25q_1)  + \Delta_{tv} \sum_{j'=1}^{j-1}  (0.3)^{j'} )
\label{epsj_eq}
\end{align}
where the second inequality follows by using $\Delta_{tv} \leq \zz$ and $\sum_{j'=1}^{j-1} (0.3)^{j'} \leq \sum_{j'=1}^{\infty} (0.3)^{j'} = 3/7$. 

%For analyzing the $j$-interval, %we use induction. The base case $j=1$ has been proved above. For the induction step,
%assume that $\SE(\Phat_{j-1}, \P_{j-1}) \le \epsse_{j-1}$. %: = \max(\zz, q_{j-1}/4)$.
%start with the assumption that $\SE(\Phat_{j-1}, \P_{j-1}) \le \epsse_{j-1}$

To prove the above claim, we use induction. Base case: $j=1$ done above.
Induction assumption: assume $\SE(\Phat_{j-1}, \P_{j-1}) \le \epsse_{j-1} $. %:= \max(\zz, (\zz + \Delta_{tv})\sum_{j'=1}^{j-2}  (0.3)^{j'}, 0.3^{j-1} (q_1/4)  + \Delta_{tv} \sum_{j'=1}^{j-2}  (0.3)^{j'} )$.
The application of the PCA-SDDN lemma is similar to that for $j=1$ with the difference being that $i = t - (j-1) \alpha$ and $\B_i$ is different now. We now have $\B_i \equiv (\bpsi_{\T_t}^\top\bpsi_{\T_t})^{-1} \bpsi_{\T_t}^{\top}$ and so $\max_{t \in \J_j} \|\B_t \P\| = \max_t \|(\bpsi_{\T_t}^\top\bpsi_{\T_t})^{-1} \bpsi_{\T_t}^{\top}\P_j\|$.
This can be bounded using Fact \ref{fact:simp_cent} as follows %with assuming $\SE(\Phat_{j-1},\P_{j-1}) \le \epsse_{j-1}$.
\begin{align*}
&\max_t \|(\bpsi_{\T_t}^\top\bpsi_{\T_t})^{-1} \bpsi_{\T_t}^{\top}\P_j\|  \\
& \leq \max_t \|(\bpsi_{\T_t}^\top\bpsi_{\T_t})^{-1}\| \|\I_{\T_t}^\top \| \| \bpsi\P_j\| \\ %\leq 1.2 \cdot 1 \cdot \SE(\Phat_{j-1},\P_j) \\
& \le \frac{1}{1 - (\epsse_{j-1} + \sqrt{0.01})^2 }  \cdot 1 \cdot \SE(\Phat_{j-1},\P_j) \\
& \le \frac{1}{1 - (\epsse_{j-1} + \sqrt{0.01})^2 }  (\epsse_{j-1} + \Delta_{tv}):= q_{j}
%&\leq  1.2 (\epsse_{j-1} + \Delta_{tv}) := q_{j}
\end{align*}
Using \eqref{epsj_eq}, $\epsse_{j-1} \le \max(\zz,  0.25q_1 + \Delta_{tv}(3/7))$ and recalling that $ \max(\zz, 0.35\sqrt{\rhocol}+ \Delta_{tv} (3/7) ) < 0.3$. %This follows by using $j-2 < \infty$, and $\zz < 0.2$, and $\Delta_{tv}< 0.1$ (from Theorem assumptions).
%(\zz+\Delta_{tv})(3/7) ,
Using this upper bound on $\epsse_{j-1}$ in the denominator expression of above,
\begin{align}
q_j \le 1.2 (\epsse_{j-1} + \Delta_{tv})
\label{qj_bnd}
\end{align}
%?? to ensure $\epsse_{j-1}<0.1$, we need to assume $\zz < 0.1/5$  and $Deltatv < 0.1/5$ so that
Apply the PCA-SDDN lemma with $q \equiv q_j$ and $\epsse = \max(\zz, 0.25 q_j)$.  With this choice of $\epsse$, it is easy to see that
$7 \sqrt{b} q_j f + \frac{\lambda_{v}^+}{\lambda^-} \le 0.4 \epsse$. Also, $\alpha$ given in the Theorem again satisfies the requirements of the lemma. Applying the PCA-SDDN lemma, and using \eqref{qj_bnd} to bound $q \equiv q_j$,
\begin{align*}
\SE(\Phat_j, \P_j) & \le \max(\zz, 0.25 q_j)  \\
& \le  \max(\zz, 0.25 \cdot 1.2 (\epsse_{j-1} + \Delta_{tv}) ) = \epsse_j  %& = \frac{1}{1 - (\epsse_{j-1} + \sqrt{0.01})^2 }  ( \Delta_{tv} + 1.2 \epsse_{j-1} )
\end{align*}
This proves our claim.
}

\subsubsection{Bounds on error in estimating $\tl_t$}
From \eqref{eq:etdef_cent}, $\lhat_t - \tl_t =-  \I_{\T_t} (\bpsi_{\T_t}^\top \bpsi_{\T_t})^{-1}  \I_{\T_t}^\top \bpsi \tl_t$ with $\bpsi = \I - \Phat_{j-1} \Phat_{j-1}{}^\top$ for $t \in \J_j$. Using this,  $\tl_t  = \l_t + \v_t = \P_j \at_t + \v_t$,  and Fact \ref{fact:simp_cent}, we can get
\[
\|\lhat_t -\tl_t\| \leq \SE(\Phat_{j-1}, \P_j) \|\l_t\| + \|\v_t\| \le (\epsse_{j-1} + \Delta_{tv})\|\l_t\| + \|\v_t\|
\]
Using the same approach that we used to derive \eqref{eq:etdef_cent}, we get that $\tlhat_t - \tl_t$ has the same expression as $\lhat_t - \tl_t$ but with $\bpsi = \I - \Phat_{j} \Phat_{j}{}^\top$ for $t \in \J_j$. Thus,
\[
\|\tlhat_t -\tl_t\| \leq \SE(\Phat_{j}, \P_j) \|\l_t\|_2 + \|\v_t\| \le \epsse_{j-1}\|\l_t\| + \|\v_t\|
\]

\subsection{Proof of Theorem \ref{thm:large_ss}}
The proof again follows by using the PCA-SDDN lemma given above along with use of Fact \ref{fact:simp_cent}. The main difference is the use of the following idea.

Consider the interval just before the subspace change, i.e., the $j$-th interval with $j = j_\gamma - 1$. At this time, by our delay assumption, $\SE(\Phat_j,\P_j) \le 2\zz$ and thus, using Fact \ref{fact:simp_cent}, $\|\I_{\T_t}{}^\top\Phat_j\| \le 2\zz + 0.1$. Also, using Fact \ref{fact:simp_cent},
\begin{align*}
&\max_t \|(\bpsi_{\T_t}^\top\bpsi_{\T_t})^{-1} \bpsi_{\T_t}^{\top}\P_j\|  \\
& \leq \max_t \|(\bpsi_{\T_t}^\top\bpsi_{\T_t})^{-1}\| \|\I_{\T_t}^\top  \bpsi\P_j\| \\ %\leq 1.2 \cdot 1 \cdot \SE(\Phat_{j-1},\P_j) \\
& \le \frac{1}{1 - (2\zz + 0.1)^2 }  \cdot (\|\I_{\T_t}^\top \Phat_{j-1}\| +  \|\I_{\T_t}^\top \P_{j}\| ) \\
& \le \frac{1}{1 - (2\zz + 0.1)^2 }  \cdot  ((0.1 + 2\zz) +  0.1 ) %\\
\end{align*}
Combining with the bound from the previous section, the final bound for this term is
\begin{align*}
%&\max_t \|(\bpsi_{\T_t}^\top\bpsi_{\T_t})^{-1} \bpsi_{\T_t}^{\top}\P_j\| \le  \\
& \frac{\min(\SE(\Phat_{j-1}, \P_j), ((0.1 + 2\zz) +  0.1) )}{1 - (2\zz + 0.1)^2 }
\end{align*}

\newcommand{\g}{\bm{g}}

\Section{Fedrated Over-Air Robust ST-Miss} \label{sec:subtrack}

In this section, we study robust ST-miss in the federated, over-air learning paradigm. There are two important distinctions with respect to the centralized ST-miss problem from Sec. \ref{sec:stmiss} namely (a) data is now available across different nodes and the proposed algorithm must obey the federated data sharing constraints and (b) the proposed algorithm must be able to deal with gross and sparse outliers. 

{ An example where such a problem formulation is valid is as follows. Consider the recommendation system design problem. Assume that there are $n$ products and a total of $d$ users/buyers distributed across a geographical area. The ``products'' could be movies, news sites, Facebook pages, blogs or even survey questions.  A subset of $d_k$ users sends their ``ratings'' of these products to worker node $k$. There are a total of $K$ worker nodes.
The master node  would like to compute a low-dimensional subspace approximation of the $n \times d$ ratings' matrix, denoted by $\Y$, in order to use this information to recommend relevant movies to them. 
%In Sec. \ref{sec:fedpm}, we first solve this problem in the easier setting of complete data and the data being static. This is a valid assumption when considering survey data analysis for example. 
%In Sec. \ref{sec:subtrack}, we consider the more difficult, but also more practical, setting where some of the entries can be missing. 
Note that the dataset is also potentially dynamic; every day new users enter the system and provide more ratings of the movies or the news sites or blogs. % If the site has time-varying content, an old user may also change their ratings of it. 
Thus, at time $t$, across all users, we  have an $n \times \alpha$ data matrix $\Y_{(t)}$. This typically has many missing entries (set to zero), and gross outliers (that arise either from unintentional rating mistakes, or presence of malicious users). Collating all such matrices together we have a very big $n \times d$ matrix with $d = t \alpha$ at time $t$. 
The goal is to track the underlying true data subspace at each time $t$; this could be fixed or slow time varying. The assumption here is that user preferences are actually governed by a small number of factors $r$; this number is much smaller than the number of products $n$ or the total number of users $d$.
}

A key observation that allows us to build upon Sec. \ref{sec:stmiss} is that only Line $10$ of Algorithm \ref{algo:norst_nodet} needs to be federated (all other operations are performed locally on each vector). To this end, we first explain why tackling iteration noise is sufficient to satisfy the Fed-OA constraints in Sec. \ref{sec:fl_explain}, we then present our result for PCA in the Fed-OA setting in Sec. \ref{sec:fedpm} (federated version of Line $10$ of Algorithm \ref{algo:norst_nodet}), and finally show how this is used to develop an algorithm that solves Robust ST-Miss in the Fed-OA setting in Sec. \ref{sec:fed_rst}.

%There are two critical differences with respect to the problem studied in Sec. \ref{sec:stmiss}: firstly, in this sectting, the data is not available to one ``server'' but is rather distributed across several local nodes and we attempt to design an algorithm that obeys the federated over-air data sharing constraints (described next) for reasons such as communication efficiency, and data privacy; the second is the presence of outliers in the partially observed data. To keep the paper compact, we address both these differences in this section. The first key observation is that even though the data is distributed across different nodes, only Line $8$ of Algorithm \ref{algo:norst_nodet} needs to be ``federated'' since this is the only global operation (operation that requires data from all nodes). Thus, we first derive a guarantee for this step in  and then carefully integrate this into the (R)-STMiss problem in Sec. \ref{sec:fed_rst}.

\subsection{Dealing with mild asynchrony and channel fading} \label{sec:fl_explain}

As discussed previously, the three key challenges while working with over-air aggregation are (a) small timing mismatches, (b) channel fading, and (c) iteration noise. There exist a plethora of techniques within physical layer communications for dealing with channel fading and mild asynchrony. The main idea is to use carefully designed pilot sequences. Pilot sequences are symbols that the transmitter-receiver pairs agree on in advance and are transmitted in the beginning of a data frame. %Thus, the central node knows what to expect in the ideal scenario in the initial part of a frame.
For instance, suppose that there are only $K=2$ transmitters and the relative offsets between the transmitters is at most $j$ symbols. In this case, both transmitters can use pilot sequences of length $2j+1$, $ [a_1, a_1, \dots, a_1]$ and $ [a_2, a_2, \dots, a_2]$ respectively. Since the offset is at most $j$, the central node receives at least one symbol with values $a_1 + a_2$. It can determine the relative offset by determining the start location of the value $a_1 + a_2$. Once the estimated offset is communicated back to the nodes, the center can then receive the correct sum by having the nodes appropriately zero pad their transmissions.
%having each of the two nodes transmit each scalar multiple times.
Extensions of these ideas can be utilized to handle the case of $K > 2$ nodes.
%$j+1$ consecutive symbols
%
Similarly, { some amount of} channel fading can compensated for by estimating the fading coefficients which can be done since the values of the pilot symbols are assumed to be known. These techniques are by now quite well-known in the single and multiple antenna scenarios \cite{wireless_comm}. { As correctly noted by anonymous reviewer, it may be impossible to compensate for a very weak channel gain since that would require a transmit power that's above the limit.}  Thus, the main problem to be addressed is iteration noise which is the focus of this paper.

\subsection{Federated Over-Air PCA  via the Power Method (PM)}\label{sec:fedpm}
Here we provide a result for subspace learning while obeying the federated data sharing constraints.

\subsubsection{Problem setting}
The goal of PCA (subspace learning) is to compute an $r$-dimensional subspace approximation in which a given data matrix $\Z \in \R^{n \times d}$ approximately lies.
The $k$-th node observes a columns' sub-matrix $\Z_k \in \R^{n \times d_k}$. We have $\Z := [\Z_1, \cdots, \Z_k, \cdots, \Z_K] \in \R^{n \times d}$ with $d = \sum_{k=1}^K d_k$ and the goal of PCA is to find an $n \times r $ basis matrix $\U$ that minimizes $\|\Z - \U \U^\top \Z\|_F^2 $. As is well known, the solution, $\U$, is given %by the top $r$ left singular vectors of $\Z$, or equivalently,
by the top $r$ eigenvectors of $\Z \Z^\top$.  Thus the goal is to estimate the span of $\U$ in a federated over-air (FedOA) fashion.

%\subsubsection{Federated Data Sharing constraints} The federated over-air setup imposes the following constraints. An algorithm iteration $l$, the central server can broadcast an $n \times r$ matrix, denoted $\Qhat_{l-1}$ to all the nodes. Each node uses this estimate and its available data to compute its local estimate, denoted $\tilde\U_{k,l}$. The nodes synchronously broadcast these to the central server but the transmission is corrupted by additive channel noise, i.e. the central server receives $\tilde\U_l:= \sum_{k=1}^K \tilde\U_{k,l} + \W_l$.  We assume that each entry of the channel noise matrix $\W_l$ is i.i.d. Gaussian, zero-mean with variance $\sigma_c^2$. The central server processes $\Uhat_l$ to get $\Qhat_l$ and broadcasts it to all $K$ nodes for next iteration.

\subsubsection{Federated Over-Air Power Method (FedOA-PM)}
The simplest algorithm for computing the top eigenvectors is the Power Method (PM) \cite{golub89}.
The distributed PM is well known, but most previous works assume the iteration-noise-free setting, e.g., see the review in \cite{distpca_review}. On the other hand, there is recent work that studies the iteration-noise-corrupted PM \cite{noisy_pm,improved_npm} but in the centralized setting. In this line of work, the authors consider two models for iteration-noise. The noise could either be deterministic, or statistical noise could be added to ensure differential privacy. Our setting is easier than the deterministic noise model, since we assume a statistical channel noise model, but is harder than the privacy setting since we do not have control over the amount of noise observed at the central server (here use the term channel noise and iteration-noise interchangeably).

%Assume that there are $K$ distributed worker or peer nodes and one central server. Assume that node $k$ observes the local data matrix $\Z_k \in \R^{n \times d_k}$, and let $\Z := [\Z_1, \Z_2, \cdots, \Z_K] \in \R^{n \times d}$ with $d = \sum_k d_k$ denote the complete data matrix.

The vanilla PM estimates $\U$ by iteratively updating $\Uhat_l = \Z \Z^\top \hat\U_{l-1}$ followed by QR decomposition to get $\hat\U_l$. FedOA-PM approximates this computation as follows.
At iteration $l$, each node $k$ computes
$
\tU_{k,l} := \Z_k \Z_k^{\top} \Qhat_{l-1}
$
and synchronously transmits it to the central server which receives the sum corrupted by channel noise, i.e., it receives% Thus, %at every iteration $l$, instead of receiving $\sum_k \Z_k \Z_k^{\top} \Qhat_{l-1} = \Z \Z^{\top} \Qhat_{l-1}$, the central server receives
\begin{align*}
\Uhat_l:= \sum_{k=1}^K \tU_{k,l} +  \W_l %= \sum_k \Z_k \Z_k^{\top} \Qhat_{l-1}
=  \Z \Z^{\top} \Qhat_{l-1} + \W_l.
\end{align*}
since $\sum_k \Z_k \Z_k^\top =  \Z \Z^{\top}$.
Here $\W_l$ is the channel noise. It then computes a QR decomposition of $\tilde\U_l$ to get a basis matrix $\Qhat_l$ which is broadcast to all the $K$ nodes for use in the next iteration.
%
%either at every iteration $l$ or after every $\eta$ iterations \footnote{For simplicity, we only consider the $\eta = 1$ scenario. The general algorithm and relevant analysis is provided in the appendix.}. The latter helps improve noise robustness but may violate maximum transmit power constraint??. This is broadcast back to all the user nodes for use in the next iteration.
We summarize this complete FedOA-PM algorithm in Algorithm \ref{algo:rankr}. If no initialization is available, it starts with a random initialization. When we use FedOA-PM for subspace tracking in the next section, the input will be the subspace estimate from the previous time instant. %(if this is a good estimate of the current subspace, it will mean FedPM needs fewer iterations to converge).

%\subsubsection{FedOA-PM: top eigenvalue computation}
%Using the final estimate of the subspace, we can also computes an estimate, $\hat\sigma_1$, of the top eigenvalue of $\Z \Z^\top$, denoted $\sigma_1$, in a FedOA fashion. This relies on the fact that the top eigenvalue of $\Qhat_L^\top (\sum_k \Z_k \Z_k^\top + \W) \Qhat_L$  is a good approximation of $\sigma_1$ if the noise $\W$ is small enough and span of $\hat\U_L$ is a good approximation of the span of $\U$. %We prove this fact below in our guarantee.
%This is implemented in lines 10-13 of Algorithm \ref{algo:rankr}.

\begin{algorithm}[t!]
%\vspace{-.5cm}
\caption{FedOA-PM: Federated Over-Air PM}\label{algo:rankr}
  \begin{algorithmic}[1]
  		%\Procedure {FedOA-PM}{$\Z$, $r$, $L$, $\Qhat_0$}
        \Require $\Z$ (data matrix), $r$ (rank), $L$ (\# iterations), $\Qhat_0$ (optional initial subspace estimate)  %, $\calI_k$
			\State $K$ nodes, $\Z_k \in \R^{n \times d_k}$ local data at $k$-th node.
            \State If no initial estimate provided, at central node, do $\Uhat_0 \overset{i.i.d.}{\sim} \mathcal{N}(0, I)_{n \times r}$; $\Qhat_0 \leftarrow \Uhat_0$, transmit to all $K$ nodes.
  	\For {$l =1,\dots, L$}
  	 \State At $k$-th node, for all $k \in [K]$, compute $\tilde{\U}_{k,l} = \Z_{k} \Z_{k}^{\top} \Qhat_{l-1} $
  	 \State All $K$ nodes transmit $\tilde{\U}_{k,l}$ synchronously  to central node.%
    %\STATE central node receives the sum of the $K$ transmissions corrupted by channel noise, denotes $\Uhat_l$ which satisfies $\Uhat_\tau := \sum_k \tilde{\U}_{k, l} + \W_{k,l }$, where $\sum_k \W_{k,l } = \W_{l}$ is channel noise.
%    \State Central node receives $\Uhat_l := \sum_k \tilde{\U}_{k, l} + \W_{k,l }$, with $\sum_k \W_{k,l } = \W_{l}$.
\State Central node receives $\Uhat_l := \sum_k \tilde{\U}_{k,l} +  \W_{l}$.
	\State Central node computes $\Qhat_l \bm R_l \overset{QR}{\leftarrow} \U_l$
    \State Central node broadcasts $\Qhat_l$ to all nodes
    \EndFor
    \State At $k$-th node, compute $\tilde{\U}_{k,L+1} = \Z_{k} \Z_{k}^{\top} \Qhat_{L} $
 \State All $K$ nodes transmit $\tilde{\U}_{k,L+1}$ synchronously  to the central node.   %All $k$ nodes compute $\Z_k \Z_k^{\top}\Qhat_{L}$, transmit synchronously to central node
    \State Central node receives $\tilde\U_{L+1}:= \sum_k \tilde{\U}_{k,L+1} + \W_{L+1}$
\State Central node computes $\hat\Lam= \Qhat_{L}^{\top} \tilde\U_{L+1}$ and its top eigenvalue, $\hat{\sigma}_1 = \lambda_{\max} (\hat\Lam)$.
   \Ensure $\Qhat_{L}$, $\hat{\sigma}_1$.
   %\EndProcedure
  \end{algorithmic}
\end{algorithm}

We use $\sigma_i$ to denote the $i$-th largest eigenvalue of $\Z \Z^{\top}$, i.e., $\sigma_1 \geq \sigma_2 \geq \cdots \sigma_n \geq 0$. We have the following guarantee for Algorithm \ref{algo:rankr}.
%We use the sine of the maximum principal angle to quantify the distance between subspaces \cite{chordal_dist}.
%For two $r$-dimensional subspaces with basis matrices, $\U_1$, $\U_2$, this is computed as
%$
%\SE(\U_1, \U_2) = \|(\I - \U_1 \U_1{}^{\top}) \U_2 \|
%$.
%Here and below $\|.\|$ denotes the induced 2-norm of a matrix and $\|.\|_F$ denotes its Frobenius norm. We will often reuse the letter $C$ to denote different numerical constants in each use.  %Since $\U_1$ and $\U_2$ have the same dimensions, $\SE(\U_1, \U_2) = \SE(\U_2, \U_1)$.
%
%.  %Also, define the following quantities: the ratio of $(r+1)$-th to $r$-th eigenvalue, $R : = \sigma_{r+1}/\sigma_r$, the noise to signal ratio, $\nsrmax := \sigma_c/\sigma_r$, and $\tilde{R} := \max(R, 1/\sigma_r)$. Thus we have the following main result:

%Consider Algorithm \ref{algo:rankr} with initial subspace estimation error $\SE_0$.

\newcommand{\rat}{R} %{\mathrm{ratio}}

\begin{lem}[FedOA-PM]\label{lem:fed_app}
Consider Algorithm \ref{algo:rankr}. Pick the desired final accuracy $\epsilon \in (0,1/3)$.
Assume that, at each iteration, the channel noise $\W_l \overset{i.i.d.}{\sim} \mathcal{N}(0, \sigma_c^2)$ with
(i) $\sigma_c < \epsilon  \sigma_r/(5 \sqrt{n})$ and  (ii) $\rat:= \sigma_{r+1}/\sigma_r < 0.99$.
%
%Pick the number of iterations, %$L$ as follows

When using random initialization, if the number of iterations,
%\begin{align*}
$
L = \Omega \left( \frac{1}{\log(1/R)} \log \left(\frac{nr}{\epsilon}\right)  \right).
$
%\end{align*}
then, with probability at least $0.9 - L \exp(-cr)$, $\SE(\U, \Qhat_L) \le \epsilon$.

When using an available initialization with $\SE(\Qhat_0,\U)< \epsilon_0$, if $L = \Omega\left( \frac{1}{\log(1/R)} \log \left(\frac{1}{\epsilon \sqrt{1 - \epsilon_0^2}}\right)  \right)$, then, with probability at least $1 - L \exp(-cr)$, $\SE(\U, \Qhat_L) \le \epsilon$.
\end{lem}

Lemma \ref{lem:fed_app} is similar to the one proved in \cite{noisy_pm, improved_npm} for private PM but with a few key differences which we discuss in the Supplementary Material (Appendix \ref{sec:proof_fedpm}) due to space constraints. We also provide a guarantee for the convergence of the maximum eigenvalue (Lines $10-13$ of Algorithm \ref{algo:rankr}) below.  

\begin{lem}[FedOA-PM: Maximum eigenvalue]\label{lem:fed_app_eval}
Let $\sigma_i$ be the $i$-th largest eigenvalue of $\Z \Z^{\top}$.
Under the assumptions of Lemma \ref{lem:fed_app}, $\hat{\sigma}_1$ computed in line $13$ of Algorithm \ref{algo:rankr} satisfies
\begin{align*}
(1-4\epsilon^2) \sigma_1 - \epsilon^2 \sigma_{r+1}  - \epsilon \sigma_r  \leq \hat{\sigma}_1   \leq (1 + \epsilon)\sigma_1
\end{align*}
\end{lem}

%
%\begin{enumerate}
%\item $\SE(\U, \Qhat_L) \le \epsilon$.
%\item $\hat{\sigma}_1$ computed in line $13$ satisfies
%\begin{align*}
%\sigma_1(1-4\epsilon^2) - \sigma_{r+1} \epsilon^2 - \sigma_r \epsilon \leq \hat{\sigma}_1   \leq (1 + \epsilon)\sigma_1
%\end{align*}
%%where $\lambda_i$ is the $i$-th largest eigenvalue of $\Z\Z^{\top}$.
%%Even if $R = 1$, the upper bound still holds but the lower bound does not.
%\end{enumerate}
%?? should we make the two items two lemmas
%\end{lem}

%\begin{corollary}[Eigenvalue convergence]\label{cor:eig}
%
%\end{corollary}
%Observe that the lower bound on $\hat\sigma_1$ is positive: it can be further lower bounded by  $(1-4\epsilon^2 - \epsilon)\sigma_r - \epsilon^2\sigma_{r+1}  > (1 - 4.99 \epsilon^2 -  \epsilon)\sigma_r$ using the bound on $\rat$. %This is why the above approach  does not require an assumption on the gap between its first and second eigenvalues. Just assuming gap between $r$-th and $(r+1)$-th eigenvalues is enough.

 %Notably, we consider a modification of Algorithm \ref{algo:rankr} where we only perform the QR decomposition every few iterations instead of at every iteration. The Algorithm is summarized in Algorithm \ref{algo:rankr_eta} and the main result is provided in Theorem \ref{thm:main_res}.
To our best knowledge, the Lemma \ref{lem:fed_app_eval} has not been proved in earlier work. This result is useful because thresholding the top eigenvalue of an appropriately defined matrix is typically used for subspace change detection, see for example \cite{rrpcp_tsp19}. The proof of Lemma \ref{lem:fed_app_eval} given in Supplementary Material requires use of Weyl's inequality and the careful bounding of two error terms.

{\bf Note:} The reason we obtain a constant probability $0.9$ in the Lemma \ref{lem:fed_app} is as follows: for any given $r$-dimensional subspace, $\U$ and a random Gaussian matrix $\Qhat$, the matrix $\Qhat^\top\U $ is an $r \times r$ random Gaussian matrix with independent entries. The singular values of $\Qhat^\top\U$ equal the cosine of the $r$ principal angles between $\Qhat_0$ and $\U$. For successfully estimation (through {\em any} iterative method) it is necessary that none of the principal angles are $\pi/2$. To ensure this, we need to lower bound the smallest singular value of $\Qhat^\top\U$. This is difficult because the smallest singular value of square or ``almost'' square random matrices can be arbitrarily close to zero \cite{smallest, smallest_rect}. %; we use the bound from this work.
%The only solution is to tradeoff a large enough lower bound with lower probability of success.
The same issue is also seen in \cite{noisy_pm, improved_npm} \footnote{These papers also provide a more general result that allows one to compute an $r'$-dimensional subspace approximation for an $r' > r$. If $r'$ is picked sufficiently large, e.g., if $r' = 2r$, then the guarantee holds with probability at least $1 - 0.1^r$.}.
In fact, this is an issue for any randomized algorithm for estimating only the top $r$ singular vectors (without a full SVD), e.g., see \cite{musco,streamingpca, streamingpca_oja}. %{\color{red} this following statement is actually wrong, and it only works for $r=1$ case. should delete this. also need some writing changes since we do not really discuss anything in the appendix etc. should we include proof of fedpm result etc?? i am including this here just for now} This is not a limitation since it is possible to decrease the failure probability to any arbitrary $\delta >0$ by running $\log(1/\delta)$ independent copies of the algorithm and computing the geometric median as explained in \cite{geometric_median} (this can be computed in near linear time). %This issue is never discussed in other works on low-rank matrix recovery many of which use $r$-SVD as an initialization or an intermediate step because $r$-SVD is assumed to be an available block.?? edit lang.
%for the above reason, for {\em any} algorithm that is used to compute the top $r$ singular vectors, the probability of success is lower bounded by a constant

%This is done to keep the paper compact. The centralized RST-miss problem and the algorithm and guarantee for it provided in the Supplementary Material. 

We next define the federated over-air robust subspace tracking with missing entries (Fed-OA-RSTMiss) problem, and show how Algorithm \ref{algo:rankr} and Lemma \ref{lem:fed_app} is used to solve Fed-OA-RSTMiss.  

\subsection{Fed-OA-RSTMiss: Problem setting} \label{sec:fed_rst} %We use a time index $t$ to reference the data matrix at time $t$. Also,
%An incomplete, outlier-corrupted, and noisy data matrix $\Y_{k,t}$, is available at node $k$ at time $t$.

%?? Relate $\Omega$ and $\T_{i,t}$. One suggestion: remove the Omega part completely -- ??pn: have defined, i think removing will complicate things even more

In this section, we use $\alpha_k$ to denote the number of data points at node $k$ at time $t$ and $\alpha := \sum_k \alpha_k$ to denote the total number at time $t$. We do this to differentiate from $d$ (in Sec. \ref{sec:fedpm}) which is used to indicate the total number of data vectors. Thus, at time $t$, $d = t \alpha$ and $d_k = t \alpha_k$. At time $t$ and node $k$, we observe a possibly incomplete and noisy data matrix $\Y_{k,t}$ of dimension $n \times \alpha_k$ with the missing entries being replaced by a zero. This means the following: let $\tL_{k,t}$ denote the unknown, complete, approximately low-rank matrix at node $k$ at time $t$.
Then
\begin{align*}
\Y_{k,t} &= \proj_{\Omega_{k,t}}(\tL_{k,t} + \G_{k,t})  =\proj_{\Omega_{k,t}}(\tL_{k,t}) + \S_{k,t}
\end{align*}
where $\G_{k,t}$'s are sparse outliers and $\S_{k,t} := \proj_{\Omega_{k,t}}(\G_{k,t})$, and $\proj_{\Omega_{k,t}}$ sets entries outside the set $\Omega_{k,t}$ to zero.
The full matrix available from all nodes at time $t$ is denoted $\Y_{t}:=[\Y_{1,t}, \Y_{2,t}, \dots, \Y_{K,t}]$. This is of size $n \times \alpha$. The true (approximately) rank-$r$ matrix $\tL_t$ is similarly defined.
Define the index sets $\calI_{1,t}:=[1,2, \dots, \alpha_1]$, $\calI_{2,t}:=[\alpha_1+1, \alpha_1+2, \dots, \alpha_1+\alpha_2]$ and so on. Denote the $i$-th column of $\Y_t$ by $\y_i$, $i=1,2,\dots, \alpha$. And with slight abuse of notation, we define
(the matrix binary masks) $\Omega_{1,t} := [(\T_{1, t})^c, (\T_{2, t})^c, \cdots, (\T_{\alpha_1, t})^c]$, $\Omega_{2,t} := [(\T_{\alpha_1 + 1, t})^c, (\T_{\alpha_1 + 2, t})^c, \cdots, (\T_{\alpha_1 + \alpha_2, t})^c]$ and so on where $\T_{i,t}$ is the set of missing entries in column $i$ of the data matrix at time $t$, $(\T_{i,t})^c$ is its complement w.r.t $[n]$. Thus, the observations 
%
%%incomplete data matrix at node $k$ at time $t$ and
%%A second, and more general problem we study is that of subspace learning and tracking from incomplete data.
%%Assume that at each time $t = 1, 2, \cdots, T_{\max}$, the $k$-th node observes $\alpha_k$ possibly incomplete $n$-dimensional data points.
%Without loss of generality, consider the index set $\calI_{k,t} := \left\{\sum_{j=0}^{k-1}\alpha_{j} + 1, \sum_{j=0}^{k-1}\alpha_{j}  + 2, \cdots, \sum_{j=0}^{k}\alpha_{j}\right\}$ (the data matrix is contiguous across nodes). We also define $\alpha_0 := 0$. Thus, at each time time, we have a $n \times \alpha$ dimensional data matrix $\Y_{t} = [\Y_{1, t}, \Y_{2,t}, \cdots, \Y_{K,t}]$ where $\alpha := \sum_k \alpha_k$. In the context of a distributed recommendation system, node $k$ gets data from a new set of $\alpha_k$ users at each time $t$ (or an independent set of ratings of from an old user). Here $t$ can be every hour or every day depending on the particular application. Of course, $\alpha_k$ could also change with time but for notational simplicity, we let it be fixed.
%We denote the ``true'' matrix of user preferences by $\L_{t}$. We would like to learn its column span at each time $t$, or every so often,  in order to be able to recommend relevant products to them. Let $\y_i$ denote the $i$-th column $i$ of the matrix $\Y_{k,t}$ (recall that $i \in \calI_{k,t}$).
%We use  $\T_i$ to denote the set of missing entries in it, so that $(\T_i)^c$ (complement set of $\T_i$ w.r.t. $[n]$) is the set of observed entries whose values are set to $0$. Then $\y_i$ satisfies
satisfy
\begin{align}
\y_i = \proj_{\T_{i,t}^c}(\tl_i) + \s_i, \quad i \in \calI_{k,t},  \quad k \in [K] %?? make this i=1,2 \dots, \alpha
\end{align}
%where %$\proj_{\T}$ is a binary mask that selects entries in $\T$ and sets the remaining entries to zero; 
where $\s_i$ are sparse vectors with support $\T_{\sparse, i}$. Notice that it is impossible to recover $\g_{i}$ on the set $\T_{i,t}$ and so by definition, $\T_{\sparse, i}, \T_{i,t}$ are disjoint. Let $\P_t$ denote the ($n \times r$ dimensional) matrix of top $r$ left singular vectors of $\tL_t$. In general, our assumptions imply that $\tL_t$ is only approximately rank $r$. As done in our result for ST-miss (in a centralized setting), we define the matrix of the principal subspace coefficients at time $t$ as $\A_t := \P_t^{\top}\tL_t$, the rank-$r$ approximation, $\L_t := \P_t\P_t^{\top}\tL_t$ and the ``noise'' orthogonal to the $\mathrm{span}(\P_t)$ as $\V_t := \tL_t - \L_t$. With these definitions, for all $i \in \calI_{k,t}$ and $k \in [K]$, we can equivalently express the measurements as follows
\begin{align*}
\yt_i &= \proj_{\T_{i,t}^c}(\tl_i) + \s_i \\
&= \tl_i - \I_{\T_{i,t}}\I_{\T_{i,t}}^\top \tl_i + \s_i \\
&:= \tl_i + \z_i + \s_i \\
&= \l_i + \z_i + \s_i + \bv_i
\end{align*}
The goal is to track the subspaces $\P_t$ quickly and reliably, and hence also reliably estimate the columns of the rank $r$ matrix $\L_t$, under the FedOA constraints given earlier. Our problem can also be understood as a dynamic (changing subspace) version of robust matrix completion \cite{normc}.%This needs to be done while respecting the FedOA constraints described earlier.

\subsection{Algorithm}
%?? need to give main idea of algorithm first so the need for FedPM is understood. 

%The overall algorithmic idea is also a careful combination of Algorithm \ref{algo:rst_miss} that studies RST-Miss in the centralized setting and Algorithm \ref{algo:rankr} that studies federated over-air subspace learning. The complete algorithm is summarized in Algorithm \ref{algo:fed_nodet_given_init}. Recall that at $t=1$, we are given an initialization $\Phat_1$ that is a ``good enough'' estimate of the top-$r$ left singular vectors of $\tL_1$, $\P_1$. 

The overall idea of the solution is similar to that for ST-miss. The algorithm still consists of two parts: (a) obtain {\em an} estimate of the columns $\tL_t$ using the previous subspace estimate $\Phat_{t-1}$; and (b) use this estimated matrix $\Lhat_t$ to update the subspace estimate, i.e., obtain $\Phat_t$  by $r$-SVD. The algorithm can be initialized via $r$-SVD (as done in ST-miss) if we assume that $\Y_1$ (the set of data available at $t=1$) contains no outliers and if not, one would need to use a batch RPCA approach such as AltProj \cite{altproj} to obtain the initial subspace estimate $\Phat_1$. %In the main paper, simply assume that we are given a good initialization, $\Phat_1$ s.t. $\SE(\Phat_1, \P_1) \leq \epsilon_{\init}$, but we also provide a guarantee for the first setting in the Supplementary Material.

In the federated setting (a) is done locally at each node, while (b) requires a Fed-OA algorithm for SVD which is done using Algorithm \ref{algo:rankr}. If one were to consider a federated but noise-free setting, there would be no need for new analysis (standard guarantees for PM would apply). 
%
%
%?? In rest of paper and Algo/ Theorem work with the assump that we have no outliers at t=1 (so can use r-svd for init). State this in theorem. 
%

For step (a) (obtaining an estimate of $\tL_t$ column-wise), we use the projected Compressive Sensing (CS) idea \cite{rrpcp_perf}. This relies on  the slow-subspace change assumption.
Let $\Phat_{t-1}$ denote the subspace basis estimate from the previous time and let $\bphi = \I - \Phat_{t-1}\Phat_{t-1}^{\top}$. Projecting $\y_i$ orthogonal to $\Phat_{t-1}$ helps mostly nullify $\l_i$ but gives projected measurements of the missing entries, $\I_{\T_{i}} \I_{\T_{i}}^{\top}\l_i$ and the sparse outliers, $\s_i$ as follows
\begin{align*}
\bphi \y_i = \underbrace{\bphi(\s_i - \I_{\T_{i}} \I_{\T_{i}}^{\top}\l_i)}_{\text{projected sparse vector}} + \underbrace{\bphi(\l_i + \v_i)}_{\text{error}}
\end{align*}
%We can recover $\l_i$ by using modified-CS \cite{modcsjournal} with ... ??. 
%
If the previous subspace estimate is good enough, and the noise is small, the error term above will be small. Now recovering the vector $\s_i - \I_{\T_{i}} \I_{\T_{i}}^{\top}\l_i$ is from $\bphi \y_i$ is a problem of noisy compressive sensing with partial support knowledge (since we know $\T_i$). We first recover the support of $\s_i$ using the approach of \cite{modcs}, and then perform a least-squares based debiasing to estimate the magnitude of the entries. Following this, {\em an} estimate of the true data, $\lhat_i$ is computed by subtraction from the observed data $\y_i$. We show in Lemma \ref{lem:projcs} that $\lhat_i$ satisfies
\begin{align}%\label{eq:etdef}
\lhat_i = \l_i -\bm{I}_{\That_i}\left(\bpsi_{\That_i}^\top\bpsi_{\That_i}\right)^{-1} \I_{\That_i}^\top \bpsi(\l_i + \v_i) + \v_i
\end{align}

Now we have $\Lhat_t := [\Lhat_{1,t}, \Lhat_{2,t}, \cdots, \Lhat_{K,t}]$ with $\Lhat_{k,t}$ available only at node $k$. To goal is to compute an estimate ($\Phat_t$) of its top $r$ left singular vectors while obeying the federated data sharing constraints.
%\footnote{Using our result for FedPM, the span of $\Phat_L$ is an $\epsilon$-accurate approximation of the span of true singular vectors of $\Lhat_t$.}.
We implement this through FedOA-PM (Algorithm \ref{algo:rankr}) with $\Z_k \equiv \Lhat_{k,t}$ being the data matrix at node $k$. We invoke FedOA-PM with an initial estimate $\Phat_{t-1}$. This simple change allows the probability of success of the overall algorithm to be close to $1$ rather than $0.9$ which is what the result of Lemma \ref{lem:fed_app} predicts. This result is obtained by carefully combining the result for PCA-SDDN in a centralized setting (Lemma \ref{cor:cent_pca_dd}) and the result for FedOA-PM (Lemma \ref{lem:fed_app}). The result is summarized in Lemma \ref{lem:fed_pca}. Applying these results in exactly the same manner as we did in Sec. \ref{sec:proof_cent} (with a few minor differences we point out in the next section), we get the main result.

\begin{algorithm}[t!]
\caption{Fed-OA-RSTMiss-NoDet}
\label{algo:fed_nodet_given_init}
\begin{algorithmic}[1]
\Require $\Y$, $\T$ %$r'$, $\tau_t^* = C \log(n /\epsilon_t)$
\State Parameters: %$T_{iter} \leftarrow C \log(1/\varepsilon)$, $\mathrm{phase} \leftarrow \mathrm{update}$,
$L \leftarrow C \log (1/\varepsilon)$, %$L_{\mathrm{det}} \leftarrow C \log(nr/\varepsilon)$, $\lthres \leftarrow 2 \zz^2\lambda^+$. \textbf{Parameters:}
$\omega_{supp}$, $\xi$, $\alpha$
\State \textbf{Init:} $\tau \leftarrow 1$, $j\leftarrow1$, $\Phat_{1}$
\For{$t > 1$ } %\in ((b-1)\alpha, b\alpha]$}
\State $\Lhat_t \leftarrow$ \Call{Fed-ModCS}{$\y_i$, $\calI_{k,t}$, $\T_i$, $\Phat_{t-1}$}
%\If {$t = t_{\train} + u \alpha - 1$ for $u = 1,\ 2,\ \cdots,$}
\State $\Phat_{t} \leftarrow$  \Call{FedOA-PM}{$\Lhat_{t}$, $r$, $L$, $\Phat_{t-1}$}
\State $\tLhat_t \leftarrow$ \Call{Fed-ModCS}{$\y_i$, $\calI_{k,t}$, $\T_i$, $\Phat_{t}$} \Comment{optional}

\EndFor
\Ensure $\Phat$
\end{algorithmic}
\end{algorithm}

\begin{algorithm}[t!]
\caption{Federated Modified Compressed Sensing}\label{algo:fed_proj_ls}
\begin{algorithmic}[1]
\Procedure {Fed-ModCS}{$\y_i$, $\calI_{k,t}$, $\T_i$, $\Phat_{t-1}$}
\ForAll  {node $k$, $i \in \calI_{k, t}$}
\State $\bpsi \leftarrow \bm{I} - \Phat_{t-1}\Phat_{t-1}^{\top}$
\State $\tilde{\y}_{i} \leftarrow \bpsi \y_{i}$
\State $\xhat_{i,cs} \leftarrow \arg \min_{\x} \|(\x)_{(\T_i)^c}\|_1$ s.t. $\|\tilde{\y}_i - \bpsi \x\| \leq \xi$.
\State $\hat{\T}_i \leftarrow \T_i \cup \{j : |(\xhat_{i,cs})_j| > \omega_{supp}\}$
\State $\hat{\bm{\ell}}_{i} \leftarrow \y_{i} -  \I_{\hat{\T}_{i}} (\bpsi_{\hat{\T}_i})^{\dagger} \tilde{\y}_{i}$.
\EndFor
%\ForAll{$i \in \calI_{k,t}$ and $k \in [K]$}
%\EndFor
\State \textbf{Output:} $\Lhat_t$
\EndProcedure
\end{algorithmic}
\end{algorithm}

%provide a result for robust subspace tracking with missing entries while obeying the federated data sharing constraints.

%So far we assumed that the data is available at a single physical location which is not necessarily true in many real world applications. In this section, we consider the setting where the data is distributed across $K$ peer nodes, and the goal is track the underlying subspace from missing and outlier corrupted data while respecting the federated over-air constraints.

\subsection{Guarantee for Fed-OA RST-miss}

Before we state the main result, we need a few definitions. 
\begin{definition}[Sparse outlier fractions]
Consider the $n \times \alpha$ sparse outlier matrix $\S_{t}:=[\S_{1,t}, \dots, \S_{K,t} ] $  at time $t$. We use $\outfraccol$ ($\outfracrow$)  to denote the maximum of the fraction of non-zero elements in any column (row) of this matrix. Also define $\xmin  = \min_{i \in \calI_{k,t}} \min_{j \in  \T_{\sparse, i}} |(\s_i)_j|$.
\end{definition}
Let $\lambda_v^+:= \max_{i \in \calI_{k,t}} \|\E[ \vt_{i} \vt_{i}^{\top}]\|$ and $\max_{i \in \calI_{k,t}} \|\vt_{i}\|^2 \le C r \lambda_v^+$ for all $k \in [K]$.
%Instead of quantifying the bounds on $\check{\V}_t$ we directly impose a bound on $\V_t := \tilde{\V}_t - \check{\V}_t$. We have the following result.

\begin{theorem}[Federated Robust Subspace Tracking NoDet]
Consider Algorithm \ref{algo:fed_nodet_given_init}. Assume that $\sqrt{\lambda_v^+/\lambda^-} := \zz \leq 0.2$. Set $L= C \log(1 /\varepsilon)$ and $\omega_{supp} = \xmin/2$, $\xi = \xmin/15$. Assume that the following hold:
\begin{enumerate}
\item At $t=1$ we are given a $\Phat_1$ s.t. $\SE(\P_1, \Phat_1) \leq \epsilon_{\init}$.

\item  {\bf Incoherence:} $\P_t$'s satisfy $\mu$-incoherence, and $\at_i$'s satisfy statistical right $\mu$-incoherence;

\item {\bf Missing Entries:} $\missfraccol \in O(1/\mu r)$, $\missfracrow \in O(1)$;

\item {\bf Sparse Outliers:} $\outfraccol \in O(1/\mu r)$, $\outfracrow \in O(1)$;

\item {\bf Channel Noise:} the channel noise seen by each FedOA-PM iteration is mutually independent at all times, isotropic, and zero mean Gaussian with standard deviation $\sigma_c \leq \varepsilon \lambda^-/10 \sqrt{n}$. %(this is required if we are using the noisy power method result).
\item {\bf Subspace Model:} The total data available at each time $t$, $\alpha \in \Omega(r \log n)$ and $\Delta_{tv} := \max_t \SE(\P_{t-1},\P_t)$ s.t.
\begin{gather*}
0.3 \epsilon_{\init} + 0.5\Delta_{tv} \leq 0.28 \quad \text{and} \\
C \sqrt{r \lambda^+} (0.3^{t-1}\epsilon_{\init} + 0.5\Delta_{tv} ) + \sqrt{r_v\lambda_v^+} \leq \xmin
\end{gather*}
\end{enumerate}
then, with probability at least $1 - 10 \tmax n^{-10}$, for $t >1$, we have
%the $j$-th subspace change is detected immediately i.e., $t_j \leq \that_j \leq t_j + 1$, and its tracking error decays exponentially after detection, i.e., $\SE(\Phat_{t}, \P_t) \le$
 \begin{align*}
&\SE(\Phat_t, \P_t) \\
&\leq \max(0.3^{t-1} \epsilon_{\init} + \Delta_{tv} (0.3 + 0.3^2 ... + 0.3^{t-1}), \zz) \\
&< \max(0.3^{t-1} \epsilon_{\init} + 0.5\Delta_{tv}, \zz)
\end{align*}
Also, at all times $t$, $\|\tlhat_{i}-\l_{i}\| \le 1.2 \cdot  \SE(\Phat_{t}, \P_t) \|\l_{i}\| + \|\bv_{i}\|$ for all $i \in \calI_{k,t}, k \in [K]$.

%Time complexity at node $k$: $\mathcal{O}(n \alpha_k r \log n \log(1/\zz))$; total time complexity: $\mathcal{O}(n d r \log n \log(1/\zz))$.
\label{thm1_newnorst}
\end{theorem}

%\subsubsection{Understanding the Subspace model Assumption from above}
%Consider the first condition, $\Delta_{tv} \leq c \varepsilon^2/f^2$. We require this condition because if we allow the subspace to change too much at each time, effective rank of the resulting matrix will be large and so we will not be able to leverage the low-rankness at all. The second condition can be understood as follows: in all robust subspace tracking approaches, a bound is imposed on the initialization. For simplicity, consider the case when $\Delta_{tv} = 0$, i.e., the subspace is fixed. In this case, our assumption is very mild since we only require $\epsilon_{\init} \leq 0.93$ (and recall that for any two subspace matrices,  $\SE(\P, \tilde\P) \in [0,1]$). In the case when the subspace is allowed to vary at each time, we need a slightly stronger assumption to ensure that the projection matrix used for Proj-CS is sufficiently well conditioned. Finally, the last assumption is an artifact of our work and arises due to the fact that there are no element-wise error bounds for Compressive Sensing. ??pn need to edit, for now included everything here.

\subsubsection{Discussion}
Items $2$-$4$ of Theorem \ref{thm1_newnorst} are necessary to ensure that the RST-miss and robust matrix completion problems are well posed \cite{normc, rrpcp_tsp19}. The initialization assumption of Theorem \ref{thm1_newnorst} is different from the requirement of Theorem \ref{thm:central_timevar} due to the presence of outliers. Just performing a $r$-SVD on $\Y_1$ as done in Algorithm \ref{algo:norst_nodet} does not work since even a few outliers can make the output arbitrarily far from the ``true subspace''. Additionally, without a ``good initialization'' Algorithm \ref{algo:fed_nodet_given_init} cannot obtain good estimates of the sparse outliers since the noise in the sparse recovery step would be too large. One possibility to extend our result is to assume that there are no outliers at $t=1$, i.e., $\S_1 = \bm{0}$ in which case, we use the initialization idea of Algorithm \ref{algo:norst_nodet} (see Remark \ref{rem:zero_out}). Item $5$ is standard in the federated learning/differential privacy literature \cite{noisy_pm, improved_npm} as without bounds on iteration noise, it is not possible to obtain a final estimate that is close to the ground truth. Finally, consider item $6$: the first part is required to ensure that the projection matrices, $\bpsi$'s satisfy the restricted isometry property \cite{candes_rip, modcs} which is necessary for provable sparse recovery (with partial support knowledge). This is a more stringent assumption than $\Delta_{tv} \leq 0.1$ assumed in Theorem \ref{thm:central_timevar} due to the presence of outliers. The second part of item $6$ is an artifact of our analysis and arises due to the fact that it is hard to obtiain element-wise error bounds for Compressive Sensing. 

In Theorem \ref{thm1_newnorst} we assumed that we are given a good enough initialization. If however, $\S_1$ were $0$, we have the following result.  

\begin{remark}\label{rem:zero_out}
Under the conditions of Theorem \ref{thm1_newnorst}, if $\S_1 = 0$, then all conclusions of Theorem \ref{thm1_newnorst} hold with the following changes
\begin{enumerate}
\item The number of iterations is set as $L = C \log(n/\zz)$
\item The subspace model (item $6$ satisfies all conditions with $\epsilon_{\init}$ replaced by $0.01\cdot 0.3$
\item The probability of success is now $0.9 - 10dn^{-10}$.
\end{enumerate}
\end{remark}

%\begin{algorithm}[ht!]
%\caption{Fed-OA-RSTMiss-NoDet ($\S_1 = 0$)} 
%\label{algo:fed_nodet_zero_init}
%\begin{algorithmic}[1]
%\Require $\Y$, $\T$ %$r'$, $\tau_t^* = C \log(n /\epsilon_t)$
%\State Parameters: %$T_{iter} \leftarrow C \log(1/\varepsilon)$, $\mathrm{phase} \leftarrow \mathrm{update}$, 
%$L \leftarrow C \log (nr/\varepsilon)$, %$L_{\mathrm{det}} \leftarrow C \log(nr/\varepsilon)$, $\lthres \leftarrow 2 \zz^2\lambda^+$. \textbf{Parameters:} 
%$\omega_{supp}$, $\xi$, $\alpha$ %\Comment{{\color{red} The value of $L$ will be smaller?}}
%\State \textbf{Init:} $\tau \leftarrow 1$, $j\leftarrow1$, $\Phat_{0} \leftarrow \bm{0}_{n \times r}$
%\State $\Lhat_0 \leftarrow$ \Call{Rob-FedOA-ProjLS}{$\calI_{k,0}$, $\y_i$, $\T_i$, $\Phat_{0}$}
%\State $\Phat_{1} \leftarrow$  \Call{FedOA-PM}{$\Lhat_{1}$, $r$, $L$} 
%\For{$t > 1$ } %\in ((b-1)\alpha, b\alpha]$}
%\State $\Lhat_t \leftarrow$ \Call{Rob-Fed-ProjLS}{$\y_i$, $\calI_{k,t}$, $\T_i$, $\Phat_{t-1}$}
%%\If {$t = t_{\train} + u \alpha - 1$ for $u = 1,\ 2,\ \cdots,$}
%\State $\Phat_{t} \leftarrow$  \Call{FedOA-PM}{$\Lhat_{t}$, $r$, $L$, $\Phat_{t-1}$}
%\EndFor
%\Ensure $\Phat$
%\end{algorithmic}
%\end{algorithm}

\subsection{Proof Outline}\label{sec:proof_outline}

Here we prove our main result for robust ST-Miss under the federated data sharing constraints. The proof relies on two main results given below -- (i) the result of (centralized) RST-Miss proved in the Supplementary Material (Appendix. \ref{sec:rst_miss_central}) and (ii) our result for federated over-air power method from Sec. \ref{sec:fedpm}. \begin{lem}[Projected-CS with partial support knowledge]\label{lem:projcs}
Consider Lines $5-7$ of Algorithm \ref{algo:fed_proj_ls}. Under the conditions of Theorem \ref{thm1_newnorst}, we have for all $t$ and all $i \in \calI_{k,t}$, the error seen by the compressed sensing step satisfies
%we have
%%\in [\hat{t}_j, \hat{t}_j + \alpha)$,
%
%
\begin{align*}
\norm{\bm\Psi (\l_i + \vt_i)} \le (0.3^{t-1} \epsilon_{\init} +  2.5\Delta_{tv}) \sqrt{\mu r \lambda^+} + \sqrt{r_v \lambda_v^+}
\end{align*}
$\norm{\xhat_{i,cs} - \x_i} \le 7 \xmint/15 < \xmint/2$,  $\That_{\sparse,i} = \T_{\sparse, i}$,
the error $\et := \lhat_i - \l_i$ satisfies
\begin{align}\label{eq:etdef}
\et &= -\bm{I}_{\That_i}\left(\bpsi_{\That_i}^\top\bpsi_{\That_i}\right)^{-1} \I_{\That_i}^\top \bpsi(\l_i + \v_i) + \v_i, \nonumber \\
&= (\e_i)_{\l} + (\e_i)_{\bv} + \v_i
\end{align}
and $\norm{\e_i} \leq 1.2 (0.3^{t-1} \epsilon_{\init} +  2.5\Delta_{tv}) \sqrt{\mu r \lambda^+} + 2.2\sqrt{r_v \lambda_v^+}$. Here, $\bpsi = \I - \Phat_{t-1}\Phat_{t-1}^\top$
%
%Here $\bpsi = \I - \Phat_{j,0} \Phat_{j,0}{}'$. Recall we let $\Phat_{j,0} = \Phat_{j-1}$.

%\item w.p. at least $1 - 10n^{-10}$,  $\Phat_{j,1}$ satisfies $\SE(\Phat_{j, 1}, \P_j) \leq \max(q_{0}/4, \zz)$, i.e., $\Gamma_{j, 1}$ holds.
%
%\end{enumerate}
\end{lem}

\begin{lem}[FedOA PCA-SDDN (available init)]\label{lem:fed_pca}
%Consider the federated data sharing constraints described in Sec. \ref{sec:fedpm}.
Consider the output $\Phat$ of FedOA-PM (Algorithm \ref{algo:rankr}) applied on data vectors $\z_i$ distributed across $K$ nodes, when $\z_i = \l_i + \et_i + \v_i$, $i=1,2, \dots, \alpha$ with
$\l_i = \P \at_i$, $\et_i = \bm{I}_{\T_i} \B_{i}\l_i$ being sparse, data-dependent noise with support $\T_i$;  the modeling error $\v_i$ is bounded with $\max_i \|\v_i\|^2 \leq C r_v  \lambda_v^+$ where $\lambda_v^+:=\|\E[\v_i \v_i^{\top}]\|$. The matrix of top-$r$ left singular vectors, $\P$ satisfies $\mu$-incoherence, and $\at_i$'s satisfy $\mu$-statistical right-incoherence.  The channel noise is zero mean i.i.d. Gaussian with standard deviation $\sigma_c \leq \epsilon_{PM} \lambda^- /10 \sqrt{n}$ and is independent of the $\l_i$'s. Let $q:= \max_{i} \|\B_{i}\P\|$ and let $\bz$ denote the fraction of non-zeros in any row of the SDDN matrix $\bm{E} = [\et_1, \cdots, \et_{\alpha}]$. Pick an $\epsilon_{PM} > 0$. If
\begin{align*}
7 \sqrt{\bz}q f + \lambda_v^+/\lambda^- < 0.4 \epsilon_{PM},
\end{align*}
$ \alpha \geq C r \log n  \max( \frac{q^2}{\epsilon_{PM}^2} f^2, \frac{\frac{\lambda_v^+}{\lambda^-}}{\epsilon_{PM}^2} f)$,
and if FedOA-PM is initialized with a matrix $\P_{\init}$ such that $\SE(\P_{\init},\P) \leq \epsilon_{\init, PM}$,  then after $L = C \log(1 /(\epsilon_{PM}\sqrt{1-\epsilon_{\init,PM}^2}))$ iterations, with  probability at least $1 - L\exp(-cr) - n^{-10}$, $\Phat$ satisfies $\SE(\Phat,\P) \leq \epsilon_{PM}$.
\end{lem}

With these two Lemmas, the proof of Theorem \ref{thm1_newnorst} is similar to the proof of Theorem \ref{thm:central_timevar}. Firstly, consider the projected CS with partial support knowledge step. Lemma \ref{lem:projcs} applied to each vector locally gives us $\lhat_i = \l_i - \e_i$ with $\e_i$ satisfying \eqref{eq:etdef}. Next, at each time $t$, we update the subspace as the top $r$ left singular vectors of $\Lhat_t$, where the $k$-th node only has access to the sub-matrix $\Lhat_{k,t}$. For a $t > 1$, we assume that the previous subspace estimate, $\Phat_{t-1}$ satisfies $\SE(\Phat_{t-1}, \P_{t-1}) \leq \max(0.3^{t-2} \epsilon_{\init} + 0.5 \Delta_{tv}, \zz)$. We invoke Lemma \ref{lem:fed_pca} with  $\Phat_{init} \equiv \Phat_{t-1}$ and thus, $\epsilon_{\init, PM} \equiv \max(0.3^{t-2} \epsilon_{\init} + 0.5 \Delta_{tv}, \varepsilon)$; $\z_i \equiv \lhat_i, i\in \calI_{k,t}$; $\P \equiv \P_t$, $\e_i \equiv (\e_i)_{\l}$, $\v_i \equiv (\e_i)_{\bv} + \bv_i$; and $\epsilon_{PM} \equiv \max(0.3^{t-2} \epsilon_{\init} + 0.5 \Delta_{tv}, \varepsilon)$. Under the conditions of theorem \ref{thm1_newnorst}, we conclude that w.h.p., $\SE(\Phat_t, \P_t) \leq \max(0.3^{t-1} \epsilon_{\init} + 0.5 \Delta_{tv}, \varepsilon)$. Thus, applying this argument inductively proves the result. For the second optional FedOA-PM step, the same ideas from the proof of Theorem \ref{thm:central_timevar} apply.

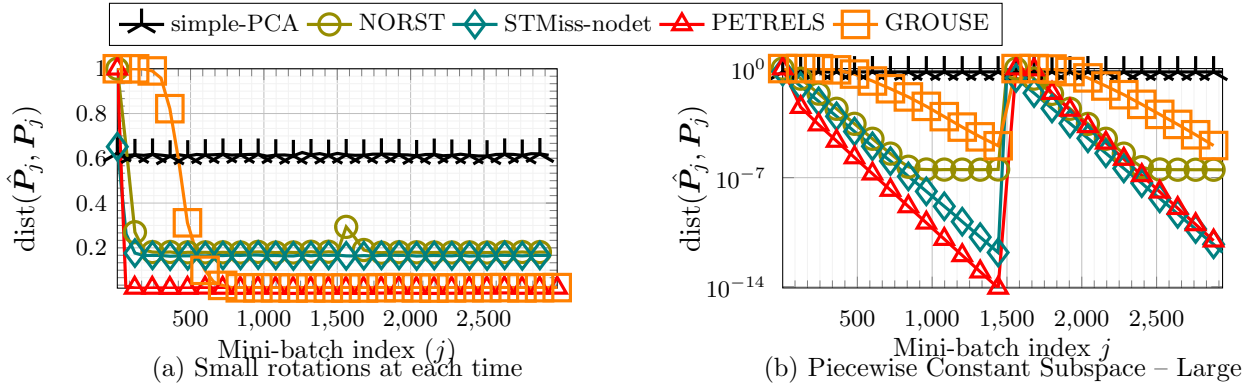
\begin{figure*}[t!]
\centering
\begin{tikzpicture}
    \begin{groupplot}[
        group style={
            group size=2 by 1,
            horizontal sep=3cm,
            vertical sep=1cm,
                        x descriptions at=edge bottom,
            %y descriptions at=edge left,
        },
        % load the style created in the preamble
        %my stylecompare,
                %ymin=1e-15, ymax=0,
        enlargelimits=false,
        width = .45\linewidth,
        height=4.5cm,
        enlargelimits=false,
                grid=both,
    grid style={line width=.1pt, draw=gray!10},
    major grid style={line width=.2pt,draw=gray!50},
    minor x tick num=5,
    minor y tick num=5,
    ]
       \nextgroupplot[
%		legend entries={
%					simple PCA,
%					NORST,
%					NORST-nodet,
%					PETRELS,
%					GROUSE,
%					%Fed-ST-Miss,
%            		%NORST-sliding[$\beta=1$ $R=0$] ($7$ms),
%            		%NORST-sliding[$\beta=10$ $R=1$] ($11$ms),            		
%            		%PETRELS ($29$ms),
%            		},
%            legend style={at={(2.5, 1.3)}},
%            %legend style={at={(1,1.8)}}, % use this if one column fig
%            legend columns = 4,
%            legend style={font=\footnotesize},
			%ymode=log,
            xlabel={\small{Mini-batch index ($j$)}},
            ylabel={{$\SE(\Phat_{j},\P_j)$}},
                        title style={at={(0.5,-.35)},anchor=north,yshift=1},
            title={\small{(a) Small rotations at each time}},
                        xticklabel style= {font=\footnotesize, yshift=-1ex},
            yticklabel style= {font=\footnotesize, xshift=0ex},
        ]
%			\addplot [black, line width=1.2pt, mark=Mercedes star,mark size=6pt, mark repeat=1] table[x index = {0}, y index = {1}, select coords between index={0}{24}]{\stmiss};
	        %\addplot [black, dotted,  line width=1.2pt, mark=square,style=solid,mark size=3pt, mark repeat=1] table[x index = {2}, y index = {3}]
	        %{\pwjustall};
	        %\addplot [black, line width=1.2pt, mark=diamond, style=solid, mark size=5pt, mark repeat=1] table[x index = {4}, y index = {5}]
	        %{\pwjustall};
			\addplot [black, line width=1.2pt, mark=Mercedes star,mark size=6pt, mark repeat=2] table[x index = {0}, y index = {3}]{\smallrot};
			\addplot [olive, line width=1.2pt, mark=o,mark size=4pt, mark repeat=2] table[x index = {0}, y index = {1}]{\smallrot};
	        %\addplot [olive, dotted,  line width=1.2pt, mark=square,style=solid,mark size=3pt, mark repeat=1] table[x index = {8}, y index = {9}]
	        %{\pwjustall};
	        %\addplot [olive, line width=1.2pt, mark=diamond,style=solid,mark size=5pt, mark repeat=1] table[x index = {10}, y index = {11}]
	        %{\pwjustall};

			\addplot [teal, line width=1.2pt, mark=diamond,mark size=5pt, mark repeat=2] table[x index = {0}, y index = {2}]{\smallrot};
			\addplot [red, line width=1.2pt, mark=triangle,mark size=4pt, mark repeat=2] table[x index = {0}, y index = {1}]{\smallrotcomp};
	        %\addplot [olive, dotted,  line width=1.2pt, mark=square,style=solid,mark size=3pt, mark repeat=1] table[x index = {8}, y index = {9}]
	        %{\pwjustall};
	        %\addplot [olive, line width=1.2pt, mark=diamond,style=solid,mark size=5pt, mark repeat=1] table[x index = {10}, y index = {11}]
	        %{\pwjustall};

			\addplot [orange, line width=1.2pt, mark=square,mark size=5pt, mark repeat=2] table[x index = {0}, y index = {2}]{\smallrotcomp};

%				        	        \addplot [blue, line width=1.6pt, mark=oplus,mark size=3pt] table[x index = {18}, y index = {19}]{\stmiss};

\nextgroupplot[
		legend entries={
					simple-PCA,
					NORST,
					STMiss-nodet,
					PETRELS,
					GROUSE,
					%Fed-ST-Miss,
            		%NORST-sliding[$\beta=1$ $R=0$] ($7$ms),
            		%NORST-sliding[$\beta=10$ $R=1$] ($11$ms),            		
            		%PETRELS ($29$ms),
            		},
            legend style={at={(0.5, 1.3)}},
            %legend style={at={(1,1.8)}}, % use this if one column fig
            legend columns = 5,
            legend style={font=\footnotesize},
			ymode=log,
            xlabel={\small{Mini-batch index $j$}},
            ylabel={{$\SE(\Phat_{j},\P_j)$}},
                        title style={at={(0.5,-.35)},anchor=north,yshift=1},
            title={\small{(b) Piecewise Constant Subspace  -- Large}},
                        xticklabel style= {font=\footnotesize, yshift=-1ex},
            yticklabel style= {font=\footnotesize, xshift=-.5ex},
        ]

			\addplot [black, line width=1.2pt, mark=Mercedes star,mark size=6pt, mark repeat=2] table[x index = {0}, y index = {3}]{\pwconst};
			\addplot [olive, line width=1.2pt, mark=o,mark size=4pt, mark repeat=2] table[x index = {0}, y index = {1}]{\pwconst};
			\addplot [teal, line width=1.2pt, mark=diamond,mark size=5pt, mark repeat=2] table[x index = {0}, y index = {2}]{\pwconst};
			\addplot [red, line width=1.2pt, mark=triangle,mark size=4pt, mark repeat=1] table[x index = {6}, y index = {7}, select coords between index={0}{24}]{\stmiss};
	        %\addplot [olive, dotted,  line width=1.2pt, mark=square,style=solid,mark size=3pt, mark repeat=1] table[x index = {8}, y index = {9}]
	        %{\pwjustall};
	        %\addplot [olive, line width=1.2pt, mark=diamond,style=solid,mark size=5pt, mark repeat=1] table[x index = {10}, y index = {11}]
	        %{\pwjustall};

			\addplot [orange, line width=1.2pt, mark=square,mark size=5pt, mark repeat=1] table[x index = {12}, y index = {13}, select coords between index={0}{24}]{\stmiss};

%       \nextgroupplot[
%			ymode=log,
%            xlabel={\small{Mini-batch index $j$}},
%            %ylabel={{$\SE(\Phat_{j},\P_j)$}},
%                        title style={at={(0.5,-.35)},anchor=north,yshift=1},
%            title={\small{(c) Piecewise Constant Subspace  -- Small}},
%            %xtick={300, 1000, 1700, 2400},
%            %xticklabels={$300$, $1000$, $1700$, $2400$},
%                        xticklabel style= {font=\footnotesize, yshift=-1ex},
%            yticklabel style= {font=\footnotesize, xshift=.5ex},
%        ]
%			\addplot [black, line width=1.2pt, mark=Mercedes star,mark size=6pt, mark repeat=2] table[x index = {0}, y index = {3}]{\newgenlarge};
%			\addplot [olive, line width=1.2pt, mark=o,mark size=4pt, mark repeat=2] table[x index = {0}, y index = {1}]{\newgenlarge};
%			\addplot [teal, line width=1.2pt, mark=diamond,mark size=5pt, mark repeat=2] table[x index = {0}, y index = {2}]{\newgenlarge};

    \end{groupplot}
\end{tikzpicture}
%\vspace{-.5cm}
\caption{\small{Comparison of ST-Miss Algorithms in the centralized setting.}}
%\vspace{-.7cm}
\label{fig:st_miss}
\end{figure*}

%%%%%%%%%%%newfig
\begin{figure*}[t!]
\centering
\begin{tikzpicture}
    \begin{groupplot}[
        group style={
            group size=2 by 1,
            horizontal sep=3cm,
            vertical sep=1cm,
                        x descriptions at=edge bottom,
            %y descriptions at=edge left,
        },
        % load the style created in the preamble
        %my stylecompare,
                %ymin=1e-15, ymax=0,
        enlargelimits=false,
        width = .45\linewidth,
        height=4.5cm,
        enlargelimits=false,
                grid=both,
    grid style={line width=.1pt, draw=gray!10},
    major grid style={line width=.2pt,draw=gray!50},
    minor x tick num=5,
    minor y tick num=5,
    ]
       \nextgroupplot[
		legend entries={
					$\sigma_c = 1$,
					$\sigma_c = 10^{-2}$,
					$\sigma_c = 10^{-4}$,
					$\sigma_c = 10^{-6}$,
%					GROUSE,
%					%Fed-ST-Miss,
%            		%NORST-sliding[$\beta=1$ $R=0$] ($7$ms),
%            		%NORST-sliding[$\beta=10$ $R=1$] ($11$ms),            		
%            		%PETRELS ($29$ms),
            		},
            legend style={at={(1.1, 1.3)}},
            %legend style={at={(1,1.8)}}, % use this if one column fig
            legend columns = 4,
            legend style={font=\footnotesize},
			ymode=log,
            xlabel={\small{Mini-batch index ($j$)}},
            ylabel={{$\SE(\Phat_{j},\P_j)$}},
                        title style={at={(0.5,-.35)},anchor=north,yshift=1},
            title={\small{(a) Comparison with respect to channel noise}},
                        xticklabel style= {font=\footnotesize, yshift=-1ex},
            yticklabel style= {font=\footnotesize, xshift=0ex},
        ]
			\addplot [black, line width=1.2pt, mark=Mercedes star,mark size=6pt, mark repeat=2] table[x index = {0}, y index = {1}]{\sigvar};
			\addplot [olive, line width=1.2pt, mark=o,mark size=4pt, mark repeat=2] table[x index = {0}, y index = {2}]{\sigvar};
			\addplot [teal, line width=1.2pt, mark=diamond,mark size=5pt, mark repeat=2] table[x index = {0}, y index = {3}]{\sigvar};
			\addplot [red, line width=1.2pt, mark=triangle,mark size=4pt, mark repeat=2] table[x index = {0}, y index = {4}]{\sigvar};

			%\addplot [orange, line width=1.2pt, mark=square,mark size=5pt, mark repeat=2] table[x index = {0}, y index = {2}]{\smallrotcomp};

\nextgroupplot[
		legend entries={
					$\rho = 0.1$,
					$\rho = 0.2$,
					$\rho = 0.4$,
					$\rho = 0.6$,
            		},
            legend style={at={(1.05, 1.3)}},
            %legend style={at={(1,1.8)}}, % use this if one column fig
            legend columns = 5,
            legend style={font=\footnotesize},
			ymode=log,
            xlabel={\small{Mini-batch index $j$}},
            ylabel={{$\SE(\Phat_{j},\P_j)$}},
                        title style={at={(0.5,-.35)},anchor=north,yshift=1},
            title={\small{(b) Comparison with respect to fraction of missing entries}},
                        xticklabel style= {font=\footnotesize, yshift=-1ex},
            yticklabel style= {font=\footnotesize, xshift=-.5ex},
        ]

			\addplot [black, line width=1.2pt, mark=Mercedes star,mark size=6pt, mark repeat=2] table[x index = {0}, y index = {1}]{\rhovar};
			\addplot [olive, line width=1.2pt, mark=o,mark size=4pt, mark repeat=2] table[x index = {0}, y index = {2}]{\rhovar};
			\addplot [teal, line width=1.2pt, mark=diamond,mark size=5pt, mark repeat=2] table[x index = {0}, y index = {3}]{\rhovar};
			\addplot [red, line width=1.2pt, mark=triangle,mark size=4pt, mark repeat=1] table[x index = {0}, y index = {4}]{\rhovar};

%			\addplot [orange, line width=1.2pt, mark=square,mark size=5pt, mark repeat=1] table[x index = {12}, y index = {13}, select coords between index={0}{24}]{\stmiss};

\end{groupplot}
\end{tikzpicture}
\caption{\small{Performance of Algorithm \ref{algo:fed_nodet_given_init} under varying model parameters.}}
\label{fig:fed_st_comp_param}
\end{figure*}

\Section{Numerical Experiments}\label{sec:sims}
%Experiments are performed on a Desktop Computer with Intel$^{\textsuperscript{\textregistered}}$ Xeon $8$-core CPU with $32$GB RAM and the results are averaged over $100$ independent trials. 
The codes are available at \url{https://github.com/praneethmurthy/distributed-pca}.

\subsection{Centralized STMiss}
\subsubsection{Small Rotations at each time} We first consider the centralized setting for Subspace Tracking with missing data (Sec. \ref{sec:stmiss}). We demonstrate results under two sets of subspace change models. First we consider the ``rotation model'' that has been commonly used in the literature \cite{petrels, grouse_global}.  At each time $t$, we generate a $n \times r$ dimensional subspace $\P_{(t)} = e^{-\delta_t \bm{B}_t} \P_{(t-1)} $ with $\P_{(0)}$ generated by orthonormalizing the columns of a i.i.d. standard Gaussian matrix and $\bm{B}_t$ is some skew symmetric matrix to simulate rotations and $\delta_t$ controls the amount of rotation (for this experiment we set $\delta_t = 10^{-4}$ which ensures that $\Delta_{tv} \approx 10^{-2}$). We generate matrix $\tilde{\bm{A}}$ as a i.i.d. uniform random matrix of size $r \times \tmax$ and set the $t$-th column of the true data matrix $\tl_t = \P_{(t)} \tilde{\at}_t$. Thus, in the notation of our result, $\P_j$ is the matrix of the top $r$ left singular vectors of $\tL_j = [\tl_{(j-1)\alpha + 1} , \cdots \tl_{j\alpha}]$ and $\A_j = \P_j^\top \tL_j$. In all experiments, we choose $n=1000$ and $d = 3000$. We simulate the set of observed entries using a Bernoulli model where each element of the matrix is observed with probability $0.9$. For all experiments, we set $r = 30$ and the fraction of missing entries to be $0.1$. We implement STMiss-nodet (Algorithm \ref{algo:norst_nodet}) and set $r = 30$. We compare with NORST \cite{rrpcp_tsp19} (the state-of-the-art theoretically), GROUSE \cite{grouse_global}, and PETRELS \cite{petrels} (the state-of-the-art experimentally). For all algorithms, we used default parameters mentioned in the codes. We also implement the simple PCA method wherein we estimate $\Phat_j$ as the top-$r$ left singular vectors of $\Y_j$ for each mini-batch. For all algorithms, the mini-batch size was chosen as $\alpha = 60$. The results are shown in Fig. \ref{fig:st_miss}(a). We see that as specified by Theorem \ref{thm:naive}, the simple PCA algorithm does not improve the recovery errors since it is not exploiting slow subspace change. However, all other algorithms exploit slow-subspace change and thus are able to provide better estimates with time. We also notice that PETRELS is the fastest to converge, followed by NORST and STMiss-nodet, and finally GROUSE. This is consistent with the previous set of results in \cite{rrpcp_tsp19}.

\subsubsection{Piecewise Constant}
Next, we consider a piecewise constant subspace change model that has been considered in the provable subspace tracking literature \cite{rrpcp_tsp19}. In this, we simulate a large subspace change at $t_1 = 1500$. The subspace is fixed until then, i.e., $\P_j = \P_1$ for all $j \in [1, \lceil t_1/\alpha \rceil)$ and $\P_j = \P_2 $ for all $j \in [\lceil t_1/\alpha \rceil, \lceil \tmax/\alpha \rceil]$. The results are shown in Fig. \ref{fig:st_miss}(b). Notice that NORST and STMiss-nodet significantly outperform simple PCA as both exploit slow subspace change. Additionally, even though the change is large (in the notation of Definition \ref{def:large_ss} given in the supplementary material, $\Delta_{\mathrm{large}} \approx 1$  and $\Delta_{tv} = 0$), STMiss-nodet is also able to adapt without requiring a detection step. Finally, since the updates are always improving, after a certain time, NORST stops improving the subspace estimates, but STMiss-nodet improves it and gets a better result.

%\subsubsection{Piecewise Constant for Small Time}
%Finally, we consider a setting that is easier than the first setting, but harder than the second setting. We generate $\P_j$'s for each $j in \J_j$ fo length $\alpha$, and use the rotation model such that $\Delta_{tv} \approx 1$. As predicted by Theorem \ref{thm:naive} and Theorem \ref{thm:central_timevar}, the simple PCA (naive) method works better than the NORST and NORST-nodet algorithms. 

%We do not compare with other STmiss methods since \cite{rrpcp_tsp19} contains an extensive numerical experiment section that shows that under the settings considered in this paper, NORST \cite{rrpcp_tsp19} is comparable to PETRELS \cite{petrels}.

\subsection{Federated ST-Miss}
We also implement Algorithm \ref{algo:fed_nodet_given_init} to corroborate our theoretical claims. We use the exact data generation parameters as we did in the centralized setting. To simulate over-air communication, we replace the inbuilt \texttt{SVD} routine of MATLAB by a power method code snippet, and by adding iteration noise\footnote{Note that in this approach, it is not required to explicitly set a value of $K$}. In each iteration, we add i.i.d. Gaussian noise with variance $10^{-6}$. The results are presented in Fig. \ref{fig:fed_st_miss}. Notice that in both cases, Algorithm \ref{algo:fed_nodet_given_init} works as well as NORST even though NORST cannot deal with iteration noise. Additionally, as opposed to the centralized setting (Fig. \ref{fig:st_miss}(b)), the error of Fed-OA-RSTMiss-nodet in Fig. \ref{fig:fed_st_miss} does not improve beyond the iteration noise level of $10^{-6}$. 

{
We next validate the performance of Algorithm \ref{algo:fed_nodet_given_init} with respect to different values of the channel noise. We generate the data as done in the previous experiment, but vary the iteration noise level. In particular, we choose $\sigma_c = \{1, 10^{-2}, 10^{-4}, 10^{-6}\}$ and provide the results in Fig. \ref{fig:fed_st_comp_param}(a). Notice that in all the cases, the subspace error saturates at roughly $\sigma_c$ as predicted. 

Finally we analyze the performance of Algorithm \ref{algo:fed_nodet_given_init} with respect to different values of missing entries. The data is generated as in the previous experiment with $\sigma_c = 10^{-6}$ and we vary the fraction of missing entries, $\rho = \{0.1, 0.2, 0.4, 0.6\}$. The results are given in Fig. \ref{fig:fed_st_comp_param}(b) and we notice that in all cases, the algorithm works, but as the fraction of missing entries increases, more samples are required for convergence.  
}

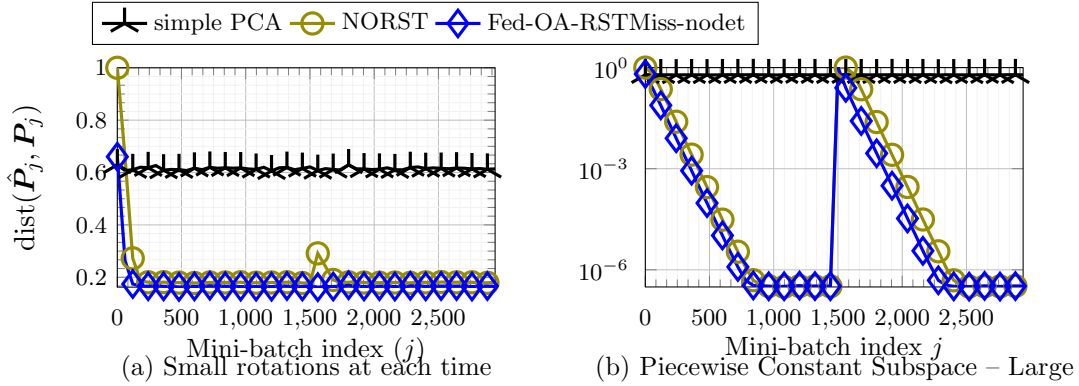
\begin{figure*}[t!]
\centering
\begin{tikzpicture}
    \begin{groupplot}[
        group style={
            group size=3 by 1,
            horizontal sep=2cm,
            vertical sep=1cm,
                        x descriptions at=edge bottom,
            %y descriptions at=edge left,
        },
        % load the style created in the preamble
        %my stylecompare,
                %ymin=1e-15, ymax=0,
        enlargelimits=false,
        width = .4\linewidth,
        height=4.5cm,
        enlargelimits=false,
                grid=both,
    grid style={line width=.1pt, draw=gray!10},
    major grid style={line width=.2pt,draw=gray!50},
    minor x tick num=5,
    minor y tick num=5,
    ]
       \nextgroupplot[
			%ymode=log,
            xlabel={\small{Mini-batch index ($j$)}},
            ylabel={{$\SE(\Phat_{j},\P_j)$}},
                        title style={at={(0.5,-.35)},anchor=north,yshift=1},
            title={\small{(a) Small rotations at each time}},
                        xticklabel style= {font=\footnotesize, yshift=-1ex},
            yticklabel style= {font=\footnotesize, xshift=0ex},
        ]
%			\addplot [black, line width=1.2pt, mark=Mercedes star,mark size=6pt, mark repeat=1] table[x index = {0}, y index = {1}, select coords between index={0}{24}]{\stmiss};
	        %\addplot [black, dotted,  line width=1.2pt, mark=square,style=solid,mark size=3pt, mark repeat=1] table[x index = {2}, y index = {3}]
	        %{\pwjustall};
	        %\addplot [black, line width=1.2pt, mark=diamond, style=solid, mark size=5pt, mark repeat=1] table[x index = {4}, y index = {5}]
	        %{\pwjustall};
			\addplot [black, line width=1.2pt, mark=Mercedes star,mark size=6pt, mark repeat=2] table[x index = {0}, y index = {3}]{\smallrotfed};
			\addplot [olive, line width=1.2pt, mark=o,mark size=4pt, mark repeat=2] table[x index = {0}, y index = {1}]{\smallrotfed};
	        %\addplot [olive, dotted,  line width=1.2pt, mark=square,style=solid,mark size=3pt, mark repeat=1] table[x index = {8}, y index = {9}]
	        %{\pwjustall};
	        %\addplot [olive, line width=1.2pt, mark=diamond,style=solid,mark size=5pt, mark repeat=1] table[x index = {10}, y index = {11}]
	        %{\pwjustall};

			\addplot [blue, line width=1.2pt, mark=diamond,mark size=5pt, mark repeat=2] table[x index = {0}, y index = {2}]{\smallrotfed};
%				        	        \addplot [blue, line width=1.6pt, mark=oplus,mark size=3pt] table[x index = {18}, y index = {19}]{\stmiss};

       \nextgroupplot[
       legend entries={
					simple PCA,
					NORST,
					Fed-OA-RSTMiss-nodet,
					PETRELS,
					GROUSE,
					%Fed-ST-Miss,
            		%NORST-sliding[$\beta=1$ $R=0$] ($7$ms),
            		%NORST-sliding[$\beta=10$ $R=1$] ($11$ms),            		
            		%PETRELS ($29$ms),
            		},
            legend style={at={(0.3, 1.3)}},
            %legend style={at={(1,1.8)}}, % use this if one column fig
            legend columns = 5,
            legend style={font=\footnotesize},
			ymode=log,
            xlabel={\small{Mini-batch index $j$}},
                        title style={at={(0.5,-.35)},anchor=north,yshift=1},
            title={\small{(b) Piecewise Constant Subspace  -- Large}},
                        xticklabel style= {font=\footnotesize, yshift=-1ex},
            yticklabel style= {font=\footnotesize, xshift=-.5ex},
        ]
			\addplot [black, line width=1.2pt, mark=Mercedes star,mark size=6pt, mark repeat=2] table[x index = {0}, y index = {3}]{\pwconstfed};
			\addplot [olive, line width=1.2pt, mark=o,mark size=4pt, mark repeat=2] table[x index = {0}, y index = {1}]{\pwconstfed};
			\addplot [blue, line width=1.2pt, mark=diamond,mark size=5pt, mark repeat=2] table[x index = {0}, y index = {2}]{\pwconstfed};
			%\addplot [red, line width=1.2pt, mark=triangle,mark size=4pt, mark repeat=1] table[x index = {6}, y index = {7}, select coords between index={0}{24}]{\stmiss};
	        %\addplot [olive, dotted,  line width=1.2pt, mark=square,style=solid,mark size=3pt, mark repeat=1] table[x index = {8}, y index = {9}]
	        %{\pwjustall};
	        %\addplot [olive, line width=1.2pt, mark=diamond,style=solid,mark size=5pt, mark repeat=1] table[x index = {10}, y index = {11}]
	        %{\pwjustall};

			%\addplot [orange, line width=1.2pt, mark=square,mark size=5pt, mark repeat=1] table[x index = {12}, y index = {13}, select coords between index={0}{24}]{\stmiss};

    \end{groupplot}
\end{tikzpicture}
%\vspace{-.5cm}
\caption{\small{Corroborating the claims of Theorem \ref{thm1_newnorst}.}}
%\vspace{-.7cm}
\label{fig:fed_st_miss}
\end{figure*}

{
\section{Conclusions and Future Work}\label{sec:conc}
In this paper we studied the problem of Subspace Tracking from missing data and outliers. In particular, we consider a generalized problem formulation that does not make the {\em piecewise constant} subspace change assumption that is common in the provable subspace tracking literature. We proposed a simple algorithm to solve this problem provably and efficiently. We also developed a an algorithm to solve (robust) subspace tracking with missing entries when the data is federated, and over-air data communication modality is used. As part of future work, there are several open questions such as (i) is it possible to modify the proposed analysis to provide a guarantee for differentially private subspace tracking? (ii) is it possible to consider row-wise federated data model with appropriate analytical modifications? 
}

%
%\clearpage
\appendices
\renewcommand\thetheorem{\arabic{section}.\arabic{theorem}}

\counterwithin{theorem}{section}
%\clearpage

%\vspace{-.42cm}
\section{Proof of Key Lemmas for Theorem \ref{thm1_newnorst}}

%Appendix, I suggest we organize as Append A: Proof of Corollary 2.10 the corollary from Sec 2, Append B: proofs for lemmas in Sec III. 
%
%Append C: we just keep this RST-miss (w/o federated parts) section, decide to keep or remove based on how we organize main paper Sec III> 
%
%Appendix D: extra material for Sec III, this is where all the other stuff comes in
%
%
%Hi - Let me know when I can see the simulations' file. By Saturday, please give me a nearly complete paper (I mean the stuff we talked about  --everything except introduction. So (i) sims, (ii) edits to Sec III, (iii) search for missing cites and add, (iv) re-org Appendix, Supplement). Can wait on Related Work section; and on Sec II. 

\begin{proof}[Proof of Lemma \ref{lem:projcs}]
Recall from Algorithm \ref{algo:fed_proj_ls} that we need solve 
\begin{align*}
\shat_{i,cs} = \arg\min_{\s} \|(\s)_{\T_i^c}\|_1 \ \text{s.t.} \|\bphi\y_t - \bphi \s\| \leq \xi
\end{align*}
This is a problem of sparse recovery from partial subspace knowledge. To prove the correctness of the result, we first need to bound the $s$-level RIC of $\bpsi = \I - \Phat_{t-1}\Phat_{t-1}^\top$ where $s := (2\outfraccol + \missfraccol)\cdot n$. Under the assumptions of Theorem \ref{thm1_newnorst} (we only assumed that $\outfraccol \in P(1/\mu r)$ and $\missfraccol \in O(1/\mu r)$ but the actual requirement is $(2\missfraccol + \outfraccol) \cdot n \leq 0.01/ \mu r$), and Fact \ref{fact:simp_cent}, we have that 
\begin{align*}
\delta_s(\I - \Phat_{t-1} \Phat_{t-1}^\top) &= \max_{|\T| \leq s} \|\I_{\T}^\top \Phat_{t-1}\|^2 \\
&\leq \max_{|\T| \leq s} (\SE(\Phat_{t-1}, \P_{t-1}) + \|\I_{\T}^\top \P_{t-1}\|)^2
\end{align*}
Recall that for $t > 1$, $\SE(\Phat_j, \P_j) \leq \max(0.1\cdot 0.3^{j-1} \epsilon_{\init} + 0.5 \Delta_{tv}, \zz) \leq 0.2$ and from the incoherence assumption on $\P_t$'s, the second term above is upper bounded by $0.01$. Thus, $\delta_{s}(\bphi) \leq 0.3^2 < 0.15$. Next, consider the error seen by the modified-CS step, 
\begin{align*}
\norm{\b_i} &= \norm{\bpsi (\l_i + \bv_i)} \leq \norm{(\I - \Phat_{t-1} \Phat_{t-1}) \P_t \at_i} + \|\bv_i\| \\
&\leq \SE(\Phat_{t-1}, \P_t) \norm{\at_i} + \|\bv_i\|\\
&{\leq} (\SE(\Phat_{t-1},\P_{t-1}) + \SE(\P_{t-1}, \P_t))\sqrt{\mu r \lambda^+} + C\sqrt{r \lambda_v^+} \\
&{\leq} (0.3^{t-1} \epsilon_{\init} +  1.5\Delta_{tv}) \sqrt{\mu r \lambda^+} + C \sqrt{r_v \lambda_v^+}
\end{align*}
under the assumptions of Theorem \ref{thm1_newnorst}, the RHS of the above is bounded by $ \xmin / 15$. This is why we have set $\xi = \xmin/15$ in Algorithm \ref{algo:fed_nodet_given_init}. Using these facts, and $\delta_{s} (\bpsi) < 0.15$, we have that 
\begin{align*}
\norm{\xhat_{i,cs} - \x_i} &\leq  7 \xi = 7\xmin/15 < \xmin/2
\end{align*}
Consider support recovery. From above,
\begin{align*}
| (\xhat_{i,cs} - \s_i)_m | \leq \norm{\xhat_{i,cs} - \s_i} \leq 7 \xmin/15 < \xmin/2
\end{align*}
The Algorithm sets $\omega_{supp} = \smin/2$. Consider an index $m \in \T_{\sparse,i}$. Since $|(\s_i)_m| \geq \smin$,
\begin{align*}
\smin - |(\xhat_{i,cs})_m| &\le  |(\s_i)_m| - |(\xhat_{i,cs})_m| \ \\
&\le | (\s_i - \xhat_{i,cs})_m | < \frac{\smin}{2}
%|(\shatcs)_i| &\geq \frac{\smin}{2}
\end{align*}
Thus, $|(\xhat_{i,cs})_m| > \frac{\smin}{2} = \omega_{supp}$ which means $m \in \hat{\T}_{\sparse,i}$. Hence $\T_{\sparse,i} \subseteq \That_{\sparse,i}$. Next, consider any $m \notin \T_{\sparse,i}$. Then, $(\s_i)_m = 0$ and so
\begin{align*}
|(\xhat_{i,cs})_m| = |(\xhat_{i,cs})_m)| - |(\s_i)_m| &\leq |(\xhat_{i,cs})_m -(\s_i)_m| < \frac{\smin}{2}
%(\shatcs)_m &\leq \frac{\smin}{2}
\end{align*}
which implies $m \notin \That_{\sparse,i}$ and $\That_{\sparse,i} \subseteq \T_{\sparse,i}$ implying that $\That_{\sparse,i} = \T_{\sparse,i}$ and consequently that $\That_i := \T_i\cup \That_{\sparse,i} = \T_i \cup \T_{\sparse,i}$.

%Finally, we get an expression for $\et$ and bound it.
With $\That_{\sparse,i} = \T_{\sparse,i}$ and since $\T_{\sparse,i}$ is the support of $\s_i$, $\s_i = \I_{\T_{\sparse,i}} \I_{\T_{\sparse,i}}^\top \s_i$, and so
\begin{align*}
\hat{\x}_i &= \bm{I}_{\That_i}\left(\bpsi_{\That_i}^{\top}\bpsi_{\That_i}\right)^{-1}\bpsi_{\That_i}^\top(\bpsi \l_i + \bpsi \z_i + \bpsi \s_i + \bpsi \v_i) \\
&= \bm{I}_{\That_i}\left(\bpsi_{\That_i}^\top\bpsi_{\That_i}\right)^{-1} \I_{\That_i}^\top \bpsi(\l_i + \v_i) + \s_i + \z_i
\end{align*}
Thus, the estimate of the true-data $\lhat_i = \y_i - \hat{\x}_i$ satisfies
\begin{align*}
\lhat_i &= \l_i + \v_i - \bm{I}_{\That_i}\left(\bpsi_{\That_i}^\top\bpsi_{\That_i}\right)^{-1} \I_{\That_i}^\top \bpsi(\l_i + \v_i) 
\end{align*}
and thus $\e_i = \lhat_i -\l_i$ satisfies
\begin{align*}
\e_i &= -\bm{I}_{\That_i}\left(\bpsi_{\That_i}^\top\bpsi_{\That_i}\right)^{-1} \I_{\That_i}^\top \bpsi(\l_i + \v_i) + \v_i  \\
\norm{\e_i} &\leq \norm{\left(\bpsi_{\That_i}^\top\bpsi_{\That_i}\right)^{-1}} \|\I_{\That_i}^\top\bpsi (\l_i + \v_i)\| + \|\v_i\| \\
&\le 1.2  \|\b_i\| + \|\v_i\|
\end{align*}
\end{proof}

We next prove Lemma \ref{lem:fed_pca}. But before we prove this, under the conditions of Lemma \ref{cor:cent_pca_dd}, the result from \cite{rrpcp_jsait} also shows the following: 

\begin{align}\label{eq:perturb}
&\norm{\mathrm{perturb}} := \norm{\frac{1}{\alpha} \sum_i (\z_i \z_i^\top - \l_i\l_i^\top)} \nonumber \\
%&= \norm{\frac{1}{\alpha} \sum_i (\l_i \et_i^{\top} + \et_i \l_i^{\top} + \et_i \et_i^{\top} + \bv_i \bv_i^{\top} + \l_i \bv_i^{\top} + \bv_i \l_i^{\top} + \bv_i \et_i^{\top} + \et_i \bv_i^{\top})}, \nonumber \\
&\leq \norm{\frac{1}{\alpha} \sum_i \et_i \et_i^{\top}} + 2 \norm{\frac{1}{\alpha} \sum_i \l_i \et_i^{\top}} + 2 \norm{\frac{1}{\alpha} \sum_i \l_i \bv_i^{\top}} \nonumber \\
& + 2 \norm{\frac{1}{\alpha} \sum_i \bv_i \et_i^{\top}} + \norm{\frac{1}{\alpha} \sum_i \bv_i \bv_i^{\top}}, \nonumber \\
&\leq \left(6.6 \sqrt{\bz} q f + 4.4 \frac{\lambda_v^+}{\lambda^-} \right) \lambda^-
\end{align}
and
\begin{align*}
\lambda_r\left(\frac{1}{\alpha} \sum_i \l_i \l_i^{\top} \right) &\geq  0.99 \lambda^- .
\end{align*}

\begin{proof}[Proof of Lemma \ref{lem:fed_pca}] Before we prove 
There are the following two parts in the proof:
\begin{enumerate}
\item First, we show that $\Phat$ is {\em close} to $\tilde{\P}$ where $\tilde{\P}$ is the top $r$ left singular vectors of $\Z$. In particular, we show that $\SE(\Phat, \tilde{\P}) \leq \epsilon_{PM}/2$. This relies on application of Lemma \ref{lem:fed_app} to the matrix $\Z\Z^\top/\alpha$ with the appropriate parameters.  
\item  Next, we use centralized Principal Components Analysis in Sparse, Data-Dependent Noise (PCA SDDN) with $\z_i \equiv \y_i$ to show that the $\tilde{\P}$ is {\em close} to the true subspace, $\P$. Here too we show that $\SE(\tilde{\P}, \P) \leq \epsilon_{PM}/2$. Combining the above two results, and the triangle inequality gives $\SE(\Phat, \P) \leq \SE(\Phat, \tilde{\P}) +   \SE(\tilde{\P}, \P) \leq  \epsilon_{PM}$.
\end{enumerate}

Notice from \eqref{eq:perturb}, with high probability, the matrix $\Z\Z^{\top}$ has a good eigen-gap, i.e., 
\begin{align*}
\lambda_{r}(\Z\Z^\top) &= \lambda_r(\L\L^\top + \mathrm{perturb}) \geq \lambda_{r}(\L\L^\top) - \|\mathrm{perturb}\| \\
&\geq 0.99 \lambda^-  - \left(7.7 \sqrt{\bz}q f + 4.4 \frac{\lambda_v^+}{\lambda^-} \right) \lambda^- \\
\lambda_{r+1}(\Z\Z^\top) &\leq \lambda_{r+1}(\L\L^\top) + \|\mathrm{pertub}\| \\
&\leq \left(7.7 \sqrt{\bz}q f + 4.4 \frac{\lambda_v^+}{\lambda^-} \right) \lambda^- 
\end{align*}
Under the assumptions of Lemma \ref{lem:fed_pca}, $7.7 \sqrt{\bz} q f + 4.4 \lambda_v^+/\lambda^- \leq  2.5 \epsilon_{SE}$. Thus, for this matrix, $R < 0.99$ with high probability. The standard deviation of the channel noise in each iteration satisfies, $\sigma_c \leq \epsilon_{PM}\lambda^-/10\sqrt{n}$. Furthermore, since we initialize Fed-PM with $\P_{\init}$ that satisfies $\SE(\P_{\init}, \P) \leq \epsilon_{\init, PM}$ it follows from second part of Lemma \ref{lem:fed_app} that after $L = C \log (1/(\epsilon_{PM}\sqrt{1-\epsilon_{\init, PM}^2}))$ iterations, with probability at least $1 - L \exp(-cr)$, the output $\Phat$ satisfies $\SE(\Phat, \tilde{\P}) \leq \epsilon_{PM}/2$. 

Next, observe that the conditions required to apply Lemma \ref{cor:cent_pca_dd} is satisfied under the assumptions of Lemma \ref{lem:fed_pca}. Thus, we apply Lemma \ref{cor:cent_pca_dd} with $\epsilon_{\mathrm{SE}} \equiv \epsilon_{PM}/2$. This ensures that with probability at least $1 - 10n^{-10}$, the eigenvectors of the empirical covariance are close to that of the the population covariance, i.e., $\SE(\tilde{\P}, \P) \leq \epsilon_{PM}/2$.

Combining the above two results we have with probability at least $1 - L\exp(-cr) - 10n^{-10}$, $\SE(\Phat, \P) \leq \SE(\Phat, \tilde{\P}) + \SE(\tilde{\P}, \P) \leq \epsilon_{PM}$. 
\end{proof}

The proof of the subspace detection step (Lemma \ref{lem:sschangedet}) is similar to that of \cite{rrpcp_jsait} applied with Lemma \ref{lem:fed_app_eval}. 

\begin{proof}[Proof of Lemma \ref{lem:fed_app}]
The proof of Lemma \ref{lem:fed_app} is a special case of Lemma \ref{thm:main_res} that is proved in the Supplementary Material. The proof of Lemma \ref{lem:fed_app_eval} is also provided in the Supplementary Material. 
\end{proof}
%?? specify the threshold in the Algo and specify for eigenvalues of (.) (.)' so square it. Also check above edits.

\clearpage

\section*{Supplementary Material}
%?? removed rst-miss-cent
The Supplementary Material is organized as follows. In Appendix \ref{sec:supp_ext}, we provide the setting, algorithm and the guarantee for (a) a generalization of Theorem \ref{thm:central_timevar} wherein we provide our result to provably detect and track large, but infrequent subspace changes; and (c) a generalization of Theorem \ref{thm1_newnorst} wherein we again deal with large infrequent subspace changes but in under the federated over-air constraints. In Appendix \ref{sec:rst_miss_central}, we provide the guarantee for robust ST-miss in a centralized setting. And finally, in Appendix \ref{sec:proof_fedpm}, we prove the convergence of Algorithm \ref{algo:rankr}, i.e., Lemma \ref{lem:fed_app} and Lemma \ref{lem:fed_app_eval} (in fact, we prove a stronger result there, but only provide a special case of it in the main paper). 

\section{Extensions of Theorem \ref{thm:central_timevar} and Theorem \ref{thm1_newnorst}}\label{sec:supp_ext}

\subsection{Generalization to detect and track larger subspace changes for centralized ST-miss}
When $\Delta_{tv}$ is small enough, the bound given by Theorem \ref{thm:central_timevar} holds and is better than that for simple PCA given in Theorem \ref{thm:naive}. When $\Delta_{tv}$ is very small but there are occasional large changes, then the guarantee of Theorem \ref{thm:large_ss} applies. However, the result does not guarantee change detection (only tracking), this is because the algorithm itself does not contain a detection step. In this section, we provide a modification of our algorithm that contains a detection step and a corollary that also guarantees quick enough detection. The proof is essentially a direct combination of the ideas given in the main paper and those used in \cite{rrpcp_tsp19} for quick and reliable subspace change detection.

 A simple modification to Algorithm \ref{algo:norst_nodet} given in Algorithm \ref{algo:norst_det} allows us to deal with such a model. Our next result shows that under such a subspace change model, we are able to recover the result of \cite{rrpcp_tsp19}. Concretely, consider the following subspace change model
\begin{definition}[Small frequent and abrupt infrequent subspace change model]\label{def:large_ss}
 Assume that the $\gamma$-th {\em large subspace changes} occurs at $t = t_{\gamma}\alpha$ for $\tau = 1, \cdots \Gamma$ such that $t_{\gamma + 1} - t_{\gamma} > (J^\ast + 2)$ with $ J^\ast := C \log(1/\epsilon)$ where $\epsilon$ chosen by the user denotes the desired final accuracy. In addition, assume that
\begin{gather*}
\min_{\gamma \in [\Gamma]} \Delta_{t_\gamma} \geq \Delta_{\mathrm{large}} \ge \Delta_{tv} \ge \max_{\{j: j \neq t_{\gamma}, \gamma \in [\Gamma]\}} \Delta_{j}
\end{gather*}
%\SE(\P_{t_{\gamma}+1}, \P_{t_{\gamma}})   %\SE(\P_{j+1}, \P_{j})
\end{definition}
Notice that this is a generalization of the ``model'' considered in previous literature on provable subspace tracking \cite{rrpcp_tsp19} and the model considered above.

To deal with the large subspace changes, we need a few minor changes to Algorithm \ref{algo:norst_nodet}, which we briefly summarize below. Firstly, the algorithm now has two phases, the subspace update phase and the subspace detect phase, akin to the algorithms of \cite{rrpcp_tsp19, rrpcp_jsait}. Second, as opposed to the algorithms studied in \cite{rrpcp_tsp19, rrpcp_jsait}, the current algorithm updates the subspace even in the detect phase (this is necessary since we only assume an approximate piecewise-constant subspace change model). The pseudo-code is provided as Algorithm \ref{algo:norst_det} in the  Appendix. With these changes, we have the following result

\begin{corollary}[Subspace tracking in the presence of infrequent, abrupt changes]
Assume that data satisfies the subspace change model in Definition \ref{def:large_ss} such that $\Delta_{\mathrm{large}} >  9 \sqrt{f} \max(0.1 \cdot 0.3^{J^\ast-1} + 1.5 \Delta_{tv}, \zz)$ Then, under the conditions of Theorem \ref{thm:central_timevar} and using Algorithm \ref{algo:norst_det}, with probability at least $1 - d n^{-10}$, the $\gamma$-th large subspace change is detected within $1$ mini-batch of $\alpha$ frames, i.e., $t_{\gamma} \leq \that_{\gamma} \leq t_{\gamma} + 1$ and
\begin{align*}
&\SE(\Phat_j, \P_j) \leq  \\
&\begin{cases}
\max(0.1 \cdot 0.25, \zz), \quad \text{if} \quad j^\ast = 1 \\
\max(0.1 \cdot 0.3^{j^\ast-1} + 0.5 \Delta_{tv}, \zz), \quad \text{if} \quad j^\ast \in [2, J^\ast]   \\
\epsilon, \quad \text{if} \quad j^\ast > J^\ast
\end{cases}
\end{align*}
where $j^\ast = \min_{\gamma} (\that_{\gamma} - j)$
\label{thm:large_ss_cent}
\end{corollary}

\begin{proof}[Proof of Corollary \ref{thm:large_ss_cent}]
The proof follows from the idea of the result of \cite{rrpcp_jsait}. The analysis of the subspace update step is exactly as mentioned in the proof of Theorem \ref{thm:central_timevar}. The proof of the subspace update step requires the following changes to \cite[Lemma 6.20]{rrpcp_jsait}: consider the case when the subspace has not changed, but $K = C \log(1/\epsilon)$ subspace updates have been completed. In this case, $q_K$ (the subspace error between the previous estimate and the current actual subspace) from \cite[Lemma 6.20]{rrpcp_jsait} gets replaced with $\epsilon + \Delta_{tv}$ using the triangle inequality for subspace errors. Next consider the case when the subspace has changed by a quantity of $\Delta_{\mathrm{large}}$. In this case, $q_1$ (the subspace error between the current algorithm estimate and the subspace {\em after} the large subspace change) gets replaced with $\epsilon + \Delta_{tv} + \Delta_{\mathrm{large}}$. Once we make these changes, the rest of the proof follows exactly in the same fashion, and we get that (i) if the subspace has changed, $\lambda_{\max}(\bphi \Lhat_j \Lhat_j^\top \bphi) \geq 5 (\epsilon + \Delta_{tv} + \Delta_{\mathrm{large}})^2 \lambda^+ $, and (ii) if the subspace has not changed, $\lambda_{\max}(\bphi \Lhat_j \Lhat_j^\top \bphi) \leq  1.5 \epsilon^2 \lambda^+$. Thus, under the conditions of Corollary \ref{thm:large_ss_cent}, as long as $\lthres = 2 \epsilon^2 \lambda^+$, w.h.p. the large subspace change is detected.
\end{proof}

\begin{algorithm}[t!]
\caption{STMiss -- Infrequent Abrupt Changes}
\label{algo:norst_det}
\begin{algorithmic}[1]
\Require $\Y$, $\T$ %$r'$, $\tau_t^* = C \log(n /\epsilon_t)$
\State \textbf{Parameters:} %$T_{iter} \leftarrow C \log(1/\zz)$, $\mathrm{phase} \leftarrow \mathrm{update}$,
%$L \leftarrow C \log (nr/\varepsilon)$, %$L_{\mathrm{det}} \leftarrow C \log(nr/\varepsilon)$, $\lthres \leftarrow 2 \zz^2\lambda^+$. \textbf{Parameters:}
 $\alpha$, $\epsilon$,
\State \textbf{Init:} $\Phat_1 \leftarrow$  $r$-$\SVD[\y_{1}, \cdots, \y_{\alpha} ] $, $j\leftarrow 2$, $k \leftarrow 2$, $\mathrm{phase} \leftarrow \mathrm{update}$, $K = C \log(1/\epsilon)$
%\State $\Lhat_0 \leftarrow$ \Call{Rob-FedOA-ProjLS}{$\calI_{k,0}$, $\y_i$, $\T_i$, $\Phat_{0}$}
%\State $\Phat_{1} \leftarrow$  \Call{FedOA-PM}{$\Lhat_{1}$, $r$, $L$}
\For{$j \geq 2$ } %\in ((b-1)\alpha, b\alpha]$}
\If{$k = 1$}
\State $\Phat_j \leftarrow r$-$\SVD[\y_{(j-2)\alpha+1}, \cdots, \y_{(j-1)\alpha} ]$
\Else
\If{$\mathrm{phase} = \mathrm{update}$}
\State $\bpsi \leftarrow \bm{I} - \Phat_{j-1}\Phat_{j-1}^{\top}$
\ForAll  {$t \in ((j-1)\alpha, j\alpha]$}
\State $\tilde{\y}_{t} \leftarrow \bpsi \y_{t}$; $\lhat_t \leftarrow \y_{t} -  \I_{{\T}_{t}} (\bpsi_{{\T}_t})^{\dagger} \tilde{\y}_{t}$.
\EndFor
%\If {$t = t_{\train} + u \alpha - 1$ for $u = 1,\ 2,\ \cdots,$}
\State $\Phat_{j} \leftarrow$  $r$-$\SVD[\lhat_{(j-1)\alpha + 1}, \cdots, \lhat_{j\alpha} ] $
%\State $\tilde{\bpsi} \leftarrow \bm{I} - \Phat_{j}\Phat_{j}^{\top}$

%\ForAll {$t \in ((j-1)\alpha, j\alpha]$} \Comment{optional}
%\State $\tilde{\y}_{t} \leftarrow \tilde{\bpsi} \y_{t}$; $\tlhat_{t} \leftarrow \y_{t} -  \I_{{\T}_{t}} (\tilde{\bpsi}_{{\T}_t})^{\dagger} \tilde{\y}_{t}$.
%\EndFor
\State $k \leftarrow k + 1$
\If{$ k = K$}
\State {$\mathrm{phase} \leftarrow \mathrm{detect}$}
\EndIf
\EndIf
\EndIf
\If{$\mathrm{phase} = \mathrm{detect}$}
\If{$\lambda_{\max}(\bphi \tLhat_j \tLhat_j^\top \bphi) \geq 2 \alpha \epsilon^2\lambda^+$}
\State $\mathrm{phase} \leftarrow \mathrm{update},\ k \leftarrow 1$ 
\Else
\State Repeat lines $5$ - $12$
\EndIf
\EndIf
\EndFor

\Ensure $\Phat_j$, $\lhat_t$, $\tlhat_t$.
\end{algorithmic}
\end{algorithm}

\subsubsection{Generalization to detect and track large subspace changes in FedOA-RST-miss}
Recall that $\P_t$ is the matrix of top-$r$ left singular vectors of data, $\Y_t = [\Y_{1,t}, \cdots, \Y_{K,t}]$. Assume that at $t=t_{\gamma}$ for $\gamma = 1, \cdots, \Gamma$, such that $t_{\gamma + 1} - t_{\gamma} > (J^* + 2)$ 	with $J^* = C\log(1/\epsilon)$ where $\epsilon$ is chosen by the user to denote the desired final accuracy. In addition, assume that
\begin{gather*}
\min_{\gamma \in [\Gamma]} \SE(\P_{t_{\gamma}+1}, \P_{t_{\gamma}}) \geq \Delta_{\mathrm{large}} \\
\max_{\{t: t \neq t_{\gamma}, \gamma \in [\Gamma]\}} \SE(\P_{t+1}, \P_{t}) \leq \Delta_{tv}
\end{gather*}

\begin{corollary}
Assume that the data satisfies the subspace change model specified above such that $\Delta_{\mathrm{large}} >  9 \sqrt{f} \max(0.1 \cdot 0.3^{J^\ast-1} + 1.5 \Delta_{tv}, \zz)$ . Then, under the conditions of Theorem \ref{thm1_newnorst} and with minor modifications to Algorithm \ref{algo:fed_nodet_given_init},
with probability at least $1 - d n^{-10}$, the $\gamma$-th large subspace change is detected within $1$ time instant, i.e., $t_{\gamma} \leq \that_{\gamma} \leq t_{\gamma} + 1$ and
\begin{align*}
&\SE(\Phat_t, \P_t) \leq  \\
&\begin{cases}
\max(0.1 \cdot 0.25, \zz), \quad \text{if} \quad j^\ast = 1 \\
\max(0.1 \cdot 0.3^{j^\ast-1} + 0.5 \Delta_{tv}, \zz), \quad \text{if} \quad j^\ast \in [2, J^\ast]   \\
\epsilon, \quad \text{if} \quad j^\ast > J^\ast
\end{cases}
\end{align*}
where $j^\ast = \min_{\gamma} (\that_{\gamma} - t)$
\label{thm:large_ss_fed}
\end{corollary}

For the proof of Corollary \ref{thm:large_ss_fed}, the approach is the same as Corollary \ref{thm:large_ss_cent}. One key difference is how we perform the subspace detection step since this needs to be done while obeying the federated, over-air data sharing constraints. To do this, we leverage Lemma \ref{lem:fed_app_eval} and derive the following result:

\begin{lem}[Subspace Change Detection]\label{lem:sschangedet}
Consider $\alpha$ data vectors at time $t > t_{\gamma-1}$. Assume that the $(t-1)$-th subspace has been estimated to $\epsilon$-accuracy, i.e., $\SE(\Phat_{t-1}, \P_{t-1}) \leq \epsilon$. Let the number of iterations of Fed-PM be $L_{\mathrm{det}} = C \log nr$. Let the detection threshold $\lthres = 2\epsilon^2 \alpha \lambda^+$. Then, under the assumptions of Theorem \ref{thm1_newnorst}, the following holds.
%and let  $\Phat_{\det}$ denote the output of line $14$ of Algorithm \ref{algo:stmiss_dyn}
\ben
\item If the subspace changes, i.e., $t > t_{\gamma}$. At this time, with probability at least $1 - 10n^{-10}$,
\begin{gather*}
\hspace{-.7cm} \hat{\lambda}_{\text{det}} \geq 0.9 \lambda_{\max}\left(\bphi \Lhat_t \Lhat_t^{\top} \bphi \right)
\geq 4.5 (\epsilon + \Delta_{tv} + \Delta_{\mathrm{large}})^2 \alpha \lambda^+ % > \alpha \cdot \lthres
\end{gather*}
\item If the subspace has not changed, then with probability at least $1 - 10n^{-10}$,
\begin{gather*}
\hat{\lambda}_{\text{det}}
\leq 1.1 \lambda_{\max}\left( \bphi \Lhat_t \Lhat_t^{\top} \bphi \right)
 \leq 1.6 \epsilon^2 \alpha  \lambda^+ %< \alpha \cdot \lthres
\end{gather*}
\een
\end{lem}

\section{Robust Subspace Tracking with Missing Data}\label{sec:rst_miss_central}
In this section, we provide the concrete problem setting, algorithm and result for RST-miss in the centralized setting. Assume that at each time $t$, we observe an $n$-dimensional data stream of the form 
\begin{align}\label{eq:rob_time_var}
\y_t = \proj_{\Omega_t}(\tl_t + \g_t), \quad t = 1, 2, \cdots, \tmax 
\end{align}
where $\g_t$'s are the sparse outliers and $\tl_t$, $\proj_{\Omega_t}(\cdot)$ etc are defined exactly as done before. We let $\s_t := \proj_{\Omega_t}(\g_t)$ and let $\Tspart$ denote the support of $\s_t$. Notice that it is impossible to recover $\g_t$ on the set $\Tmisst$ and thus we only work with $\s_t$ in the sequel. Furthermore, by definition, $\s_t$ is supported outside $\Tmisst$ and thus $\Tmisst$ and $\Tspart$ are disjoint. With $\s_t$ defined as above, the measurements can also be expressed as 
\begin{align*}
\y_t &= \proj_{\Omega_t}(\tl_t) + \s_t  \\
&= \tl_t - \I_{\T_t}\I_{\T_t}^\top \tl_t + \s_t \\
&:= \tl_t + \z_t + \s_t = \l_t + \z_t + \s_t + \v_t.
\end{align*} 
One main difference required in the algorithm is how we estimate the sparse vector, $\tilde{\s}_t = \z_t + \s_t$. Recovering $\tilde{\s}_t$ is a problem of sparse recovery with partial support knowledge, $\Tmisst$. In this paper, we use noisy modified CS \cite{modcs} which was introduced to solve exactly this problem. Another main difference is in the initialization step. Observe that due to the presence of sparse outliers, a simple PCA step does not ensure a ``good enough'' initialization in this case. 

\subsubsection{Assumptions}
We need all the assumptions from the previous section. In addition, it is well known from the RPCA literature that the fraction of outliers in each row and column of the matrix $\S_j$ needs to be bounded. 
\begin{definition}[Sparse outlier fractions]
Consider the sparse outlier matrix $\S_j:=[\s_{(j-1)\alpha+1}, \dots, \s_{j\alpha} ] $ . We use $\outfraccol$ ($\outfracrow$)  to denote the maximum of the fraction of non-zero elements in any column (row) of this matrix. Also define $\xmin  = \min_{t\in ((j-1\alpha,j\alpha]} \min_{i \in  \T_{\sparse, t}} |(\s_t)_i|$.
\end{definition}

\subsubsection{Algorithm and Main Result}

\begin{algorithm}[t!]
\caption{Robust Subspace Tracking with missing entries (RST-miss)} 
\label{algo:rst_miss}
\begin{algorithmic}[1]
\Require $\Y$, $\T$ %$r'$, $\tau_t^* = C \log(n /\epsilon_t)$
\State \textbf{Parameters:} %$T_{iter} \leftarrow C \log(1/\varepsilon)$, $\mathrm{phase} \leftarrow \mathrm{update}$, 
%$L \leftarrow C \log (nr/\varepsilon)$, %$L_{\mathrm{det}} \leftarrow C \log(nr/\varepsilon)$, $\lthres \leftarrow 2 \zz^2\lambda^+$. \textbf{Parameters:} 
 $\alpha$, $\lthres$, $\xi$ 
\State \textbf{Init:} $\Phat_1 \leftarrow$  AltProj$[\y_{1}, \cdots, \y_{\alpha} ] $, $j\leftarrow 2$
\State Lines $3$-$13$ of Algorithm \ref{algo:norst_nodet} with line $6$ replaced by 
\Statex \hspace{.2cm} $\tilde{\y}_{t} \leftarrow \bpsi \y_{t}$ 
\Statex \hspace{.2cm} $\xhat_{t,cs} \leftarrow \arg \min_{\x} \|(\x)_{(\T_t)^c}\|_1$ s.t. $\|\tilde{\y}_t - \bpsi \x\| \leq \xi$. 
\Statex \hspace{.2cm} $\hat{\T}_t \leftarrow \T_t \cup \{m : |(\xhat_{t,cs})_m| > \omega_{supp}\}$ 
\Statex \hspace{.2cm} $\lhat_{t} \leftarrow \y_{t} -  \I_{\hat{\T}_{t}} (\bpsi_{\hat{\T}_t})^{\dagger} \tilde{\y}_{t}$.
\State Line $11$ replaced by 
\Statex \hspace{.2cm} $\tilde{\y}_{t} \leftarrow \tilde\bpsi \y_{t}$ 
\Statex \hspace{.2cm} $\xhat_{t,cs} \leftarrow \arg \min_{\x} \|(\x)_{(\T_t)^c}\|_1$ s.t. $\|\tilde{\y}_t - \tilde\bpsi \x\| \leq \xi$. 
\Statex \hspace{.2cm} $\hat{\T}_t \leftarrow \T_t \cup \{m : |(\xhat_{t,cs})_m| > \omega_{supp}\}$ 
\Statex \hspace{.2cm} $\tlhat_{t} \leftarrow \y_{t} -  \I_{\hat{\T}_{t}} (\tilde\bpsi_{\hat{\T}_t})^{\dagger} \tilde{\y}_{t}$.
\Ensure $\Phat_j$, $\lhat_t$, $\tlhat_t$, $\That_t$.
\end{algorithmic}
\end{algorithm}

We have the following result for robust subspace tracking with missing entries
\begin{theorem}[Robust Subspace Tracking with missing entries]
Consider Algorithm \ref{algo:fed_nodet_given_init}. Assume that $\zz \leq 0.2$. Set $\omega_{supp} = \xmin/2$ and $\xi = \xmin/15$. Assume that the following hold:
\begin{enumerate}
\item At $t=1$ we are given a $\Phat_1$ s.t. $\SE(\Phat_1, \P_1) \leq \epsilon_{\init}$.

\item  {\bf Incoherence:} $\P_j$'s satisfy $\mu$-incoherence, and $\at_t$'s satisfy statistical right $\mu$-incoherence;

\item {\bf Missing Entries:} $\missfraccol \in O(1/\mu r)$, $\missfracrow \in O(1)$;

\item {\bf Sparse Outliers:} $\outfraccol \in O(1/\mu r)$, $\outfracrow \in O(1)$;

%\item {\bf Channel Noise:} the channel noise seen by each FedOA-PM iteration is mutually independent at all times, isotropic, and zero mean Gaussian with standard deviation $\sigma_c \leq \varepsilon \lambda^-/10 \sqrt{n}$. %(this is required if we are using the noisy power method result).
\item {\bf Subspace Model:} let $\Delta_{tv} := \max_j \SE(\P_{j-1},\P_j)$ s.t. 
\begin{gather*}
0.3 \epsilon_{\init} + 0.5\Delta_{tv} \leq 0.28 \quad \text{and} \\
C \sqrt{r \lambda^+} (0.3^{j-1}\epsilon_{\init} + 0.5\Delta_{tv} ) + \sqrt{r_v\lambda_v^+} \leq \xmin 
\end{gather*}
\end{enumerate}
then, with probability at least $1 - 10 \tmax n^{-10}$, 
for all $j >1$, we have
%the $j$-th subspace change is detected immediately i.e., $t_j \leq \that_j \leq t_j + 1$, and its tracking error decays exponentially after detection, i.e., $\SE(\Phat_{t}, \P_t) \le$
 \begin{align*}
&\SE(\Phat_j, \P_j) \\
&\leq \max(0.3^{j-1} \epsilon_{\init} + \Delta_{tv} (0.3 + 0.3 + 0.3^2 ... + 0.3^{j-1}), \zz) \\
&< \max(0.3^{j-1} \epsilon_{\init} + 0.5\Delta_{tv}, \zz) 
\end{align*}
Also, at all $j$ and $t \in ((j-1)\alpha, j\alpha]$, $\|\tlhat_t-\l_t\| \le 1.2 \cdot  \SE(\Phat_j, \P_j) \|\tl_t\| + \|\bv_t\|$.
\label{thm:central_rstmiss}
\end{theorem}

%\clearpage

%\clearpage
%comments to be addressed in fedpm proof section
%
%\begin{enumerate}
%\item master -- central server -- done
%\item $\Y$ - $\Z$
%\item $\tau$ - $\eta$
%\item define $A = \Z\Z^\top = \bm{U} \bm{\Sigma} \bm{V}^\top$ and not lambda to avoid confusion with next section ??pn mostly done check
%\item getting rid of $Q$ for subspaces ??pn dne
%\item get rid of t as index for fedpm ?? pn done
%\item use bm for all matrices -- done i think
%\item add a simple descent lemma for $\eta=1$ ??pn done
%\item can we use $\mathrm{gap}$ instead of $R$? is there a better notation for $\tilde R$?
%\end{enumerate}
%
%??copied here for now, can delete if not required.

\newcommand{\Sig}{\bm{\Sigma}}

%\clearpage
\section{Convergence Analysis for FedPM}\label{sec:proof_fedpm}

\subsection{Eigenvalue convergence}
First we present the proof of the eigenvalue convergence result (Lemma \ref{lem:fed_app_eval}). To our best knowledge, this has not been studied in the federated ML literature.
\begin{proof}[Proof of Lemma \ref{lem:fed_app_eval}]
We now wish to compute the error bounds of in convergence of eigenvalues. To this end, at the end of $L$ iterations, we compute $\hat{\Sig} = \Qhat_L^{\top} \bm{A} \Qhat_L + \Qhat_L^{\top}\W_L$. The intuition is that if the eigenvectors are estimated well, then this matrix will be approximately diagonal (off diagonal entries $\approx \epsilon$), and the diagonal entries will be close to the true eigenvalues. Furthermore, in the application of this result for the Subspace Change detection problem, we will only consider the largest eigenvalue of $\hat \Sig$ and thus we have
\begin{align*}
\lambda_{\max} (\hat{\Sig}) &=  \lambda_{\max}(\Qhat_L^{\top} \bm{A} \Qhat_L + \Qhat_L^{\top} \W_L) \\
&= \lambda_{\max}(\Sig + (\Qhat_L^{\top} \bm{A} \Qhat_L - \Sig) + \Qhat_L^{\top}\W_L) \\
&\geq \lambda_{\max}(\Sig) - \|\Qhat_L^{\top} \bm{A} \Qhat_L - \Sig\| - \|\Qhat_L^{\top}\W_L\| \\
&\geq \sigma_1 - \|\Qhat_L^{\top} \bm{A} \Qhat_L - \Sig\| - \|\W_L\|
\end{align*}
The second term can be upper bounded as follows
\begin{align*}
&\|\Qhat_L^{\top} \bm{A} \Qhat_L - \Sig \| \\
&= \|(\Qhat_L^{\top} \U \Sig \U^{\top} \Qhat_L - \Sig) + \Qhat_L^{\top} \U_{\perp} \Sig_{\perp} \U_{\perp}^{\top} \Qhat_L\| \\
&\leq \| \Qhat_L^{\top} \U \Sig \U^{\top} \Qhat_L - \Sig \| + \|\Qhat_L^{\top} \U_{\perp} \Sig_{\perp} \U_{\perp}^{\top} \Qhat_L\| \\
&\leq \| \Qhat_L^{\top} \U \Sig \U^{\top} \Qhat_L - \Sig \| + \|\Sig_{\perp}\| \|\U_\perp^{\top} \Qhat_L\|^2 \\
&= \|\Qhat_L^{\top} \U \Sig \U^{\top} \Qhat_L - \Sig \| + \|\Sig_{\perp}\| \|\U_\perp \U_\perp^{\top} \Qhat_L\|^2 \\
&\leq \|\Qhat_L^{\top} \U \Sig \U^{\top} \Qhat_L - \Sig \| + \sigma_{r+1} \SE^2(\Qhat_L, \U) \\
\end{align*}
The first term above can be bounded as
\begin{align*}
&\|\Qhat_L^{\top} \U \Sig \U^{\top} \Qhat_L - \Sig \| \\
&= \|(\bm{I} -\bm{I} +  \Qhat_L^{\top} \U )\Sig (\U^{\top} \Qhat_L + \bm{I} - \bm{I}) - \Sig \| \\
&\leq \|(\Qhat_L^{\top} \U - \bm{I}) \Sig\| + \|\Sig (\U^{\top} \Qhat_L - \bm{I})\| \\
&+ \|(\Qhat_L^{\top} \U - \bm{I}) \Sig (\U^{\top} \Qhat_L - \bm{I})\| \\
&\leq \sigma_1 (2\|\bm{I} - \Qhat_L^{\top} \U\| + \|\bm{I} - \Qhat_L^{\top} \U\|^2) \\
&\leq \sigma_1 ( 2( 1 - \sigma_r(\Qhat_L^{\top} \U)) + (1 - \sigma_r(\Qhat_L^{\top} \U))^2)
\end{align*}
and since $\SE^2(\Qhat_L, \U) = 1 - \sigma_r^2(\Qhat_L^{\top} \U) \leq \epsilon^2$ and thus we get that $\sigma_r(\Qhat_L^{\top}\U) \geq \sqrt{1 - \epsilon^2} \geq 1 - \epsilon^2$. Finally, the assumption on the channel noise implies that with high probability, $\|\W_L\| \leq C \sqrt{n} \sigma_c \leq 1.5 \sigma_r \epsilon$. Thus,
\begin{align*}
\lambda_{\max}(\hat \Sig) \geq \sigma_1( 1 - 4 \epsilon^2) - \sigma_{r+1}\epsilon^2 - \sigma_r \epsilon
\end{align*}
We also get
\begin{align*}
\lambda_{\max}(\hat \Sig) &\leq \lambda_{\max}(\Qhat_L^{\top} \B \B^{\top} \Qhat_L) + \|\W_L\| \\
&\leq \|\Qhat_L\|^2 \|\B \B^{\top}\| + \|\W_L\| = \lambda_{\max}(\B \B^{\top}) + 1.5 \sigma_r \epsilon
\end{align*}
This completes the proof.
\end{proof}

\subsection{The Noise Tolerant FedOA-PM, Algorithm, and Guarantee}
Next, we present a ``robust'' version of the FedOA-PM algorithm. As mentioned earlier, by normalizing the subspace estimates once every $\eta \geq 1$ iterations allows for a larger noise tolerance than the vanilla FedOA-PM algorithm. This is summarized in Algorithm \ref{algo:rankr_eta} and the main result is provided below.

\begin{algorithm}[t!]
%\vspace{-.5cm}
\caption{FedPM: Federated Noise-Tolerant Power Method }\label{algo:rankr_eta}
  \begin{algorithmic}[1]
        \Require $\Z$, $r$, $L$, $\eta$, $K$ nodes, for each $i \in \calI_k$, data $\y_i$ %, $\calI_k$
            \State At central server, $\Uhat_0 \overset{i.i.d.}{\sim} \mathcal{N}(0, I)_{n \times r}$; $\Qhat_0 \leftarrow \Uhat_0$, transmit to all $K$ workers.
  	\For{$l =1,\dots, L\eta$}
  	 \State At $k$-th node, do $\tilde{\U}_{k,l} = \Z_{k} \Z_{k}^{\top} \Qhat_{l-1} $
  	 \State All $k$ nodes transmit $\tilde{\U}_{k,l}$ synchronously  to the central server.
    %\STATE central server receives the sum of the $K$ transmissions corrupted by channel noise, denotes $\Uhat_l$ which satisfies $\Uhat_\eta := \sum_k \tilde{\U}_{k, l} + \W_{k,l }$, where $\sum_k \W_{k,l } = \W_{l}$ is channel noise.
    \State Central server receives $\Uhat_\eta := \sum_k \tilde{\U}_{k, l} + \W_{k,l }$, with $\sum_k \W_{k,l } = \W_{l}$.
    \State $\Qhat_l \leftarrow \Qhat_{l-1}$ %\COMMENT{standard updating}
	\State \textbf{if} $ (l \mod \taubatch) = 0 $ {\textbf{then}} $\Qhat_l \bm{R}_l \overset{QR}{\leftarrow} \Uhat_l$ \textbf{end if}
    \State Central server broadcasts $\Qhat_l$ to all nodes
    \EndFor
    \State All $k$ nodes compute $\Z_k \Z_k^{\top}\Qhat_{L }$, transmit synchronously to central server
    \State Central server receives $\bm{B} = \sum_k \Z_k \Z_k^{\top}\Qhat_{L } + \W_{L}$, computes the top eigenvalue, $\hat{\sigma}_1 = \lambda_{\max}(\Qhat_{L}^{\top} \bm{B})$.
    \Ensure $\Qhat_{L }$, $\hat{\sigma}_1$.
  \end{algorithmic}
\end{algorithm}

Before we state the main result, we need to define the following quantities. For this section we use $\A = \Z\Z^\top$ and let $\A \overset{EVD}{=} \U \Sig \U^\top + \U_{\perp}\Sig_{\perp}\U_{\perp}^\top$ denote its eigenvalue decomposition. Recall that $\U \in \R^{n \times r}$ denote the principal subspace that we are interesting in estimating. We also use $\sigma_i$ to denote the $i$-th eigenvalue of $\A$ with $\sigma_1 \geq  \cdots \geq \sigma_r > \sigma_{r+1} \geq \cdots \geq  \sigma_n \geq 0$. We also let the ratio of $(r+1)$-th to $r$-th eigenvalue, $R : = \sigma_{r+1}/\sigma_r$, the noise to signal ratio, $\nsrmax := \sigma_c/\sigma_r$, and $\tilde{R} := \max(R, 1/\sigma_r)$. We use $\SE_l := \SE(\Qhat_l, \U)$. 

We have the following main result:

\begin{theorem}\label{thm:main_res}
Consider Algorithm \ref{algo:rankr_eta} with initial subspace estimation error $\SE_0$.
\begin{enumerate}
\item Let $\taubatch = 1$. Assume that $R < 0.99$.
If, at each iteration, the channel noise $\W_l$ satisfies
$
 \nsrmax < c \min \left(  \frac{\epsilon}{\sqrt{n}}, 0.2 \sqrt{\frac{1 - \SE_{l-1}^2}{r}} \right)
$
then, after $L = \Omega \left( \frac{1}{\log(1/R)} \left( \log \frac{1}{\epsilon}  + \log\frac{1}{\sqrt{1- \SE_0^2}} \right)  \right)$
iterations, with probability at least $1 - L \exp(-cr)$, $\SE(\U, \Qhat_L) \le \epsilon$.
\item Consider Algorithm \ref{algo:rankr} with $\taubatch >1$.
If, $\sigma_r > 1$,  and if
$
 \nsrmax < c \min \left( \frac{\epsilon }{\sqrt{n}}  \cdot  \frac{1}{\sqrt{\taubatch} R^{\taubatch -1}},   0.2 \sqrt{\frac{\sigma_r^2 - 1}{\sigma_r^2}} \cdot \sqrt{\frac{1 - \SE_{(l-1)\taubatch}^2}{r}} \right),
$
then the above conclusion holds.
\item If $\Uhat_0 \overset{i.i.d}{\sim} \mathcal{N}(0,\I)_{n \times r}$, then $\SE_0  = \mathcal{O}(\sqrt{1 - 1/\gamma nr})$ with probability $1- 1/\gamma$.
\end{enumerate}
\end{theorem}

%{\textbf Note:} In our main result, we use the Taylor series approximation for 
%\begin{align*}
%\frac{1}{\sqrt{1- \SE_0^2}} = 1 + 0.5\SE_0^2
%\end{align*}
%and this might seem to give an incorrect answer for values of $\SE_0 \approx 1$. The reason is that the approximation is valid for values of $\SE_0 \ll 1$. Since this is the setting that we generally consider in the tracking scenario, we present the results in this way.

To understand the above theorem, first consider $\taubatch=1$. %Then, $\Gamma_{num}(1) = \Gamma_{denom}(1) =1$. 
In this case, %as the first term of the numerator goes to $0$ exponentially, 
we require $\nsrmax \sqrt{n} < \epsilon$  to achieve $\epsilon$-accurate recovery of the subspace. % and $L = \Omega(\log (1/\epsilon)/ \log (1/R) )$ to achieve $\epsilon$-accurate recovery of the subspace.
In this setting, with a random initialization, our result essentially recovers the main result of \cite{noisy_pm,improved_npm}. 
But we  can choose to pick $\taubatch>1$. To understand its advantage, suppose that $\lambda_r > 1.5$ (this is easy to satisfy by assuming that all the data transmitted is scaled by a large enough factor). Then, clearly, $\lambda_r^2/(\lambda_r^2-1) < 3$ and so the first term in the upper bound of $\nsrmax$ dominates. Thus, as $\taubatch$ is increased, we only require $\nsrmax \sqrt{n} \cdot \sqrt{\taubatch} R^{\taubatch - 1} \leq \epsilon$ which is a significantly weaker requirement. Thus, a larger $\taubatch$ means we can allow the noise variance to be larger. However, we cannot pick $\taubatch$ too large because it will lead to numerical problems (bit overflow problems) and may also result in violation of the transmit power constraint.
As an example, if we set $\taubatch = C \log n$, for a constant $C$ that is large enough (depends on $\tilde{R}$), then the we only require $(\nsrmax \sqrt{n} /\log n) \leq \epsilon$ which provides a $\log n$ factor of noise robustness. Observe that the  number of iterations needed, $L$, depends on the initialization. If $\SE_0 < c_0$ with $c_0$ being a constant, then we only need $L = \Omega\left( \frac{1}{\log(1/R)} \log (1/\epsilon) \right)$ iterations (which we leverage in the ST-miss result). Finally, if we use random initialization we need  $L = \Omega\left( \frac{1}{\log(1/R)} \log (nr /\epsilon) \right)$, i.e., $O(\log nr)$ more iterations. We provide a comparison with \cite{noisy_pm, improved_npm} in Table \ref{tab:comp}.

\renewcommand{\taubatch}{\eta}
\begin{table*}[t!]
%\begin{minipage}{0.6\textwidth}
%\begin{table}[H]
\centering
%\vspace{-.1cm}
\caption{\small{Comparing bounds on  channel noise variance $\sigma_c^2$  and on number of iterations $L$. Let $\gap_1 := \lambda_r - \lambda_{r+1}$, $\gap_q := \lambda_{r} - \lambda_{q+1}$ for some $r \leq q \leq  r'$. Also, we assume $\epsilon \le c / r$.
}}
\label{tab:comp}
\resizebox{.6\textwidth}{!}{
\begin{tabular}{c c c}
\toprule
& Noisy Power Method & This Work \\
 & {\footnotesize \cite{noisy_pm, improved_npm}} &  \\  \toprule
$\taubatch = 1$ & $\sigma_c = \mathcal{O}\left( \frac{\gap_1 \epsilon}{\sqrt{n}} \right)$ & $\sigma_c = \mathcal{O}\left( \frac{\lambda_r \epsilon}{\sqrt{n}} \right)$, \\
$r'=r$ & & $R < 0.99$ \\ \midrule
Random init & $L = \Omega\left( \frac{\lambda_r}{\gap_q} \log\left(\frac{n}{\epsilon}\right)\right)$ & $L = \Omega\left( \frac{1}{\log(1/R)} \log\left(\frac{n}{\epsilon}\right)\right)$ \\ \midrule
Good init & - & $L = \Omega\left( \frac{1}{\log(1/R)} \log\left(\frac{1}{\epsilon}\right)\right)$ \\  
($\SE_0 \le c_0$)  \\  \midrule
$\taubatch = 1$ & $\sigma_c = \mathcal{O}\left( \frac{\gap_q \epsilon }{\sqrt{n}} \right)$ &  --  \\
$r' > r$ & &  \\ \midrule
$\taubatch >1 $ & -- & $\sigma_c = \mathcal{O}\left( \frac{\lambda_r \epsilon }{R^\taubatch \taubatch} \right)$, \\
$r' = r$ & & $R < 0.99$, $\lambda_r >1$ \\ \midrule
$\taubatch = \mathcal{O}(\log n)$ & -- & $\sigma_c = \mathcal{O}\left( \frac{\lambda_r \epsilon n }{\log n} \right)$, \\
$r' = r$ & & $R < 0.99$, $\lambda_r >1$ \\ 
\bottomrule
\end{tabular}
}
%\vspace{-.2cm}
%\vspace{-0.25cm}
%\end{table}
%\end{minipage}
\end{table*}

\subsection{Proof of Theorem \ref{thm:main_res}}

Before we state the proof, we define two auxiliary quantities
\begin{gather*}
\Gamma_{num}^2(\eta) := \frac{1+\sigma_{r+1}^2 +\sigma_{r+1}^4 + \dots \sigma_{r+1}^{2\eta-2}}{\sigma_r^{2\eta-2}}, \\
\Gamma_{denom}^2(\eta): = \frac{1+\sigma_r^2 + \sigma_r^4 + \dots + \sigma_r^{2\eta-2}}{\sigma_r^{2\eta-2}}
\end{gather*}
Intuitively, $\Gamma_{num}(\eta)$ captures the effect of the ratio of the ``effective channel noise orthogonal to the signal space'', to the signal energy, while $\Gamma_{denom}(\eta)$ captures the ``effective channel noise along the signal space'' and the signal energy.  The following lemma bounds the reduction in error from iteration $(l-1)\taubatch$ to $l\taubatch$. %Recall $\taubatch$ is the number of iterations after which the QR decomposition is computed.
\begin{lem}[Descent Lemma, general $\eta$]\label{lem:desc}
Consider Algorithm \ref{algo:rankr_eta}. Assume that $R < 0.99$.
With probability at least $1-\exp(-c r)$, the following holds:
\begin{align*}
\SE_{l \taubatch} \leq \frac{R^{\taubatch} \ \SE_{(l-1)\taubatch} +  \sqrt{n} \ \nsrmax \ \Gamma_{num}(\taubatch) }{0.9 \sqrt{1 - \SE_{(l-1)\taubatch}^2} - \sqrt{r} \ \nsrmax \ \Gamma_{denom}(\taubatch)  }
\end{align*}
\end{lem}

By recursively applying the above lemma at each iteration, we have the following. It assumes that the initial subspace estimate has error  $\SE_0:= \SE(\Uhat_0,\U)$. The proof is provided in Appendix \ref{sec:proof_fedpm}.

\begin{proof}[Proof of Lemma \ref{lem:desc}]
Consider the setting where we normalize our subpsace estimates every $\eta$ iterations. In other words, we start with a basis matrix estimate at $l = l_0$, and then analyze the subspace error after $\eta$ iterations. In this case, the subspace update equations can be written as
\begin{gather*}
\Uhat_{l_0 + 1} =  \A \Qhat_{l_0} + \W_{l_0 +1} \\
\Uhat_{l_0 + 2} = \A \Uhat_{l_0 +1} + \W_{l_0 + 2} = \A^2 \Qhat_{l_0} + \A \W_{l_0 + 1} + \W_{l_0 + 2} \\
\vdots \\
\Uhat_{l_0 + \eta} = \Uhat_l = \A^{\eta} \Qhat_{l_0} + \sum_{i = 1}^{\eta} \A^{\eta - i} \W_{l_0 + i}
\end{gather*}
Recall that $\Qhat_{l_0} \overset{QR}{=} \Uhat_{l_0} \bm{R}_{l_0}$. Thus, we have
\begin{align*}
\Uhat_l &= \A^{\eta} \Uhat_{l_0} \bm{R}_{l_0}^{-1} + \sum_{i=1}^{\eta} \A^{\eta - i} \W_{l_0 +i} \\
&= \A^{\eta} (\U\U^{\top} \Uhat_{l_0} + \U_\perp \U_\perp^{\top} \Uhat_{l_0}) \bm{R}_{l_0}^{-1} \\
&+ \sum_{i=1}^{\eta} \A^{\eta - i} (\U \U^{\top} \W_{l_0 +i} + \U_{\perp} \U_\perp^{\top} \W_{l_0 + i}) \\
&= \U \Sig^{\eta} (\U^{\top} \Uhat_{l_0}) \bm{R}_{l_0}^{-1} + \U_\perp \Sig_\perp^{\eta} (\U_\perp^{\top} \Uhat_{l_0}) \bm{R}_{l_0}^{-1}  \\
&+ \sum_{i=1}^{\eta} \left[ \U \Sig^{\eta -i} (\U^{\top} \W_{l_0 +i}) + \U_\perp \Sig_\perp^{\eta -i} (\U_\perp^{\top} \W_{l_0 +i}) \right]
\end{align*}
and thus, $\SE(\U, \Qhat_l) = \SE(\U, \Uhat_l) = \|\U_\perp^{\top} \Uhat_l \bm{R}_{l}^{-1}\|$ simplifies to
\begin{align*}
&\SE(\U, \hat \U_l) \\
&= \norm{ \left[ \Sig_\perp^{\eta} (\U_\perp^{\top} \Uhat_{l_0}) \bm{R}_{l_0}^{-1} + \sum_{i=1}^{\eta} \Sig_\perp^{\eta -i} (\U_\perp^{\top} \W_{l_0 +i}) \right] \bm{R}_{t}^{-1}} \\
&\leq \left( \|\Sig_{\perp}^{\eta}\| \|\U_\perp^{\top} \Uhat_{l_0} \bm{R}_{l_0}^{-1}\| + \norm{\sum_{i=1}^{\eta} \Sig_\perp^{\eta -i} (\U_\perp^{\top} \W_{l_0 +i})} \right) \|\bm{R}_t^{-1}\| \\
&= \left(\|\Sig_{\perp}^{\eta}\| \SE(\U, \Uhat_{l_0})+ \norm{\sum_{i=1}^{\eta} \Sig_\perp^{\eta -i} (\U_\perp^{\top} \W_{l_0 +i})}\right) \|\bm{R}_t^{-1}\| \\
&\leq \frac{\|\Sig_\perp^{\eta}\|  \SE(\U, \hat \U_{l_0}) + \norm{\sum_{i=1}^{\eta} \Sig_\perp^{\eta-i} (\U_\perp^{\top} \W_{l_0 +i})} }{\sigma_r(\bm{R}_t)}
\end{align*}
We also have that
\begin{gather*}
\sigma_r^2(\bm{R}_t) = \sigma_r^2(\Uhat_t) \\
= \lambda_{\min}((\U\U^{\top} \Uhat_t + \U_\perp \U_\perp^{\top} \hat \U_t)^{\top} (\U\U^{\top} \Uhat_t + \U_\perp \U_\perp^{\top} \hat \U_t))  \\
\geq \lambda_{\min} (\Uhat_t^{\top} \U \U^{\top} \hat \U_t) = \sigma_r^2(\U^{\top} \hat \U_t) \\
\implies \sigma_r(\U^{\top} \Uhat_t) = \sigma_r\left( \Sig^{\eta} \left( \U^{\top} \Qhat_{l_0} + \sum_{i=1}^{\eta} \Sig^{-i} \U^{\top}\W_{l_0+i} \right) \right) \\
\geq \sigma_r^\eta \left[ \sigma_r(\U^{\top} \Qhat_{l_0}) - \norm{\sum_{i=1}^{\eta}\Sig^{-i} \U^{\top} \W_{l_0+i} } \right]
\end{gather*}

We define $\SE(\U, \Uhat_{l_0}) = \SE(\U, \Qhat_{l_0}) = \SE_{l_0}$ and $R = \sigma_{r+1}/\sigma_r$, $\tilde R = \max(1,\sigma_{r+1})/\sigma_r$ and thus we have
\begin{align*}
&\SE(\U, \hat \U_l) \\
&\leq \frac{\|\Sig_\perp^{\eta}\|  \SE(\U, \hat \U_{l_0}) + \norm{\sum_{i=1}^{\eta} \Sig_\perp^{\eta-i} (\U_\perp^{\top} \W_{l_0 +i})} }{\sigma_r^\eta \left[ \sqrt{1 - \SE^2(\U, \Uhat_{l_0})} - \norm{\sum_{i=1}^{\eta}\Sig^{-i} \U^{\top} \W_{l_0+i} } \right]} \\
&\leq \frac{R^\eta \SE_{l_0} +  \sigma_r^{-\eta} \| \sum_{i=1}^{\eta} \Sig_\perp^{\eta-i} \U_\perp^{\top} \W_{l_0 + i} \| }{\sqrt{1 - \SE^2_{l_0}} - \| \sum_{i=1}^{\eta} \Sig^{-i} \U^{\top} \W_{l_0 + i} \|}
\end{align*}

notice that the entries of $\U^{\top} \W_{l_0 + i}$ and $\U_\perp^{\top} \W_{l_0 + i}$ are i.i.d. Gaussian r.v's with variance $\sigma_c^2$. Next we define the matrix $M = \sum_{i=1}^\eta \Sig_\perp^{\eta-i} (\U_\perp^{\top} \W_{l_0 + i})$ and we apply Theorem \ref{thm:upper_bnd} to $M$. We can apply this theorem because we know that each entry of $M$ is a weighted sum of $\eta$ indepdendent Gaussian r.v.'s. In other words
\begin{gather*}
M_{jk} = \sum_{i=1}^\taubatch (\sigma_\perp)_j^{\eta-i} (\U_\perp^{\top} \W_{l_0+i})_{jk} \\
\implies M_{jk} \sim \mathcal{N}\left(0, \sigma_c^2 \sum_{i=1}^\taubatch (\lambda_\perp)_j^{2(\eta-i)}\right)\\
\implies \max_{jk} \|(M)_{jk}\|_{\psi_2} = \sigma_c \sqrt{\sum_{i=1}^\taubatch \sigma_{r+1}^{2(\eta-i)}} %\leq \sigma_c \sqrt{\taubatch} \max(1, \sigma_{r+1})^{\taubatch-1}
\end{gather*}
Recall that there is a factor of $\sigma_r^{- \eta}$ multiplying $M$ so effectively, the sub-Gaussian norm is $K = \sigma_r^{-\eta} \sigma_c \sqrt{\sum_{i=1}^\taubatch \sigma_{r+1}^{2(\eta-i)}} = \nsrmax \cdot \Gamma_{num}(\taubatch)$. Now, using Theorem \ref{thm:upper_bnd}, we get that with probability at least $1 - e^{-\epsilon^2}$
\begin{align*}
 \| \sum_{i=1}^{\eta} \Sig_\perp^{\eta-i} \U_\perp^{\top} \W_{l_0 + i} \| \leq C\nsrmax \cdot \Gamma_{num}(\taubatch) \cdot (\sqrt{n-r} + \sqrt{r} + \epsilon)
\end{align*}
and now picking $\epsilon = 0.01 \sqrt{n}$ followed by simple algebra yields
\begin{align*}
\Pr\left( \| \sum_{i=1}^{\eta} \Sig_\perp^{\eta-i} \U_\perp^{\top} \W_{l_0 + i} \|  \leq \sqrt{n} \nsrmax \cdot \Gamma_{num}(\taubatch)  \right) \\ \geq 1 - \exp(-cn)
\end{align*}

Next consider the denominator term. Again, we notice that the matrix $M =  \sum_{i=1}^{\eta} \Sig^{-i} \U^{\top} \W_{l_0 + i}$ has entries that are gaussian r.v.'s and are independent. Moreover, the sub Gaussian norm bound is
\begin{gather*}
M_{jk} = \sum_{i=1}^\taubatch \sigma_j^{-i} (\U^{\top} \W_{l_0+i})_{jk} \\
\implies M_{jk} \sim \mathcal{N}\left(0, \sigma_c^2 \sum_{i=1}^\taubatch \sigma_j^{-2i}\right) \\
\implies \max_{jk} \|(M)_{jk}\|_{\psi_2} = \sigma_c \sqrt{\sum_{i=1}^\taubatch \sigma_{r}^{-2 i}} := \nsrmax \cdot \Gamma_{denom}(\taubatch) %\leq \sigma_c \sqrt{\sum_{i=1}^\infty \sigma_{r}^{-2 i}} := \sigma_c \sqrt{\frac{1}{\sigma_r^2 -1}}
\end{gather*}
Now we apply Theorem \ref{thm:upper_bnd} to get that with probability $ 1 - \exp(-\epsilon^2)$
\begin{align*}
\norm{\sum_{i=1}^{\eta} \Sig^{-i} \U^{\top} \W_{l_0 + i}} \leq \nsrmax \cdot \Gamma_{denom}(\taubatch)\cdot (2 \sqrt{r} + \epsilon)
\end{align*}
picking $\epsilon = 0.01 \sqrt{r}$ yields that
\begin{align*}
\Pr\left( \norm{\sum_{i=1}^{\eta} \Sig^{-i} \U^{\top} \W_{l_0 + i}} \leq \sqrt{r}\nsrmax \cdot  \Gamma_{denom}(\taubatch) \right) \\ 
\geq 1 - \exp(-cr)
\end{align*}
This completes the proof of Lemma \ref{lem:desc}.
\end{proof}

\begin{proof}[Proof of Theorem \ref{thm:main_res}]
The idea for proving Theorem \ref{thm:main_res} is a straightforward extension from Lemma \ref{lem:desc}. Consider $\taubatch = 1$, and assume that the initial subspace estimtate, $\Uhat_0$ satisfies $\SE(\Uhat_0, \U) = \SE_0 < 1$ we know that with probability $1 -\exp(-cr) - \exp(-cn)$,
\begin{align*}
\SE(\Uhat_\taubatch, \U) &\leq \frac{R^\eta \SE_0 + \sqrt{n} \nsrmax \Gamma_{num}(\taubatch)}{0.9 \sqrt{1 - \SE_0^2} - \sqrt{r} \nsrmax \Gamma_{denom}(\taubatch)} \\
&= \frac{R \SE_0 + \sqrt{n} \nsrmax}{0.9 \sqrt{1 - \SE_0^2} - \sqrt{r} \nsrmax }
\end{align*}
thus, as long as $\nsrmax \leq 0.2 \sqrt{\frac{1 - \SE_0^2}{r}}$ the denominator is positive. Next, to achieve an $\epsilon$-accurate estimate, we note that the second term in the numerator is the larger term (since $R < 1$ and this goes to $0$ with every iteration) and thus as long as $\nsrmax \leq \frac{\epsilon}{\sqrt{n}}$ we can ensure that the numerator is small enough. Combining the two bounds, followed by a union bound over $L$ iterations gives the final conclusion.

Finally, consider the case of $\taubatch > 1$ and the $l$-th iteration. Assume that $\sigma_r > 1$. This is used to simplify the $\Gamma_{denom}(\taubatch)$ expression as follows: $\Gamma_{denom}^2(\taubatch) = (1 + \sigma_r^2 +  \cdots + \sigma_r^{2\taubatch - 2})/ \sigma_r^{2\taubatch - 2} = \sum_{i=0}^{\taubatch-1} 1/\sigma_r^{2i} \leq \sum_{i=0}^{\infty} 1/\sigma_r^{2i} = \frac{\sigma_r^2}{\sigma_r^2 - 1}$. Using the same reasoning as in the $\taubatch = 1$ case, as long as
\begin{align*}
\nsrmax \leq 0.2 \sqrt{\frac{\sigma_r^2 -1 }{\sigma_r^2}} \cdot \sqrt{\frac{1 - \SE_{(l-1)\taubatch}^2}{r}}
\end{align*}
the denominator is positive. We also have that $\Gamma_{num}^2(\taubatch) = \sum_{i=1}^{\taubatch} \sigma_{r+1}^{2(\taubatch - i)}/\sigma_r^{2\taubatch} \leq \taubatch R^{2\taubatch-2}$. Thus, as long as $\nsrmax \leq \frac{\epsilon}{\sqrt{n}} \cdot \frac{1}{\sqrt{\taubatch} R^{\taubatch-1}}$ the first term of the numerator is small enough and this gives us the final result.
\end{proof}

\subsubsection{Random Initialization Lemma}
Finally, we provide the proof for random initialization. This is a well known result as shown in \cite{smallest_rect, noisy_pm} but we prove it here for completeness.
\begin{proof}[Proof of Item 3 of  Theorem \ref{thm:main_res}]
The proof follows by application of Theorem \ref{thm:upper_bnd}, \ref{thm:lower_bnd_rect} to a standard normal random matrix, and definition of principal angles. Recall that $(\Uhat_0)_{ij} \overset{iid}{\sim} \mathcal{N}(0,1)$ and consider its reduced QR decomposition, $\Uhat_0 = \Qhat_0 \bm{R}_0$. We know that
\begin{align*}
\SE^2(\Uhat_0, \U) &= \|(\bm{I} - \Qhat_0 \Qhat_0^{\top}) \U\|^2 = \lambda_{\max}(\bm{I} - \U^{\top} \Qhat_0 \Qhat_0^{\top} \U) \\
&= 1 - \lambda_{\min} (\U^{\top} \Qhat_0 \Qhat_0^{\top} \U)  \\
&= 1 - \lambda_{\min}(\U^{\top} \Uhat_0 \bm{R}_0^{-1} (\bm{R}_0^{-1})^{\top} \Uhat_0^{\top} \U) \\
&\overset{(a)}{\leq} 1 - \lambda_{\min}(\U^{\top} \Uhat_0 \Uhat_0^{\top} \U) \lambda_{\min}(\bm{R}_0^{-1} (\bm{R}_0^{-1})^{\top})) \\
&= 1 - \frac{\sigma_{\min}^2(\U^{\top} \Uhat_0)}{\|\Uhat_0\|_2^2}
\end{align*}
where $(a)$ follows from Ostrowski's Theorem (Theorem 4.5.9, \cite{horn_johnson}) and the last relation follows since reduced qr decomposition preserves the singular values. It is easy to see that $(\U^{\top} \Uhat_0)_{ij} \sim \mathcal{N}(0,1)$. We can apply Theorem \ref{thm:lower_bnd_rect} to get that with probability at least $1 - \exp(-c r) - (c / \gamma)$,
\begin{align*}
\sigma_{\min} (\U^{\top} \Uhat_0) \geq c (\sqrt{r} - \sqrt{r-1})/\gamma
\end{align*}
and we also know that $\sqrt{r} - \sqrt{r-1} = O(1/\sqrt{r})$. Additionally, the denominator term is bounded using Theorem \ref{thm:upper_bnd} as done before and thus, with probability $1 - \exp(-\epsilon^2)$,
\begin{align*}
\|\Uhat_0\| \leq C (\sqrt{n} + \sqrt{r} + \epsilon)
\end{align*}
and now picking $\epsilon = 0.01 \sqrt{n}$ we get that with probability at least $1 - \exp(-cn) - \exp(-cr) - (1/c\gamma)$,
\begin{align*}
\SE^2(\Uhat_0, \U) \leq 1 - \frac{1}{\gamma n r}
\end{align*}
which completes the proof.
\end{proof}
While invoking the above result, to simplify notation, we set $\gamma=10$.

\newcommand{\Span}{\mathrm{Span}}

\subsection{Numerical Verificaion of Theorem \ref{thm:main_res}}
We generate $\X = \U \Lambda \V^T + \U_{\perp} \Lambda_{\perp} \V_{\perp}^T$ with $\U^* = [\U, \U_{\perp}]$, $\V^* = [\V, \V_{\perp}]$ being orthonormal matrices of appropriate dimensions. We then set $\Y = \X \X^T$ and the goal is to estimate the span of the $n \times r$ dimensional matrix, $\U$. We choose $n = 1000$ and $r=30$. We consider two settings where $\Lambda = 1.1 \I$, $\Lambda_{\perp} = \I$ so that $R = 0.91$; and $\Lambda = 3.3 \I$, $\Lambda_{\perp} = \I$ so that $R = 0.33$.  At each iteration we generate channel noise as i.i.d. $\mathcal{N}(0, \sigma_c^2)$. We verify the claims of Theorem \ref{thm:main_res} and (i) show that choosing  a larger value of $\taubatch$ considerably increases robustness to noise. We set $R = 0.91$, and consider $\taubatch = 1,10$ and $\sigma_c = 10^{-4}, 10^{-4}$. See from Fig. \ref{fig:fed_pm_sig}(a) that increasing $\taubatch$ has a similar effect as that of reducing $\sigma_c$ (the $\taubatch =10, \sigma_c = 10^{-8}$ plot overlaps with $\taubatch = 1, \sigma_c=10^{-8}$); and (ii) in Fig. \ref{fig:fed_pm_sig}(b) we show that choosing a smaller value of $R$ speeds up convergence, and also increases noise robustness. Here we use $\sigma_c = 10^{-8}$ and consider two eigengaps, $R = \{0.91, 0.30\}$. 

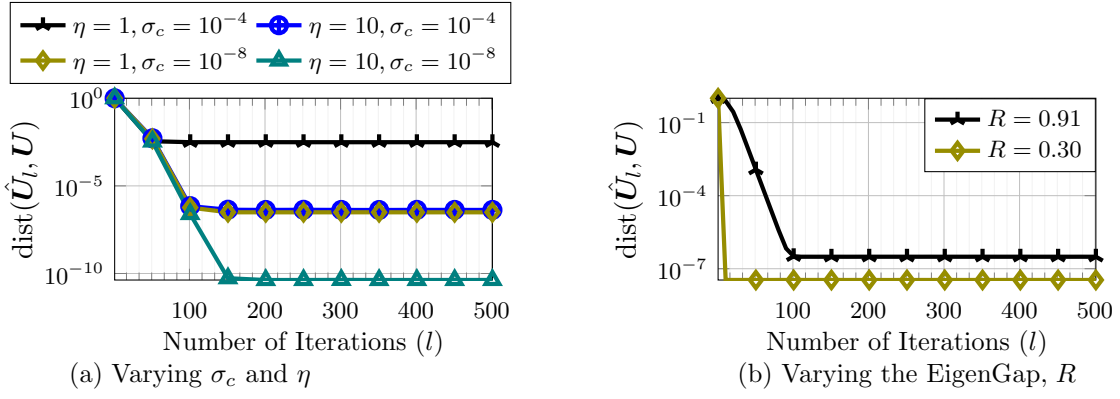
\begin{figure*}[t!]
\centering
\begin{tikzpicture}
    \begin{groupplot}[
        group style={
            group size=2 by 1,
            horizontal sep=3cm,
            vertical sep=.5cm,
                        %x descriptions at=edge bottom,
            %y descriptions at=edge left,
        },
        % load the style created in the preamble
        %my stylecompare,
                %ymin=1e-15, ymax=0,
        enlargelimits=false,
        width = .4\linewidth,
        height=4cm,
        enlargelimits=false,
                grid=both,
    grid style={line width=.1pt, draw=gray!10},
    major grid style={line width=.2pt,draw=gray!50},
    minor x tick num=5,
    minor y tick num=5,
    ]
	               \nextgroupplot[
	               		legend entries={
					{$\taubatch = 1, \sigma_c = 10^{-4}$},
            		{$\taubatch = 10, \sigma_c = 10^{-4}$},
            		{$\taubatch = 1, \sigma_c = 10^{-8}$},
            		{$\taubatch = 10, \sigma_c = 10^{-8}$},
            		},
            legend style={at={(1.05,1.53)}},
            %legend style={at={(1,1.8)}}, % use this if one column fig
            legend columns = 2,
            legend style={font=\footnotesize},
            ymode=log,
            %ymin=1e-15, ymax=0,
            %xlabel={\small{Number of Samples ($t$)}},
            xlabel={\small{Number of Iterations ($l$)}},
            ylabel={{$\SE(\Qhat_{l},\U)$}},
                        title style={at={(0.2,-.5)},anchor=north,yshift=1},
            title={\small{(a) Varying $\sigma_c$ and $\taubatch$}},
                        xticklabel style= {font=\footnotesize, yshift=-1ex},
                                    yticklabel style= {font=\footnotesize},
        ]
                	        \addplot [black, line width=1.6pt, mark=Mercedes star,mark size=3pt] table[x index = {0}, y index = {1}]{\varsigc};
	        \addplot [blue, line width=1.6pt, mark=oplus,mark size=3pt] table[x index = {0}, y index = {4}]{\varsigc};
	        \addplot [olive, line width=1.6pt, mark=diamond,mark size=3pt] table[x index = {0}, y index = {3}]{\varsigc};
	        \addplot [teal, line width=1.6pt, mark=triangle,mark size=3pt] table[x index = {0}, y index = {6}]{\varsigc};
	
		               \nextgroupplot[
	               		legend entries={
					{$R = 0.91 $},
            		{$R = 0.30 $},
            		},
            legend style={at={(1,1)}},
            %legend style={at={(1,1.8)}}, % use this if one column fig
            legend columns = 1,
            legend style={font=\footnotesize},
            ymode=log,
            %ymin=1e-15, ymax=0,
            %xlabel={\small{Number of Samples ($t$)}},
            xlabel={\small{Number of Iterations ($l$)}},
                        ylabel={{$\SE(\Qhat_{l},\U)$}},
                                    title style={at={(0.5,-.5)},anchor=north,yshift=1},
            title={\small{(b) Varying the EigenGap, $R$}},
                        xticklabel style= {font=\footnotesize, yshift=-1ex},
                                    yticklabel style= {font=\footnotesize},
        ]
                	        \addplot [black, line width=1.6pt, mark=Mercedes star,mark size=3pt, mark repeat = 5] table[x index = {0}, y index = {1}]{\varrat};
	        %\addplot [blue, line width=1.6pt, mark=oplus,mark size=3pt, mark repeat = 5] table[x index = {0}, y index = {4}]{\varrat};
	        \addplot [olive, line width=1.6pt, mark=diamond,mark size=3pt, mark repeat = 5] table[x index = {0}, y index = {3}]{\varrat};
	        %\addplot [teal, line width=1.6pt, mark=triangle,mark size=3pt, mark repeat = 5] table[x index = {0}, y index = {6}]{\varrat};
    \end{groupplot}
\end{tikzpicture}
%\vspace{-.45cm}
\caption{\small{Numerical verification of Theorem \ref{thm:main_res}: {\bf Left:} increasing $\taubatch$ increases robustness to noise; {\bf Right:} Increasing the ``gap'' helps achieve faster, better convergence. 
}}
\vspace{-.45cm}
\label{fig:fed_pm_sig}
\end{figure*}

\section{Preliminaries}
The following result is Theorem 4.4.5, \cite{hdp_book}
\begin{theorem}[Upper Bounding Spectral Norm] \label{thm:upper_bnd}
Let $A$ be a $m \times n$ random matrix whose entries are independent zero-mean sub-Gaussian r.v.'s and let $K = \max_{i,j} \|A_{i,j}\|_{\psi_2}$. Then for any $\epsilon >0$ with probability at least $1 - 2\exp(-\epsilon^2)$,
\begin{align*}
\|A\| \leq C K (\sqrt{m} + \sqrt{n} + \epsilon)
\end{align*}
\end{theorem}

The following result (Theorem 1.1, \cite{smallest_rect}) bounds the smallest singular value of a random rectangular matrix.

\begin{theorem}[Lower Bounding Smallest Singular Value for Rectangular matrices]\label{thm:lower_bnd_rect}.
Let $A$ be a $m \times n$ random matrix whose entries are independent zero-mean sub-Gaussian r.v.'s. Then for any $\epsilon >0$ we have
\begin{align*}
\sigma_{\min}(A) \geq \epsilon C_K(\sqrt{m} - \sqrt{n-1})
\end{align*}
with probability at least $1 -  \exp(-c_K n) - (c_K \epsilon)^{m -n +1}$. Here, $K = max_{i,j} \|A_{i,j}\|_{\psi_2}$.
\end{theorem}

\bibliographystyle{IEEEbib}
\bibliography{numerical_analysis}

\end{document}